%% file: main.tex
\documentclass[10pt]{article} 
\PassOptionsToPackage{table}{xcolor}
\usepackage[main,final]{includes/conf}
\usepackage[normalem]{ulem}
\usepackage{wrapfig}
\usepackage{graphicx}
\usepackage{rotating}
\usepackage{empheq}
\usepackage{caption}

\newcommand{\zoomedImage}[4][]{
\begin{tikzpicture}[spy using outlines={circle,red,magnification=2,size=0.8cm, connect spies}]
  \node (img) {\includegraphics[#1]{#2}};
  \spy on (#3) in node [left] at (#4);
\end{tikzpicture}%
}

\input{includes/mathcommands}

\input{includes/includes}

\newcommand{\scaleY}{0.40\textwidth}
\newcommand{\sidecap}[1]{{\begin{sideways}\parbox{\scaleY}{\centering #1}\end{sideways}}}

\title{Consistency Regularised Gradient Flows for Inverse Problems}
\newcommand{\equalcontribstar}{\textsuperscript{*}}
\newcommand{\equalcontribdagger}{\textsuperscript{\ensuremath{\dagger}}}

\author{%
\normalfont
\begin{minipage}{0.98\textwidth}
\centering
\textbf{Alessio Spagnoletti}$^{1}$\equalcontribstar \quad
\textbf{Tim Y. J. Wang}$^{2}$\equalcontribstar \\ [0.35em]
\textbf{Marcelo Pereyra}$^{3}$\equalcontribdagger \quad
\textbf{O. Deniz Akyildiz}$^{2}$\equalcontribdagger
\\[0.35em]
$^{1}$Laboratoire MAP5, UMR 8145, Université Paris Cité, CNRS
\\
$^{2}$Department of Mathematics, Imperial College London
\\
$^{3}$Heriot-Watt University, MACS \& Maxwell Institute
\\[0.35em]
\small
\texttt{alessio.spagnoletti@etu.u-paris.fr \quad tw1320@ic.ac.uk}
\\
\texttt{M.Pereyra@hw.ac.uk \quad deniz.akyildiz@imperial.ac.uk}
\\[0.25em]
\footnotesize
\equalcontribstar Joint first authors \,\,\, \equalcontribdagger Equal contribution
\end{minipage}%
}

\begin{document}

\maketitle
\begingroup
\renewcommand{\thefootnote}{\ensuremath{*,\dagger}}
\endgroup
\input{sections/0-abstract.tex}
\input{sections/1-intro-background.tex}
\input{sections/2-gradflowintro.tex}
\input{sections/method.tex}
\input{sections/3-prior-method.tex}
\input{sections/4-likelihood-algo.tex}
\input{sections/5-experiments.tex}
\input{sections/6-conclusion.tex}

\input{sections/7-acknow}

\bibliographystyle{plainnat}  
\bibliography{bibliography} 

\newpage
\appendix
\input{appendix.tex}

\clearpage
\input{includes/checklist}

\end{document}

%% file: includes/mathcommands.tex
\usepackage{amsmath,amsfonts,bm}

\def\dd{{\mathrm{d}}}

\def\sZ{{\mathsf{Z}}}
\def\sC{{\mathsf{C}}}

\def\1{\bm{1}}

\def\vc{{\bm{c}}}

\def\vn{{\bm{n}}}

\def\vu{{\bm{u}}}

\def\vw{{\bm{w}}}
\def\vx{{\bm{x}}}
\def\vy{{\bm{y}}}
\def\vz{{\bm{z}}}

\DeclareMathAlphabet{\mathsfit}{\encodingdefault}{\sfdefault}{m}{sl}
\SetMathAlphabet{\mathsfit}{bold}{\encodingdefault}{\sfdefault}{bx}{n}

\newcommand{\Id}{\mathrm{Id}}
\newcommand{\E}{\mathbb{E}}

\newcommand{\R}{\mathbb{R}}

\newcommand{\KL}{D_{\mathrm{KL}}}
\newcommand{\DF}{D_{\mathrm{F}}}

\DeclareMathOperator*{\argmax}{arg\,max}
\DeclareMathOperator*{\argmin}{arg\,min}

%% file: includes/includes.tex
\usepackage{framed, multido}
\usepackage[utf8]{inputenc}
\usepackage[T1]{fontenc}    
\usepackage{booktabs}       
\usepackage{nicefrac}       
\usepackage{microtype}      
\usepackage{lipsum}
\usepackage{adjustbox}
\usepackage[table]{xcolor}

\usepackage{amsmath,amsthm,amssymb, amsfonts}
\usepackage{pythonhighlight}
\usepackage{setspace}
\usepackage{svg} 
\usepackage{float}
\usepackage{enumerate, setspace}

\usepackage{tikz}
\usetikzlibrary{spy}
\usepackage{tikz-cd}
\usetikzlibrary{bayesnet,positioning}
\usepackage{wrapfig}
\usepackage{soul}
\sethlcolor{yellow}
\usepackage{subcaption}
\usepackage{algorithm}
\usepackage{array}
\usepackage{tabularx}
\usepackage{algpseudocode}
\usepackage{natbib}
\usepackage[makeroom]{cancel}
\usepackage{pifont}
\definecolor{cmarkgreen}{RGB}{0, 150, 0}
\definecolor{xmarkred}{RGB}{200, 0, 0}
\newcommand{\cmark}{\textcolor{cmarkgreen}{\ding{51}}}
\newcommand{\xmark}{\textcolor{xmarkred}{\ding{55}}}

\definecolor{astral}        {RGB}{46,116,181}
\definecolor{cb-blue}       {RGB}{70, 130, 180}
\definecolor{orange}        {RGB}{214,150, 92}
\definecolor{darkorange}    {RGB}{171, 120, 74}
\definecolor{green}         {RGB}{136,196,136}
\definecolor{darkgreen}{RGB}{34, 95, 78}
\definecolor{lightgreen}{RGB}{248, 245, 238}
\definecolor{darkblue}{RGB}{31, 52, 122}

\usepackage{hyperref}
\hypersetup{
colorlinks=true,
allcolors=darkblue
}
\usepackage{cleveref}
\usepackage{bookmark}
\svgpath{{figures/}}
\usepackage{subcaption}

\makeatletter
\newcommand{\HEADER}[1]{\ALC@it\underline{\textsc{#1}}\begin{ALC@g}}
\newcommand{\ENDHEADER}{\end{ALC@g}}
\makeatother

\theoremstyle{plain}
\newtheorem{theorem}{Theorem}
\newtheorem{lemma}{Lemma}
\newtheorem{corollary}{Corollary}
\newtheorem{proposition}{Proposition}
\newtheorem{assumption}[theorem]{Assumption}

\Crefname{assumption_app}{\textbf{A}\hspace{-0.12cm}}{\textbf{H}\hspace{-3pt}}
\crefname{assumption_app}{\textbf{A}}{\textbf{A}}
\theoremstyle{definition}
\newtheorem{definition}{Definition}

\newtheorem{remark}{Remark}

\usepackage[acronym]{glossaries}

\newacronym{DM}{{{\textsc{\small DM}}}}{Diffusion Model}
\newacronym{GAN}{{{\textsc{\small GAN}}}}{Generative Adversarial Network}
\newacronym{VAE}{{{\textsc{\small VAE}}}}{Variational Autoencoder}
\newacronym{IWAE}{{{\textsc{\small IWAE}}}}{Importance Weighted Variational Autoencoder}
\newacronym{DiffusionVAE}{{{\textsc{\small DiffusionVAE}}}}{ VAE with a Diffusion Prior}
\newacronym{VI}{{{\textsc{\small VI}}}}{Variational Inference}
\newacronym{PSVI}{{{\textsc{\small PSVI}}}}{Particle Semi-implicit Variational Inference}
\newacronym{UIVI}{{{\textsc{\small UIVI}}}}{Unbiased Implicit Variational Inference}
\newacronym{SVI}{{{\textsc{\small SVI}}}}{Semi-implicit Variational Inference}
\newacronym{VQVAE}{{{\textsc{\small VQ-VAE}}}}{Vector-Quantized Variational Autoencoder}
\newacronym{LDM}{{{\textsc{\small LDM}}}}{Latent Diffusion Model}
\newacronym{EBM}{{{\textsc{\small EBM}}}}{Energy-based Model}
\newacronym{IPLD}{{{\textsc{\small IPLD}}}}{Interacting Particle Latent Diffusion}
\newacronym{DDPM}{{{\textsc{\small DDPM}}}}{Denoising Diffusion Probabilistic Model}
\newacronym{LSGM}{{{\textsc{\small LSGM}}}}{Latent Score-based Generative Models}
\newacronym{SVGD}{{{\textsc{\small SVGD}}}}{Stein Variational Gradient Descent}
\newacronym{MMD}{{{\textsc{\small MMD}}}}{Maximum Mean Discrepancy}
\newacronym{ELBO}{{{\textsc{\small ELBO}}}}{Evidence Lower BOund}
\newacronym{LDDM}{{{\textsc{\small LDDM}}}}{Latent Denoising Diffusion Model}
\newacronym{LVM}{{{\textsc{\small LVM}}}}{Latent Variable Model}
\newacronym{PGD}{{{\textsc{\small PGD}}}}{Particle Gradient Descent}
\newacronym{MPGD}{{{\textsc{\small MPGD}}}}{Momentum Particle Gradient Descent}
\newacronym{EM}{{{\textsc{\small EM}}}}{Expectation-Maximisation}
\newacronym{FID}{{{\textsc{\small FID}}}}{Fréchet Inception Distance}
\newacronym{GMM}{{\textsc{\small GMM}}}{Gaussian Mixture Model}
\newacronym{SWD}{{\textsc{\small SWD}}}{Sliced Wasserstein Distance}
\newacronym{LPIPS}{\textsc{\small LPIPS}}{Learned Perceptual Image Patch Similarity}
\newacronym{PSNR}{\textsc{\small PSNR}}{Peak Signal-to-Noise Ratio}
\newacronym{WGF}{\textsc{\small WGF}}{Wasserstein Gradient Flow}
\newacronym{DWGF}{\textsc{\small DWGF}}{Diffusion-regularized Wasserstein Gradient Flow}
\newacronym{PSLD}{\textsc{\small PSLD}}{Posterior Sampling with Latent Diffusion}
\newacronym{RLSD}{\textsc{\small RLSD}}{Repulsive Latent Score Distillation}

%% file: sections/0-abstract.tex
\begin{abstract}
Vision-Language Latent Diffusion Models (LDMs)~\citep{rombachhighresolutionimagesynthesis2022} provide powerful generative priors for inverse problems. However, existing LDM-based inverse solvers typically require a large number of neural function evaluations (NFEs) and backpropagation through large pretrained components, leading to substantial computational cost and, in some cases, degraded reconstruction quality. We propose a unified Euclidean-Wasserstein-2 gradient-flow framework that jointly performs posterior sampling and prompt optimization in the latent space through a single flow that aligns the prior and posterior with the observed data. Combined with few-step latent text-to-image models, this formulation enables low-NFE inference without backpropagation through autoencoders. Experiments across several canonical imaging inverse problems show that our method achieves state-of-the-art performance with significantly reduced computational cost.
\end{abstract}

%% file: sections/1-intro-background.tex
\section{Introduction}
\begin{wrapfigure}{r}{0.65\textwidth}
    \centering
    \includegraphics[width=0.6\textwidth]{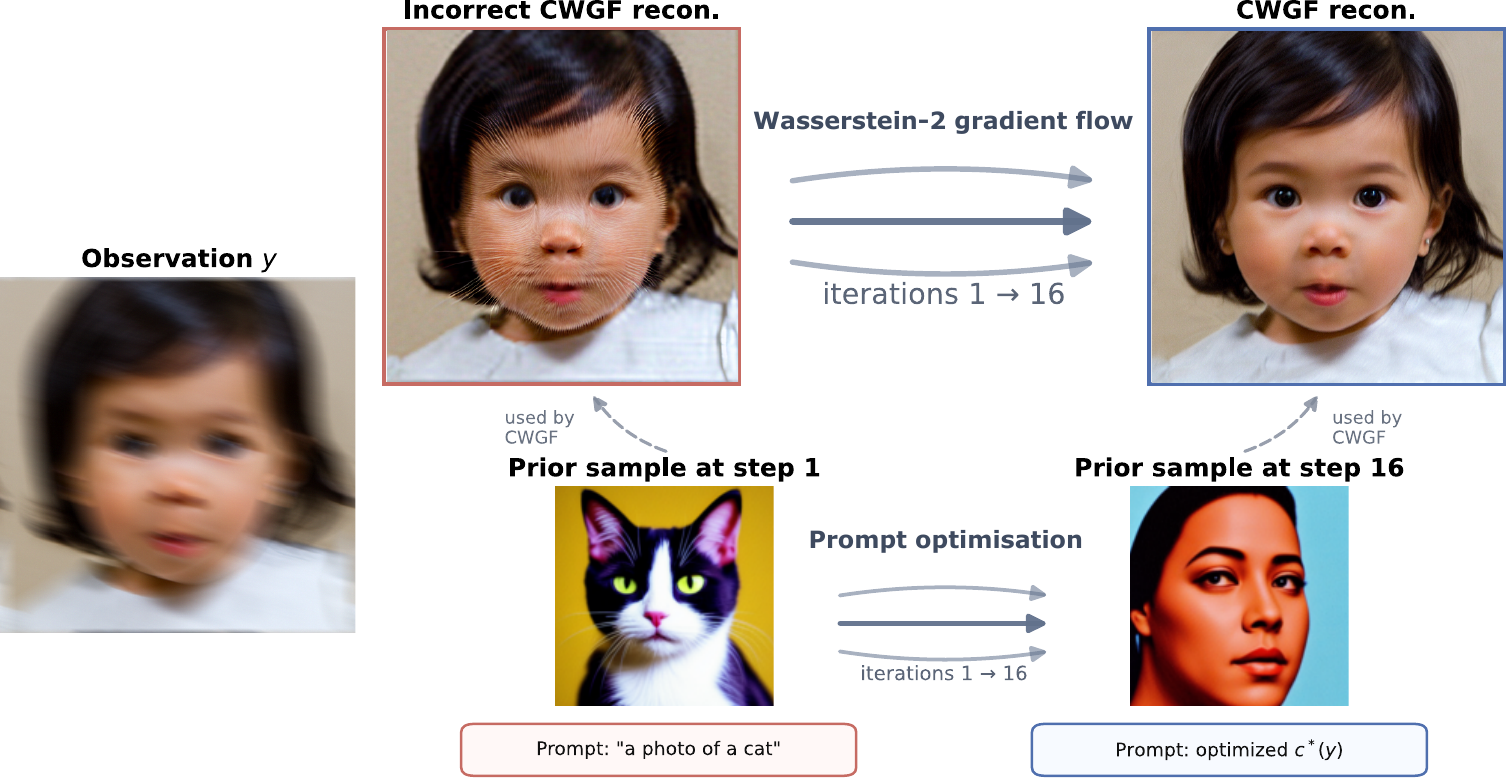}
    \caption{Our Euclidean-Wasserstein flow jointly optimizes the text prompt and yields accurate posterior samples in few NFEs. 
    Zoom in for fine details.
    }
    \label{fig:intro_teaser}
\end{wrapfigure}
We consider methods for performing inference on a signal $\vx_0 \in \mathbb{R}^n$ from a measurement
\begin{equation}\label{eq:simple_inv}
    \vy=\mathcal{A}(\vx_0)+\vn,
\end{equation}
where ${\mathcal{A}: \mathbb{R}^n \to \mathbb{R}^m}$ is a known forward operator and $\vn$ is Gaussian noise distributed as $\mathcal{N}(0, \sigma_\vy^2 I)$. Such inverse problems can be effectively addressed in a Bayesian statistical framework \citep{stuart2010inverse}, where inference about $\vx_0$ is based on the posterior distribution \(p(\vx_0|\vy) \propto p(\vy|\vx_0) p(\vx_0)\). 
Modern Bayesian imaging techniques increasingly rely on pretrained generative models as priors $p(\vx_0)$, with a likelihood $p(\vy \mid \vx)$ defined at test time. These so-called Plug \& Play (PnP) zero-shot inference methods are a rapidly growing research area, as they require no task-specific fine-tuning and provide effective, general-purpose inverse solvers. Methods based on Diffusion Models (DMs)~\citep{SohlDickstein2015DeepUL, Song2019, Song2020improved, Song2020SDE, Ho2020Denoising,ho2022imagenvideohighdefinition,blattmann2023stable,wan2025} are the current standard, with approaches such as DPS~\citep{chung2023diffusion}, DDRM~\citep{kawar2022denoising}, DiffPIR~\citep{zhudenoisingdiffusionmodels2023}, $\Pi$GDM~\citep{Song2023PseudoinverseGuidedDM}, and TMPD~\citep{boys2024tweedie} for image restoration and~\citep{kwon2025visionxlhighdefinitionvideo,Kwon2025VideoDP} for video. More recently, Latent Diffusion Models (LDMs)~\citep{Rombach2021,Podell2023SDXLIL} have improved efficiency by operating in the latent space of a VAE rather than in pixel space~\citep{Kingma2013AutoEncodingVB}. However, because the likelihood is defined in pixel space, naively applying DM-based solvers to LDMs fails. Instead, the specific properties of the latent space need to be considered, as proposed in PSLD~\citep{routsolvinglinearinverse2023}.

State-of-the-art results are achieved with vision-language LDMs such as Stable Diffusion~\citep{Rombach2021, Podell2023SDXLIL}, which condition generation on semantic information via an (embedded) text prompt $c \in \mathbb{R}^d$ that must be carefully adjusted to the scene. Early methods tuned prompts heuristically (see, e.g., P2L~\citep{Chung2023PrompttuningLD} and TReg~\citep{kim2025regularizationtextslatentdiffusion}), while LATINO-PRO~\citep{Spagnoletti_2025_ICCV} adopts an empirical Bayesian approach, estimating $c$ by maximum marginal likelihood estimation (MMLE) $\hat{c}(\vy) = \arg\max_{c\in\mathbb{R}^d} p_c(\vy)$.

However, leveraging vision-language LDMs for inverse problems is computationally demanding: naive solvers require hundreds of NFEs per posterior sample, and incorporating the likelihood $p(\vy \mid \vx)$ requires evaluating (and possibly differentiating) the latent-to-pixel decoder model, substantially increasing memory footprint. Recent efforts to speed up sampling rely on distillation techniques, including direct distillation~\citep{Luhman2021KnowledgeDI,Zheng2022FastSO}, adversarial distillation~\citep{Wang2022DiffusionGANTG,Sauer2023AdversarialDD}, and variational score distillation~\citep{wangprolificdreamerhighfidelitydiverse2023,yinonestepdiffusiondistribution2024,yin2024improveddistributionmatchingdistillation,Luo2023DiffInstructAU,Salimans2024MultistepDO}. In particular, Consistency Models (CMs)~\citep{songconsistencymodels2023,songdhariwal2023improved,kim2023consistency} have emerged as effective few-step generators, and their formulation enables both distillation from pretrained DMs and training in isolation via consistency training (CT). CMs learn consistency functions that directly map noise to data, enabling accurate generation in few steps. Several recent works have explored CMs within a PnP sampling framework~\citep{Garber_2025_CVPR, xu2024consistencymodeleffectiveposterior, li2025decoupleddataconsistencydiffusion, Spagnoletti_2025_ICCV, Spagnoletti2025LVTINOLV}, as well as fine-tuning models to sample directly from $p(\vx | \vy)$, as done in CoSIGN \citep{Zhao2024CoSIGNFG} or UD2M~\citep{Mbakam2025LearningFP}.

The state of the art in this context is LATINO-PRO, a CM-based method with automatic prompt tuning that avoids backpropagation through autoencoders. However, its stochastic approximation structure, which alternates sampling and prompt optimization steps, converges slowly and requires over 50 NFEs per sample. Instead, we pursue a unified approach that jointly samples and optimizes prompts in a single flow. Our contributions are as follows:

\begin{itemize}
    \item We propose a novel Euclidean-Wasserstein gradient flow framework for inverse problems with LCMs. This leads to a fast, provably convergent method for joint posterior sampling and prompt optimization, requiring gradients only w.r.t. the prompt.
    
    \item Our method achieves SOTA performance in only $16$ NFEs, significantly outperforming competing strategies, even under severely misspecified prompts.
\end{itemize}

\section{Background}
\label{sec:background}

\paragraph{Notation.} We use $\mathsf{X} = \mathbb{R}^n$ to denote the ambient (pixel) space, $\mathsf{Y}\subseteq \R^m$ the observation space, and $\mathsf{Z} = \mathbb{R}^l$ the latent space in which the prior is defined. We denote the space of probability measures on $\mathbb{R}^d$ with finite $q$-th moments as $\mathcal{P}_q(\mathbb{R}^d)$; in this work, we focus on the case $q=2$. We use $\delta\mathcal{F}[\mu]/\delta\mu$ to denote the first $L^2$-variation of functional $\mathcal{F}$ at $\mu$.

\paragraph{Latent diffusion models.} LDMs operate in the latent space of a VAE $(\mathcal E_{\phi^-},\mathcal D_{\phi^-})$~\citep{Kingma2013AutoEncodingVB,rombachhighresolutionimagesynthesis2022}. They are defined by a forward (noising) process on $\sZ$ that transports the clean embeddings $\vz_0\sim p_c(\vz_0)$ to a reference distribution (e.g., $\mathcal N(0,I)$), and a corresponding reverse-time process that transports back to $p_c(\vz_0)$, described by the SDEs~\citep{SohlDickstein2015DeepUL,Ho2020Denoising,Song2020SDE,rombachhighresolutionimagesynthesis2022}:
\begin{equation*}
    \dd\vz_t = -\tfrac{\beta(t)}{2} \vz_t \,dt + \sqrt{\beta(t)} \, \dd\vw \quad \quad \dd\vz_t = \left( -\tfrac{\beta(t)}{2} \vz_t - \beta(t) \nabla_{\vz_t} \log p_{c,t}(\vz_t) \right) \dd t + \sqrt{\beta(t)} \, \dd\bar{\vw},
\end{equation*}
where the (intractable) score $\nabla_{\vz_t}\log p_{c,t}(\vz_t)$ controls the drift and $\beta(t)$ the transitions:
\begin{equation}
    q(\vz_t \mid \vz_0)
    =\mathcal N\!\big(\vz_t;\sqrt{\alpha_t}\,\vz_0,\sigma_t^2 I\big),\label{eq:latent_forward_kernel}
    \qquad
    \alpha_t = \exp\!\left(-\int_0^t \beta(s)\,\dd s\right),
    \qquad
    \sigma_t^2 = 1-\alpha_t.
\end{equation}
In practice, the score is approximated by a neural network $s_{\theta,c}(\vz_t,t)\ \approx\ \nabla_{\vz_t}\log p_{c,t}(\vz_t)$ via score matching~\citep{Vincent2011}, and the reverse SDE is replaced by a \emph{probability flow ODE} (PF-ODE) equivalent in marginals and more efficient to compute~\citep{Song2020SDE}:
\begin{equation}
    \frac{\dd \vz_t}{\dd t}
    = f(\vz_t,t) - \frac{1}{2}g(t)^2\, s_{\theta,c}(\vz_t,t),
    \label{eq:pf_ode_generic}
\end{equation}
for suitable drift $f$ and diffusion coefficient $g$ (e.g., in the Variance Preserving (VP) case, $f(\vz_t,t)=-\tfrac12\beta(t) \vz_t$ and $g(t)=\sqrt{\beta(t)}$~\citep{Ho2020Denoising,Song2020SDE}).

\paragraph{Consistency models.}
CMs enable few-step generation by replacing iterative SDE/ODE sampling with a learned \emph{consistency function} that maps a noisy state at time $t$ directly to a clean sample by transporting along PF-ODE trajectories~\citep{songconsistencymodels2023,songdhariwal2023improved}.
LCMs use a prompt-conditioned consistency function $g_{\theta}:\ (\vz_t,t,c)\ \mapsto\ \hat{\vz}_0$ that predicts the clean embedding on the trajectory. Informally, for any $\{\vz_t\}_{t\in[\eta,T]}$, we have
\begin{equation}
    g_{\theta}(\vz_t,t,c)\;=\;g_{\theta}(\vz_{t'},t',c)
    \qquad \forall\, t,t'\in[\eta,T].
    \label{eq:self_consistency}
\end{equation}
CMs can be trained in isolation or by distilling a pretrained diffusion model, and can achieve excellent quality in as few as $4-8$ NFEs~\citep{songconsistencymodels2023,yinonestepdiffusiondistribution2024,yin2024improveddistributionmatchingdistillation}.

At inference, given a decreasing time sequence $T=t_0>t_1>\cdots>t_{K-1}>\eta$ and a purely noisy sample $\hat{\vz}_{t_0}\sim\mathcal N(0,I)$, multistep CM sampling alternates between (i) denoising with $g_{\theta}$ and (ii) re-noising to the next time point via the noising process~\eqref{eq:latent_forward_kernel}~\citep{songconsistencymodels2023}:
\begin{align}
    \hat{\vz}_0 &\leftarrow g_{\theta}(\hat{\vz}_{t_k},t_k,c), \\
    \hat{\vz}_{t_{k+1}} &\leftarrow \sqrt{\alpha_{t_{k+1}}}\,\hat{\vz}_0 + \sigma_{t_{k+1}}\,\varepsilon,
    \qquad \varepsilon\sim\mathcal N(0,I).
    \label{eq:cm_multistep}
\end{align}
This ``denoise--renoise'' recursion underlies Latent Consistency Models (LCMs) as fast priors for inverse problems. In practice, LCMs are typically distilled from Stable Diffusion LDMs in latent space~\citep{rombachhighresolutionimagesynthesis2022,Luo2023LatentCM}, often using lightweight adapters such as LoRA~\citep{Hu2021LoRALA,Luo2023LCMLoRAAU}.

%% file: sections/2-gradflowintro.tex
\section{A Euclidean-Wasserstein Gradient Flow Approach}\label{sec:Method}
Our proposed Consistency-regularised Wasserstein Gradient Flow (CWGF) simultaneously optimises the prompt embedding by maximum marginal likelihood estimation and computes the associated empirical Bayesian posterior via a Wasserstein-2 gradient flow in latent space. We formulate our method as a continuous-time gradient flow over probability distributions and text embeddings, and derive efficient discretisations through an LCM in VAE latent space.

\textbf{Bayesian model.} Recall that we seek to perform inference on $\vx \in \mathsf{X}$ from an observation $\vy = \mathcal A(\vx) + \epsilon$, where $\mathcal A:\mathsf{X}\to \mathsf{Y}$ is a known linear operator and $\epsilon\sim \mathcal{N}(0, \sigma_\vy^2 I)$ is measurement noise. We further assume that $\vx$ is generated from a latent variable $\vz\in \sZ$ via a pretrained VAE decoder $p_{\phi^-}(\vx|\vz)$. The prior on $\vz$ is parametrised by a text prompt embedding $c\in \sC$ encoded by a CM $p_c(\vz)$~\citep{Luo2023LatentCM,yinonestepdiffusiondistribution2024}. From Bayes' theorem:
\begin{equation}\label{eq:main_model}
    p_c(\vx,\vz|\vy)=\frac{p(\vy|\vx)p_{\phi^-}(\vx|\vz)p_c(\vz)}{p_c(\vy)},
\end{equation}
with data likelihood $p(\vy|\vx):=\mathcal{N}(\vy;\mathcal A(\vx),\sigma_\vy^2 I)$ and where $p_c(\vy)$ is the marginal likelihood of $c$.

\textbf{Optimisation objective in \texorpdfstring{$\sC\times \mathcal{P}_2(\sZ)$}{CxP2(Z)}}. Our method seeks to solve optimisation problems defined on the product space $\sC\times \mathcal{P}_2(\sZ)$, where $\sC$ is the prompt embedding space and $\mathcal{P}_2(\sZ)$ the space of probability measures on the latent space $\sZ$ with finite second moment. We consider optimisation problems of the form:
\begin{equation}\label{eq:main-joint-opt}
    (c_\star, \mu_\star) = \argmin_{(c,\mu) \in \sC \times \mathcal{P}_2(\sZ)} \mathcal{F}[c,\mu]:=\mathcal{R}[c,\mu]+\mathcal{L}[\mu],
\end{equation}
where $\mathcal{R}$ and $\mathcal{L}$ are functionals defined on $\sC\times \mathcal{P}_2(\sZ)$ and $\mathcal{P}_2(\sZ)$ respectively. While this formulation (hence our method) can be quite general, we focus on the case where $\mathcal{R}$ and $\mathcal{L}$ are the KL divergence and negative log-likelihood respectively:
\begin{equation}\label{eq:loss-choices}
    \mathcal{R}[c,\mu]=\KL(\mu\|p_c), \qquad \mathcal{L}[\mu]:=-\mathbb{E}_\mu[\log p(\vy|\vz)],
\end{equation} 
with $\mathcal{R}$ being a regularisation term encouraging the latent distribution $\mu$ to be close to the CM-induced prior $p_c$, and $\mathcal{L}[\mu]$ being the negative log-likelihood term that matches the observed data $\vy$ through $p(\vy|\vz):=\int p(\vy|\vx)p_{\phi^-}(\vx|\vz)\dd \vx$. This specific choice of $\mathcal{R}$ and $\mathcal{L}$ leads to empirical Bayesian inference with MMLE~\citep{dempster1977maximum} for the model in \eqref{eq:main_model}, i.e., to find $c_\star \in \argmax_{c\in \sC} \log p_c(\vy)$, where $p_c(\vy):=\int p_c(\vx,\vy,\vz)\dd \vx \dd \vz$~\citep{kuntz23a} and $\mu_\star = p_{c_{\star}}(\vz|\vy)$. However, our framework can be applied to other choices of $\mathcal{R}$ and $\mathcal{L}$, such as those arising from generalised Bayesian inference~\citep{bissiri2016general}.

\textbf{A Gradient Flow in \texorpdfstring{$\sC\times \mathcal{P}_2(\sZ)$}{CxP2(Z)}}. Similar to how gradient flow $\dot{x}_\tau=-\nabla f(x_\tau)$ in $\mathbb{R}^d$ minimises a function $f:\R^d\to\R$ in the Euclidean space, we seek an analogue of this flow applicable to $\mathcal{F}[c,\mu]$, which features gradients obtained by endowing $\sC$ with the Euclidean geometry and $\mathcal{P}_2(\sZ)$ with the Wasserstein-2 one (cf.~\citet{kuntz23a} and Appendix~\ref{app-sec:background-w2}):
\begin{empheq}[left=\empheqlbrace]{align}
    \dot{c}_\tau
    &= -\nabla_c \mathcal{R}[c_\tau,\mu_\tau]
    \label{eq:prompt-cont-flow}
    \\
    \dot{\mu}_\tau
    &=
    -\mathrm{grad}_{W_2}\mathcal{R}[c_\tau,\mu_\tau]
    -\mathrm{grad}_{W_2}\mathcal{L}[\mu_\tau]
    \label{eq:distr-cont-flow}.
\end{empheq}
where $\mathrm{grad}_{W_2}\mathcal{F}$ is the Wasserstein-2 ($W_2$) gradient of a functional $\mathcal{F}$~\citep[Chapter 4]{figalli2021invitation} defined via the first variation~\citep[Section 7.2]{santambrogio2015optimal}: $\mathrm{grad}_{W_2}\mathcal{F}[\mu](\vz) = -\nabla \cdot \left(\mu \nabla_\vz ({\delta \mathcal{F}}/{\delta \mu})(\vz)\right)$. Along the flow $(c_\tau, \mu_\tau)$, the value of $\mathcal{F}[c_\tau,\mu_\tau]$ is non-increasing and we recover the optimal solution $(c_\star, \mu_\star)$ at stationarity~\citep[Theorem 2]{kuntz23a} (cf. Appendix~\ref{app:proof-prop-opt-cond} for the proof).
\begin{theorem}[Optimality Conditions]\label{thm:1st-opt-cond}
     Under the loss \eqref{eq:loss-choices}, the gradient $\nabla \mathcal{F}[c,\mu]:=(\nabla_c \mathcal{R}[c,\mu], \mathrm{grad}_{W_2}\mathcal{F}[c,\mu])$ vanishes if and only if we attain the posterior $\mu=p_{c_\star}(\vz|\vy)$ and the marginal likelihood stationary point $\nabla_c \log p_c(\vy)=0$, where $p_c(\vy)=\int p_c(\vx,\vy,\vz)\dd \vx \dd \vz$.
\end{theorem}
This shows that the objective $\mathcal{F}$ in \eqref{eq:loss-choices} recovers the optimal prompt $c^\star$ and posterior $\mu^\star$ in latent space, matching the MMLE solution of \eqref{eq:main_model}. Under mild conditions (e.g., strong log-concavity of $p_c(\vy,\vz)$), the flow converges exponentially fast to $(c^\star,\mu^\star)$ (see Appendix~\ref{app:proof-exp-conv}).

We now discretise the flow in time and space, yielding a fast algorithm for computing $(c_\star,\mu_\star)$.%

\textbf{Time Discretisation of the Gradient Flow \eqref{eq:prompt-cont-flow}-\eqref{eq:distr-cont-flow}}. As a first step to convert~\cref{eq:prompt-cont-flow,eq:distr-cont-flow} into an implementable algorithm, we write out the time-discretised version of the flow in this section. The main challenge in discretising the flow is the structure of the Wasserstein gradient in~\eqref{eq:distr-cont-flow}. To address this, we exploit the additive composition of the $W_2$ gradient in~\eqref{eq:distr-cont-flow} and consider a Lie-Trotter splitting of the gradient flow similar to those in~\citet{clement2011trotter,pmlr-v75-wibisono18a,pmlr-v75-bernton18a}. Together with the prompt update in~\eqref{eq:prompt-cont-flow}, we obtain the following splitting scheme for the flow where one-step updates are performed sequentially for the prompt and the distribution:
\begin{empheq}[left=\empheqlbrace]{align}
c_{\tau + h}
&= c_\tau - h \nabla_c \mathcal{R}[c_\tau, \mu_\tau]
\label{eq:prompt-split-flow}
\\
\bar\mu_{\tau+h}
&= S_{h}^{\mathcal{R},c_\tau} \mu_\tau
\label{eq:split-flow-R}
\\
\mu_{\tau+h}
&= S_{h}^\mathcal{L} \bar\mu_{\tau+h}
\label{eq:split-flow-L}.
\end{empheq}
where $S_t^{\mathcal{R}, c_\tau}$ and $S_t^\mathcal{L}$ are the semigroups associated with the Wasserstein gradient subflows of $\mathcal{R}$ and $\mathcal{L}$ respectively in~\eqref{eq:distr-cont-flow} with the prompt $c_\tau$ fixed. Under standard regularity assumptions, each subflow admits a Lagrangian representation via characteristics~\citep[Chapter~8]{ambrosio2008gradient}: $\dot{\vz}_\tau = -\nabla_\vz ({\delta \mathcal{G}}/{\delta \mu})(\vz_\tau)$, where $\vz_\tau \sim \mu_\tau$ for all $\tau\geq 0$ and $\mathcal{G}\in\{\mathcal{R}, \mathcal{L}\}$. However, updates \eqref{eq:prompt-split-flow}--\eqref{eq:split-flow-L} are still idealised. In order to implement these updates, we next look at efficient spatial discretisations for the two subflows.

%% file: sections/method.tex
\subsection{The prior subflow: kernelised entropy and CM-induced prior score}
\label{sec:prior-subflow}

We now make the $\mathcal R$-subflow in~\eqref{eq:split-flow-R} explicit. To spatially discretise the flow of $\mathcal{R}[c,\mu]=\KL(\mu\|p_c)$, we need to compute the velocity field $-\nabla \delta \mathcal{R}/\delta \mu$ (cf. Lemma~\ref{lemma:first_var_kl})
\begin{equation}
-\nabla_\vz \frac{\delta \mathcal R}{\delta \mu}(\vz)=
\nabla_\vz \log p_c(\vz)-\nabla_\vz \log \mu(\vz).
\label{eq:R-particle-drift}
\end{equation}
Direct simulation via~\eqref{eq:R-particle-drift} is often infeasible, as $\mu^N(\dd \vz):=N^{-1}\sum_{n} \delta_{\vz^{(n)}}(\dd \vz)$ is an empirical measure and $\nabla \log p_c(\vz)$ can be ill-defined in high dimensions~\citep{Song2019}. Instead, we consider approximating both $\nabla \log \mu$ and $\nabla \log p_c$ with a convex average of scores along the diffusion path. To this end, we first define the diffusion path for both $p_c$ and $\mu$ under the VP schedule $\dot\alpha_t=-\beta_t\alpha_t$~\eqref{eq:latent_forward_kernel}, with $\alpha_t\in(0,1]$ and $\sigma_t^2=1-\alpha_t$:
\begin{equation}
    p_{c,t}:=Q_t \ast p_c, \quad \mu_{t} := Q_t \ast \mu, \qquad 
    (Q_t \ast \mu)(\vz_t):= \int q(\vz_t|\vz_0) \mu(\vz_0) \dd\vz_0,
\end{equation}
where $q(\vz_t|\vz_0)=\mathcal N(\vz_t;\sqrt{\alpha_t}\vz_0,\sigma_t^2\mathbf I)$. For $r\in\{\mu,p_c\}$, the marginals $r_t:=Q_t\ast r$ are equal to the laws of the particles following the probability-flow ODEs $\dot{\vz}_s^r=-({\beta_s}/{2})(\vz_s^r+\nabla \log r_s(\vz_s^r;c))$, where $t\in [0,T]$ and $\vz_t^r=\vz_t\sim r_t$. Denoting $G_t^r(\vz_t):=\vz_0^r$ as the flow map obtained by integrating the PF-ODE~\eqref{eq:pf_ode_generic} from time $t$ to $0$, we have the following exact identity (proved in Appendix~\ref{app:proof-prop-exact-flow-map}).
\begin{proposition}[Exact flow-map identity]\label{prop:exact-flow-map-identity}
For every $t>0$ and $s\in[0,t]$, we have for any law $r\in \mathcal{P}_2(\sZ)$, and  $\vz_t\sim r_t(\vz_t)$
\begin{equation}
\frac{\sqrt{\alpha_t} G_t^r(\vz_t)-\vz_t}{1 - \sqrt{\alpha_t}}
=\int_0^t
\gamma_t(s) \nabla_{\vz_s^r} \log r_s(\vz_s^r)\,\dd s, \quad 
\gamma_t(s):=\frac{\sqrt{\alpha_t}}{1+\sqrt{\alpha_t}}
\frac{\beta_s}{2\sqrt{\alpha_s}},
\label{eq:canonical-averaged-score-identity}
\end{equation}
where $\gamma_t(s)\ge 0$ and $\int_0^t\gamma_t(s)\dd s=1$. In particular, this convex average of scores approximates $\nabla\log r(\vz_0^r)$ accurately at small $t$ (cf. Appendix~\ref{app:gaussian-cm-score-error} for a discussion).
\end{proposition}
Proposition~\ref{prop:exact-flow-map-identity} implies that we can approximate the intractable drift~\eqref{eq:R-particle-drift} by a flow-map discrepancy that can be evaluated by a CM
\begin{equation}
    -\widehat{\nabla_\vz\frac{\delta \mathcal{R}}{\delta \mu}}(\vz):= \frac{\sqrt{\alpha_t}}{1 - \sqrt{\alpha_t}}
\left(G_t^{p_c}(\vz_t)-G_t^\mu(\vz_t)\right) \approx \nabla_\vz \log p_c(\vz)-\nabla_\vz \log \mu(\vz).
\label{eq:flow-map-discrepancy}
\end{equation}
The prior origin map $G_t^{p_c}$ is approximated by the consistency model
$g_\theta(\vz_t,t,c)$. For the current particle law $\mu^N$, we approximate the map $G_t^{\mu^N}$ by the denoiser under the VP forward kernel $\mathbb{E}_\mu[\vz_0|\vz_t]\approx\widehat m_t^{\mu,N}(\vz_t)=\sum_{m=1}^N \pi_m(\vz_t;t)\vz_0^{(m)}$ with $\pi_m(\vz_t;t)\propto q(\vz_t|\vz_0^{(m)})$. Such an approximation has a bias of $\mathcal{O}(t)$ for small $t$~\citep{kim2023consistency}. 
Absorbing the factor $\sqrt{\alpha_t}/(1 - \sqrt{\alpha_t})$ into the step size
$\eta_t$, we obtain the practical update
\begin{equation}
\vz_{0,+}^{(n)}=\vz_0^{(n)}
+
\eta_t \left(g_\theta(\vz_t^{(n)},t,c) - \vz_0^{(n)}\right)
+
\eta_t
\left(
\vz_0^{(n)}
-
\sum_{m=1}^N
\pi_{nm}(t)\vz_0^{(m)}
\right),\label{eq:denoiser-discrepancy-update}
\end{equation}
where $\pi_{nm}(t):=\pi_m(\vz_t^{(n)};t)$ is the weight of the $m$-th particle in the denoiser estimate. The first two terms form a relaxation toward the CM output, while the final term is a kernelised repulsion away from the local barycentre of neighbouring particles.

%% file: sections/3-prior-method.tex
\subsection{A Score-Based Prompt Flow}
To jointly optimise the prompt $c$, we need to compute the gradient $\nabla_c \KL(\mu\|p_c)$ to implement~\eqref{eq:prompt-split-flow}. Directly differentiating this KL divergence through the prior density $p_c(\vz)$ would require expensive simulation of the reverse SDE/ODE~\citep{songscorebasedgenerativemodeling2021,skreta2025the}. Instead, under the same VP noising path as in the prior subflow and applying \citet[Theorem 2]{songMaximumLikelihoodTraining2021} gives the equivalent KL representation
\begin{equation}\label{eq:score_kl}
    \KL(\mu\|p_c) = \frac{1}{2} \int_{0}^{T} \beta(t) \DF(\mu_t\|p_{c,t}) \dd t + \KL(\mu_T\|p_{c,T}),
\end{equation}
where $T$ is the terminal diffusion time of the schedule in~\eqref{eq:latent_forward_kernel} and we denote the Fisher divergence as $\DF(\mu\|\nu):=\mathbb{E}_{\mu} [\left\| \nabla \log \mu({\vz}) - \nabla\log \nu(\vz) \right\|_{2}^{2}]$ with $\mu_t=Q_t\ast\mu$.
Moreover, using denoising score matching~\citep{Vincent2011,Song2019} and assuming the terminal mismatch $\KL(\mu_T\|p_{c,T})$ is negligible, we can estimate the gradient with
\begin{equation}\label{eq:prompt_gradient}
    \nabla_c \KL(\mu\|p_c) = \frac12 \nabla_c \int_0^T \beta(t)
    \mathbb{E}_{\vz_0\sim\mu}\mathbb{E}_{\vz_t\sim q_t(\cdot\mid\vz_0)} [\|\nabla_{\vz_t} \log q(\vz_t\mid\vz_0) - s_{\theta}(\vz_t,t;c)\|^2_2]\dd t,
\end{equation}
which is approximated using Monte-Carlo in practice. To evaluate the scores $s_{\theta}$, we use the score output from the LCM as a proxy by reparametrising $s_{\theta}(\vz,t;c) = -\sigma_t^{-1}\hat{\epsilon}_\theta(\vz,t;c)$. Since the LCM is initialised from the diffusion model teacher~\citep{songconsistencymodels2023,Luo2023LatentCM}, it can often provide good surrogates for score estimates (we refer to Appendix~\ref{sec:ablate-prompt-score} for a more detailed discussion). Our method is also directly applicable to Consistency Trajectory Models (CTM)~\citep{kim2023consistency}, which directly output score predictions.

\begin{wrapfigure}{r}{0.5\textwidth}
\vspace{-\intextsep}
\vspace{-0.3cm}
    \centering
    \includegraphics[width=1\linewidth]{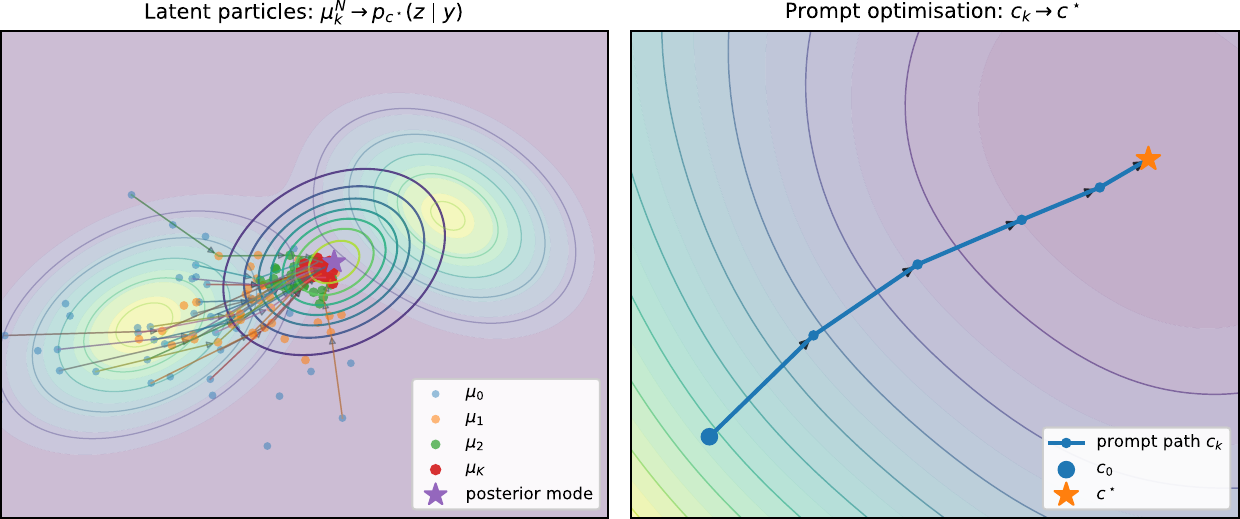}
    \caption{Schematic view of CWGF. The particle distribution $\mu_k^N$ is transported toward the posterior $p_{c^\star}(\vz\mid \vy)$, while the prompt embedding $c_k$ is updated toward the marginal-likelihood optimum $c^\star$.}
    \label{fig:schematic}
\vspace{-\intextsep}
\end{wrapfigure}

%% file: sections/4-likelihood-algo.tex
\subsection{An Encoder-based Likelihood Flow}\label{sec:method-likelihood}
We now consider the spatial discretisation of the flow of the likelihood part $\mathcal{L}[\mu]=-\mathbb{E}_\mu[\log p_{\phi^-}(\vy|\vz)]$ in~\eqref{eq:split-flow-L}. A standard calculation yields ${\delta \mathcal{L}}/{\delta \mu}(\vz) = -\log p(\vy|\vz)$~(cf. Lemma~\ref{lemma:semi-implicit}). Thus, the drift of the likelihood descent subflow is 
\(-\nabla_\vz \delta\mathcal{L}/\delta\mu (\vz) = \nabla_\vz\log p(\vy|\vz)\).
Using Fisher's identity to exchange the order of differentiation and integration, we can express this drift as
\begin{align}
    &g(\vz;\vy):=\nabla_{\vz} \log p(\vy|\vz)\nonumber \\ 
    & = \mathbb{E}_{p_c(\vx|\vz,\vy)}[\nabla_\vz \log p_{\phi^-}(\vx|\vz)],\label{eq:gzy-exact}
\end{align}
where we define $p_c(\vx|\vz,\vy)\propto p(\vy|\vx) p(\vx|\vz) p_c(\vz)$ as the pixel space posterior. However, computing~\eqref{eq:gzy-exact} would require backpropagating through the decoder, incurring high memory usage and potential artefacts~\citep{raphaeliSILOSolvingInverse2025}. Instead, we exploit the encoder $q_{\phi^-}(\vz|\vx)$ in the pretrained latent CM to obtain an approximation for $\nabla_\vz \log p(\vx|\vz)$ as illustrated in the following proposition (proved in Appendix~\ref{app:proof-prop-approx-delta-L}).
\begin{proposition}\label{prop:approx-delta-L}
    For a VAE trained with the objective $\mathcal{L}_{\mathrm{VAE}}(\phi) = \mathbb{E}_{q_\phi}\left[-\log p_\phi(\vx|\vz)\right] + \lambda \KL(q_\phi(\vz|\vx)\|p_0(\vz))$ with $\lambda>0$ and $p_0(\vz)$ being a standard Gaussian, we have the approximation to~\eqref{eq:gzy-exact} as:
    \begin{align}\label{eq:approx-encoder-score}
        \hat{g}(\vz;\vy):= \lambda \int [\nabla_\vz \log q_{\phi}(\vz|\vx)] p_c(\vx|\vz,\vy) \dd \vx - \lambda \nabla_\vz \log p_0(\vz).
    \end{align}
    Furthermore, the error of approximation $\varepsilon(\hat{g}):=\|g(\vz;\vy) - \hat{g}(\vz;\vy)\|_2^2$ satisfies $\varepsilon(\hat{g}) \leq \lambda^2\mathbb{E}\left[\|\nabla_\vz \log q^\star(\vz|\vx) - \nabla_\vz \log q_{\phi}(\vz|\vx)\|_2^2\right]$, where $q^\star(\vz|\vx) \propto p_0(\vz) \exp\left({\lambda}^{-1} \log p_{\phi}(\vx|\vz)\right)$ is the optimal encoder distribution for the VAE objective and the expectation is over $p_c(\vx|\vz,\vy)$.
\end{proposition}
We point out that the proposition above provides a theoretical justification for using the encoder to approximate the likelihood gradient and the encoder approximation error is often small for a well-trained VAE. To form a Monte Carlo/maximum a posteriori (MAP) estimate of the integral in~\eqref{eq:approx-encoder-score}, we require the pixel space posterior $p_c(\vx|\vz,\vy)$. In the specific case of linear inverse problems $p(\vy|\vx) = \mathcal{N}(\vy; A\vx, \sigma_\vy^2 I)$ and Gaussian decoder $p_{\phi^-}(\vx|\vz) = \mathcal{N}(\vx; \mathcal{D}_{\phi^-}(\vz), \sigma_{dec}^2 I)$, the posterior $p_c(\vx|\vz,\vy) = \mathcal{N}(\vx; m(\vz), \Sigma)$ is also Gaussian (cf. Appendix~\ref{app:proof-prop-posterior-gauss}) with $m(\vz)$ and $\Sigma$ given by:
\begin{align}\label{eq:mean-text}
    m(\vz) = \Sigma \left(\sigma_{dec}^{-2}\mathcal{D}_{\phi^-}(\vz) + \sigma_\vy^{-2}A^\top \vy\right), \qquad \Sigma = \left({\sigma_{dec}^{-2}}I + {\sigma_\vy^{-2}} A^\top A
    \right)^{-1}.
\end{align}
Therefore, under the MAP approximation, the flow of $\mathcal{L}$ in~\eqref{eq:split-flow-L} admits the following approximate particle drift:
\begin{align}\label{eq:approx-first-var-lik}
    \dot{\vz}_\tau=-\widehat{\nabla_\vz\frac{\delta \mathcal{L}}{\delta \mu}}(\vz_\tau) := \lambda\Sigma_{\phi^-}^{-1} (\mathcal{E}_{\phi^-}(m(\vz_\tau)) - \vz_\tau) + \lambda\vz_\tau,
\end{align}
where we have used that the prior and encoder have diagonal Gaussian forms, \textit{i.e.} $p_0(\vz)=\mathcal{N}(\vz;0,I)$ and $q_{\phi^-}(\vz|\vx)=\mathcal{N}(\vz;\mathcal{E}_{\phi^-}(\vx),\Sigma_{\phi^-}(\vx))$. We point out that~\eqref{eq:approx-first-var-lik} does not require backpropagation through either the CM or the VAE components, hence avoiding the potential ill-conditioning of decoder Jacobians (cf. Appendix~\ref{sec:ablate-encoder-approx-image} for further discussion). 
Concurrent to our work, \citet{shtanchaev2026gradientfree} also derived a similar objective for the likelihood gradient, albeit using heuristic arguments. We summarise the updates in~\Cref{algo:dwgf-prompt-simple}.
\vspace{0.5cm}

%% file: sections/5-experiments.tex
\newpage
\section{Experiments}
\label{sec:experiments}

\begin{wrapfigure}{r}{0.55\textwidth}
\vspace{-\intextsep}
\begin{minipage}{0.55\textwidth}
\vspace{-1.0em}
\begin{algorithm}[H]
\caption{Consistency-regularised Wasserstein Gradient Flow (CWGF)}
\label{algo:dwgf-prompt-simple}
\begin{algorithmic}[1]
\small
\State \textbf{Inputs:} $\vy$, particles $\{\vz_0^{(n)}\}_{n=1}^N$, prompt $c_0$
\For{$k=0,\ldots,K-1$}
    \State Sample $t=t(k)$ and $\varepsilon^{(n)}\sim\mathcal N(0,I)$
    \State $\vz_t^{(n)}\gets \sqrt{\alpha_t}\vz_k^{(n)}+\sigma_t\varepsilon^{(n)}$

    \State $c_{k+1}\gets c_k-\eta_c\,\widehat{\nabla_c\mathcal R}$ \hfill\eqref{eq:prompt_gradient}

    \State $\pi_{nm}(t)\propto q(\vz_t^{(n)}\mid \vz_k^{(m)})$,
    \State  $\widehat m_t^{\mu,N}(\vz_t^{(n)})\gets\sum_m\pi_{nm}(t)\vz_k^{(m)}$

    \State $\bar\vz_k^{(n)}\gets
    \vz_k^{(n)}
    +\eta_R \left(g_\theta(\vz_t^{(n)},t,c_k) - \vz_k^{(n)}\right)
    +\eta_{R}\bigl(\vz_k^{(n)}-\widehat m_t^{\mu,N}(\vz_t^{(n)})\bigr)$
    \hfill\eqref{eq:denoiser-discrepancy-update}

    \State Compute $m(\bar\vz_k^{(n)})$ as in~\eqref{eq:mean-text}

    \State $\widehat g(\bar\vz_k^{(n)};\vy)
    \gets
    \lambda\Sigma_{\phi^-}^{-1}
    \bigl(\mathcal E_{\phi^-}(m(\bar\vz_k^{(n)}))-\bar\vz_k^{(n)}\bigr)
    +\lambda\bar\vz_k^{(n)}$
    \hfill\eqref{eq:approx-first-var-lik}

    \State $\vz_{k+1}^{(n)}\gets \bar\vz_k^{(n)}+\eta_L\,\widehat g(\bar\vz_k^{(n)};\vy)$
\EndFor
\State \textbf{Outputs:} $\{\mathcal D_{\phi^-}(\vz_K^{(n)})\}_{n=1}^N$, $c_K$
\end{algorithmic}
\end{algorithm}
\vspace{-2.0em}
\end{minipage}
\vspace{-\intextsep}
\end{wrapfigure}

\paragraph{Datasets and prior model.}
We evaluate our method~\footnote{We release the code at \url{https://github.com/akyildiz-group/CWGF-release}.} on two high-quality datasets at resolution $512\times512$: FFHQ~\citep{karras2019stylebasedgeneratorarchitecturegenerative} and ImageNet~\citep{Deng2009ImageNetAL}. For FFHQ, we use the first $1$k test images, as in~\citet{chung2023diffusion}; for ImageNet, we use the \emph{cval1k} validation set introduced by~\citet{Larsson2016LearningRF}. Our prior is the LCM-LoRA~\citep{Luo2023LCMLoRAAU} distilled from Stable Diffusion~1.5. This is the same latent diffusion backbone used to evaluate diffusion-based baselines such as P2L~\citep{Chung2023PrompttuningLD}, PSLD~\citep{routsolvinglinearinverse2023}, LDPS, and LATINO/LATINO-PRO~\citep{Spagnoletti_2025_ICCV}.

\paragraph{Inverse problems.}
We consider the standard linear degradations used in~\citet{chung2023diffusion}. For Gaussian deblurring, we use a $61\times61$ kernel with $\sigma=3.0$. For motion deblurring, we use a $61\times61$ kernel randomly sampled with intensity $0.5$\footnote{\url{https://github.com/LeviBorodenko/motionblur}}. For super-resolution, we evaluate $\times8$ upscaling with bicubic interpolation. In all experiments, the measurements are corrupted with additive white Gaussian noise of standard deviation $\sigma_\vy=0.01$. Additional results on nonlinear inverse problems are reported in Appendix~\ref{app:non-linear} and we include extensive ablation experiments in Appendix~\ref{sec:ablations}.

\paragraph{Implementation details.}
We discretise the proposed gradient-flow dynamics with explicit Euler steps. Following the coarse-to-fine nature of diffusion sampling~\citep{wangprolificdreamerhighfidelitydiverse2023}, we use a two-stage timestep schedule for the prior flow: during the first half of the algorithm we traverse the full LCM timestep set $\{\texttt{999,879,759,639,499,379,259,139}\}$, while during the second half we cycle through the four lowest-noise timesteps to refine high-frequency details. We run CWGF for $16$ iterations. Since prompt and particle variables evolve on different scales~\citep{pavliotis2008,akyildiz2024multiscaleperspectivemaximummarginal}, we use separate step sizes for the prompt and particle updates, and precondition the likelihood gradient by $\lambda^{-1}\Sigma_{\phi^-}$; see~\cref{eq:approx-first-var-lik}. Full pseudo-code is given in Appendix~\ref{sec:algorithm}.

\begin{table*}[htbp]
\centering
\caption{
Results for Gaussian deblurring with $\sigma=3.0$, motion deblurring, and $\times8$ super-resolution, all with measurement noise $\sigma_\vy=0.01$, on FFHQ-512 and ImageNet-512 validation datasets. \textbf{Bold}: best within each dataset/problem; \underline{underline}: second best.
}
\label{tab:comparison_ffhq_imagenet}
\resizebox{\textwidth}{!}{%
\footnotesize
\setlength{\tabcolsep}{4pt}
\renewcommand{\arraystretch}{1.08}
\rowcolors{5}{gray!10}{white}
\begin{tabular}{l c ccc ccc ccc}
\toprule
 &  & \multicolumn{3}{c}{\textbf{Gaussian Deblur}} & \multicolumn{3}{c}{\textbf{Motion Deblur}} & \multicolumn{3}{c}{\textbf{SR $\times 8$}} \\
\cmidrule(lr){3-5}\cmidrule(lr){6-8}\cmidrule(lr){9-11}
\textbf{Method} & \textbf{NFE$\downarrow$}
& \textbf{FID$\downarrow$} & \textbf{PSNR$\uparrow$} & \textbf{LPIPS$\downarrow$}
& \textbf{FID$\downarrow$} & \textbf{PSNR$\uparrow$} & \textbf{LPIPS$\downarrow$}
& \textbf{FID$\downarrow$} & \textbf{PSNR$\uparrow$} & \textbf{LPIPS$\downarrow$} \\
\midrule

\rowcolor{gray!25}
\multicolumn{11}{l}{\textbf{FFHQ-512} \quad Prompt: \texttt{a photo of a face}} \\
LATINO & \textbf{8} & 40.01 & 26.27 & 0.352 & \underline{35.36} & \underline{25.19} & \underline{0.405} & \underline{50.34} & 25.25 & \underline{0.444} \\
LATINO-PRO & 65 & \underline{32.48} & \underline{28.49} & \underline{0.325} & 39.72 & 25.16 & 0.410 & 75.97 & \underline{26.00} & 0.445 \\
P2L & 400 & 45.12 & 26.47 & 0.433 & 55.73 & 23.92 & 0.503 & 52.14 & 23.76  & 0.485 \\
TReg & 200 & 48.73 & 24.27 & 0.418 & 44.97 & 22.82 & 0.432 & 62.40 & 22.68 & 0.460 \\
PSLD & 200 & 128.79 & 17.15 & 0.595 & 173.6 & 15.81 & 0.629 & 61.90 & 23.19 & 0.535 \\
LDPS & 200 & 173.37 & 15.12 & 0.644 & 219.26 & 14.19 & 0.670 & 50.52 & 23.89 & 0.491 \\
\rowcolor{orange!20}
CWGF & \underline{16} & \textbf{21.95} & \textbf{29.23} & \textbf{0.306} & \textbf{23.03} & \textbf{27.82} & \textbf{0.338} & \textbf{35.52} & \textbf{26.95} & \textbf{0.372} \\

\midrule
\rowcolor{gray!25}
\multicolumn{11}{l}{\textbf{ImageNet-512} \quad Prompt: \texttt{a photo in high resolution}} \\
LATINO & \textbf{8} & 52.95 & 23.43 & 0.406 & \underline{53.49} & 22.29 & \underline{0.468} & 73.90 & 22.53 & 0.502\\
LATINO-PRO & 65 & \underline{35.18} & \underline{25.28} & \textbf{0.357} & 67.44 & \underline{22.51} & 0.478 & 68.89 & \underline{23.06} & 0.498 \\
P2L & 400 & 59.77 & 23.05 & 0.487 & 159.32 & 14.29 & 0.683 & \underline{55.04} & 22.56 & \underline{0.468} \\
TReg & 200 & 56.54 & 22.38 & 0.441 & 78.75 & 20.99 & 0.501 & 75.91 & 20.85 & 0.504 \\
PSLD & 200 & 129.8 & 15.43 & 0.615 & 152.9 & 14.15 & 0.645 & 129.8 & 15.43 & 0.615 \\
LDPS & 1000 & 152.0 & 13.81 & 0.649 & 166.8 & 13.05 & 0.667 & 105.3 & 20.32 & 0.610 \\
\rowcolor{orange!20}
CWGF & \underline{16} & \textbf{28.19} & \textbf{25.40} & \underline{0.394} & \textbf{41.81} & \textbf{23.86} & \textbf{0.409} & \textbf{48.66} & \textbf{23.40} & \textbf{0.458} \\
\bottomrule
\end{tabular}
}
\end{table*}

\begin{wrapfigure}{l}{0.6\textwidth}
    \vspace{-1.5em}
    \centering
    \begin{minipage}{0.38\linewidth}
        \centering
        \includegraphics[width=\linewidth]{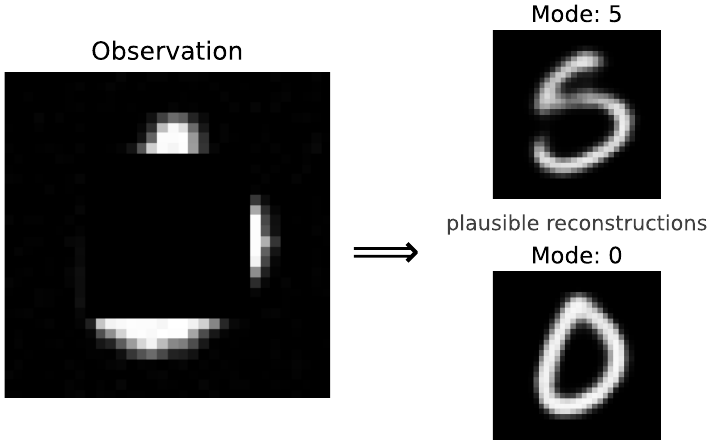}
    \end{minipage}
    \hfill
    \begin{minipage}{0.58\linewidth}
        \centering
        \includegraphics[width=\linewidth]{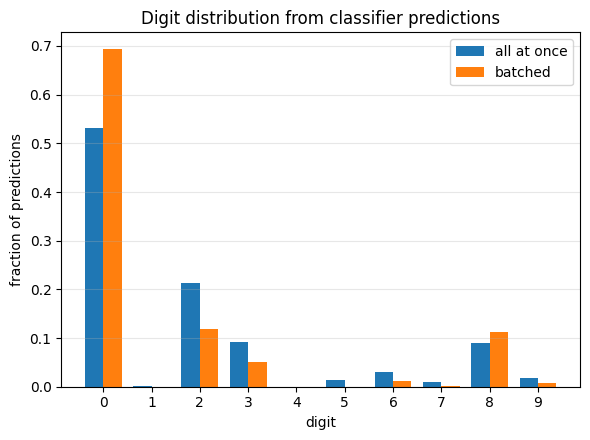}
    \end{minipage}
    \caption{Distribution of digits among the samples.}
    \label{fig:digits-hist-mnist}
    \vspace{-0.8em}
\end{wrapfigure}

\paragraph{Results.}
Table~\ref{tab:comparison_ffhq_imagenet} reports quantitative results on the selected inverse problems, while qualitative results are shown in Figure~\ref{fig:qualitative_ImageNet} for ImageNet, and in Figure~\ref{fig:qualitative_FFHQ} for FFHQ. All results are obtained with only $N=1$ particle. In Appendix~\ref{sec:ablations}, we report results with $N>1$, for which we observe mild improvements compared to the case $N=1$, as the high dimensionality of the latent space weakens the interactions. The results change when we run the algorithm on low-dimensional problems using the MNIST dataset~\cite{Deng2012TheMD}, where particle interactions are important. Indeed, as shown in Figure~\ref{fig:digits-hist-mnist}, where the problem solved was inpainting, the distribution of the particles exhibits greater variability when sampled \emph{all at once} (i.e., with $N>1$) than in the \emph{batched} case (i.e., in parallel with $N=1$). A more detailed analysis of this setting is discussed in Appendix~\ref{app:mnist}. We also include results on a synthetic experiment using tractable Gaussian mixture models to further validate the effect of particle interaction in Appendix~\ref{app_sec:synth}.

\begin{figure*}[htbp]
\centering
\begin{minipage}{0.03\textwidth}
  \centering
  \begin{tabular}{c}
    \\[-8mm]
    \rotatebox{90}{\scalebox{.6}{\textbf{Gaussian deblur}}} \\[2mm]
    \rotatebox{90}{\scalebox{.6}{\textbf{Motion deblur}}}  \\[8mm]
    \rotatebox{90}{\scalebox{.6}{\textbf{SRx8}}}
  \end{tabular}
\end{minipage}
\hfill
\begin{minipage}{0.13\textwidth}
    \centering \tiny\textbf{GT} \\[-1mm]
    \zoomedImage[width=\linewidth]{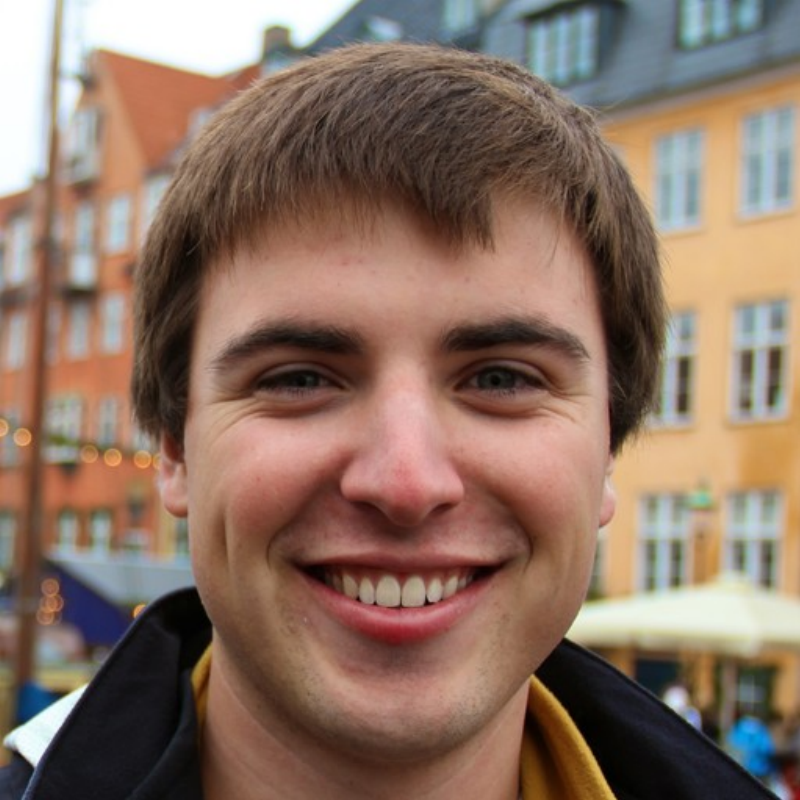}{-0.2,0.1}{1.0,0.5} \\[-1mm]
    \zoomedImage[width=\linewidth]{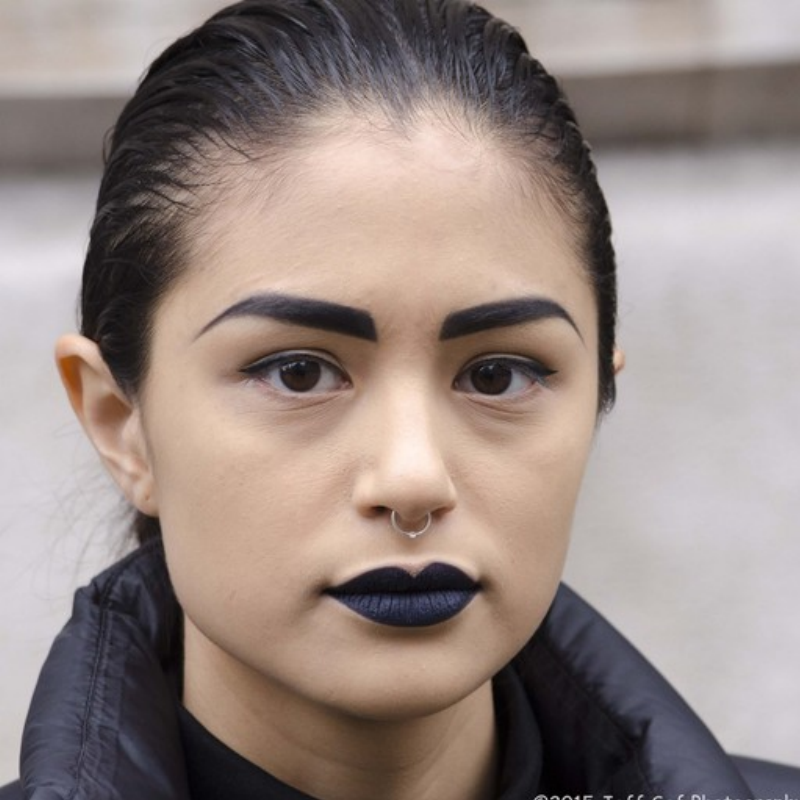}{-0.2,0.1}{1.0,0.5} \\[-1mm]
    \zoomedImage[width=\linewidth]{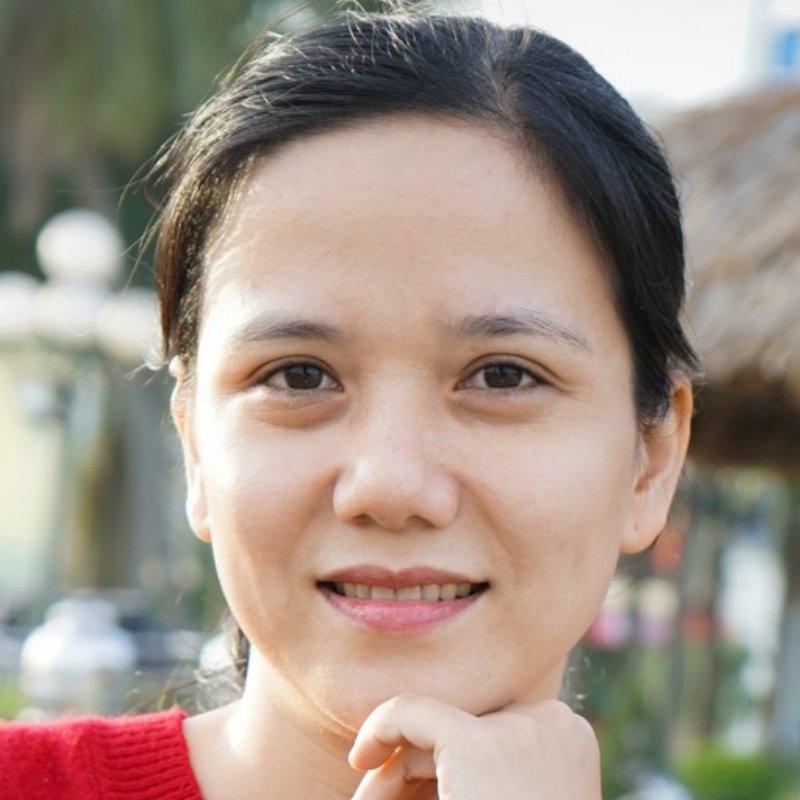}{-0.2,0.1}{1.0,-0.5} \\[-1mm]
\end{minipage}
\hfill
\begin{minipage}{0.13\textwidth}
  \centering \tiny\textbf{Measurement} \\[-1mm]
    \zoomedImage[width=\linewidth]{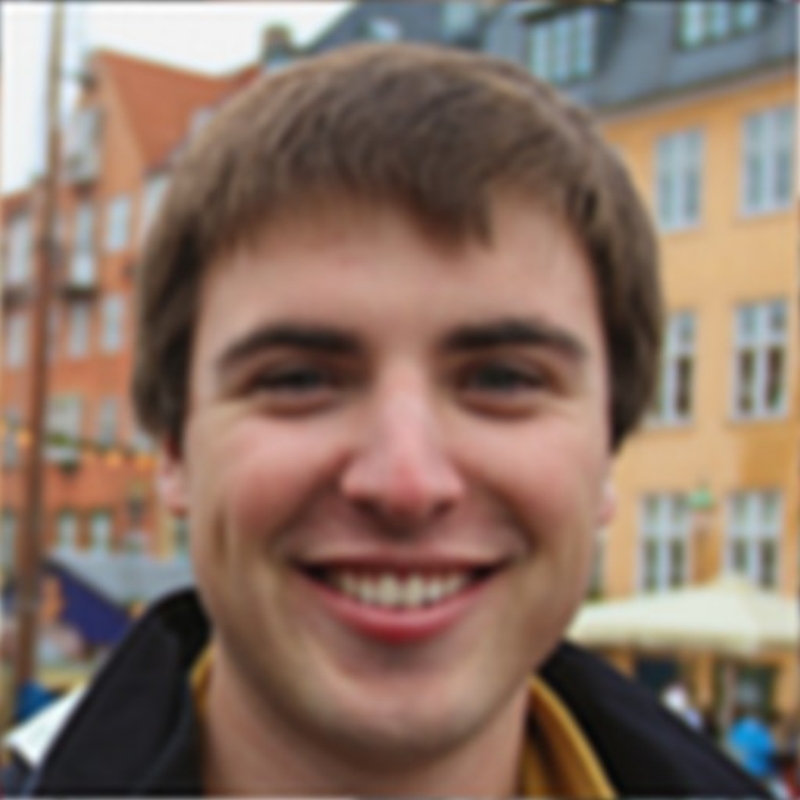}{-0.2,0.1}{1.0,0.5}\\[-1mm]
    \zoomedImage[width=\linewidth]{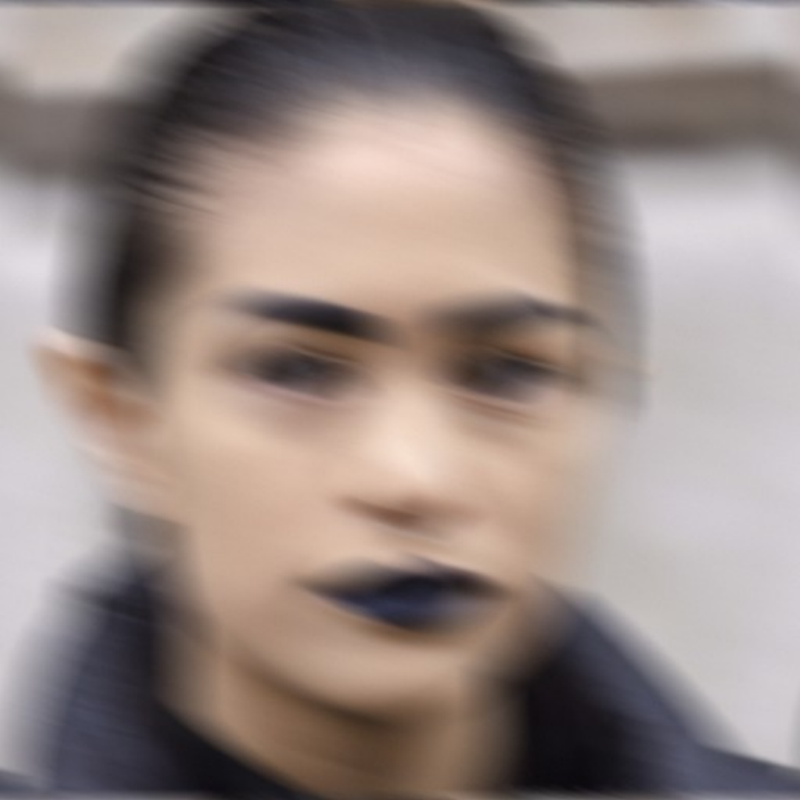}{-0.2,0.1}{1.0,0.5}\\[-1mm]
    \zoomedImage[width=\linewidth]{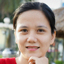}{-0.2,0.1}{1.0,-0.5} \\[-1mm]
\end{minipage}
\hfill
\begin{minipage}{0.13\textwidth}
    \centering \tiny\textbf{CWGF} \\[-1mm]
    \zoomedImage[width=\linewidth]{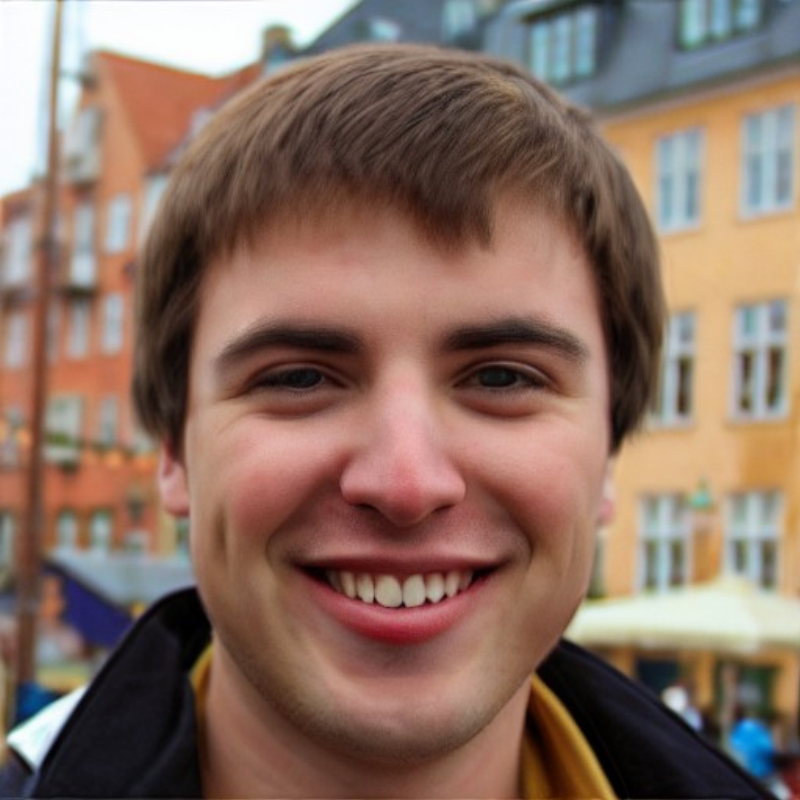}{-0.2,0.1}{1.0,0.5}\\[-1mm]
    \zoomedImage[width=\linewidth]{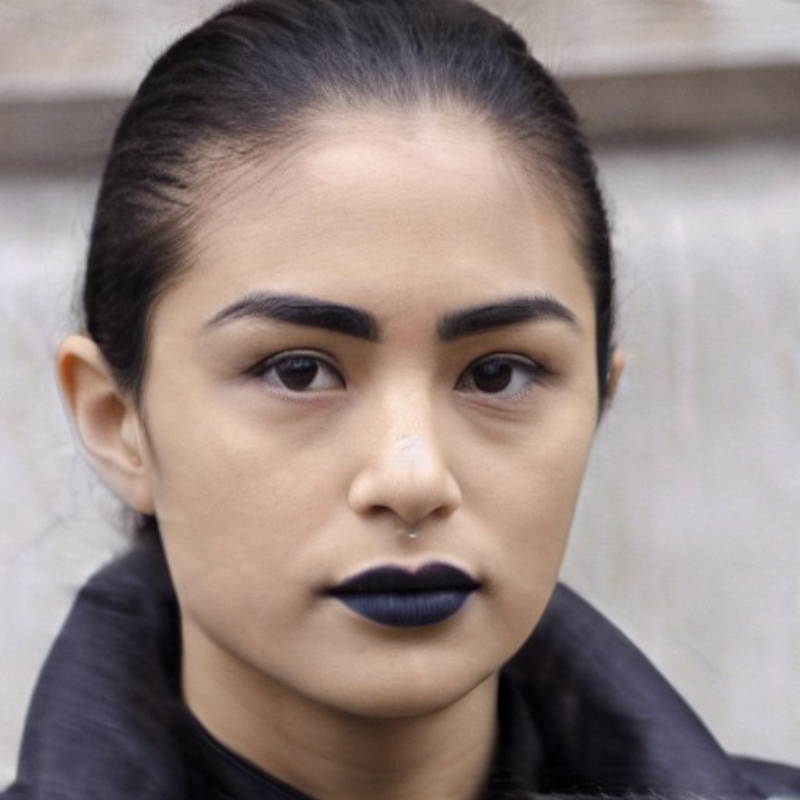}{-0.2,0.1}{1.0,0.5}\\[-1mm]
    \zoomedImage[width=\linewidth]{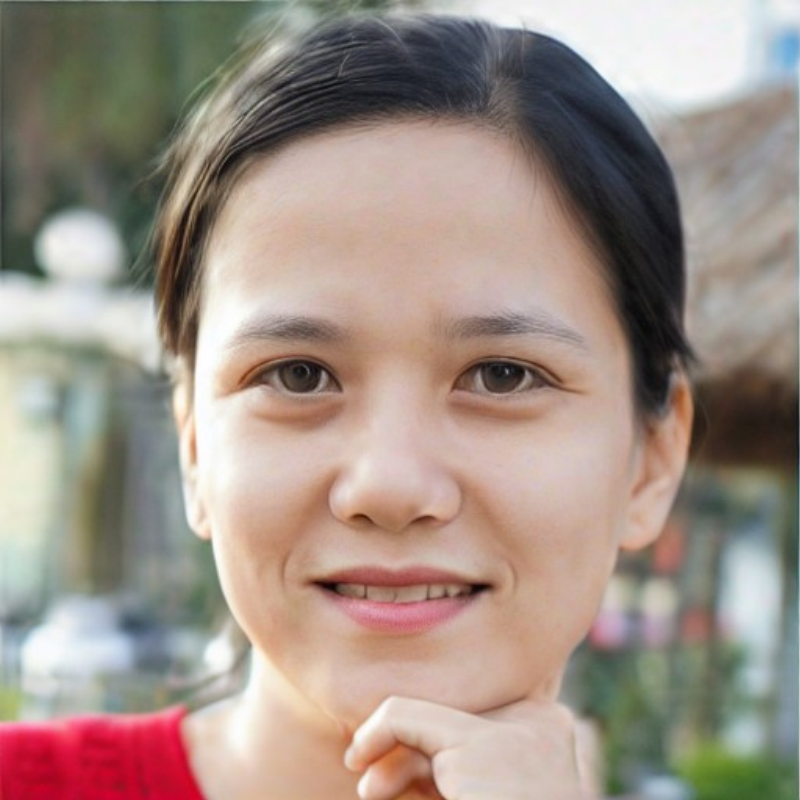}{-0.2,0.1}{1.0,-0.5} \\[-1mm]
\end{minipage}
\hfill
\begin{minipage}{0.13\textwidth}
    \centering \tiny\textbf{LATINO} \\[-1mm]
    \zoomedImage[width=\linewidth]{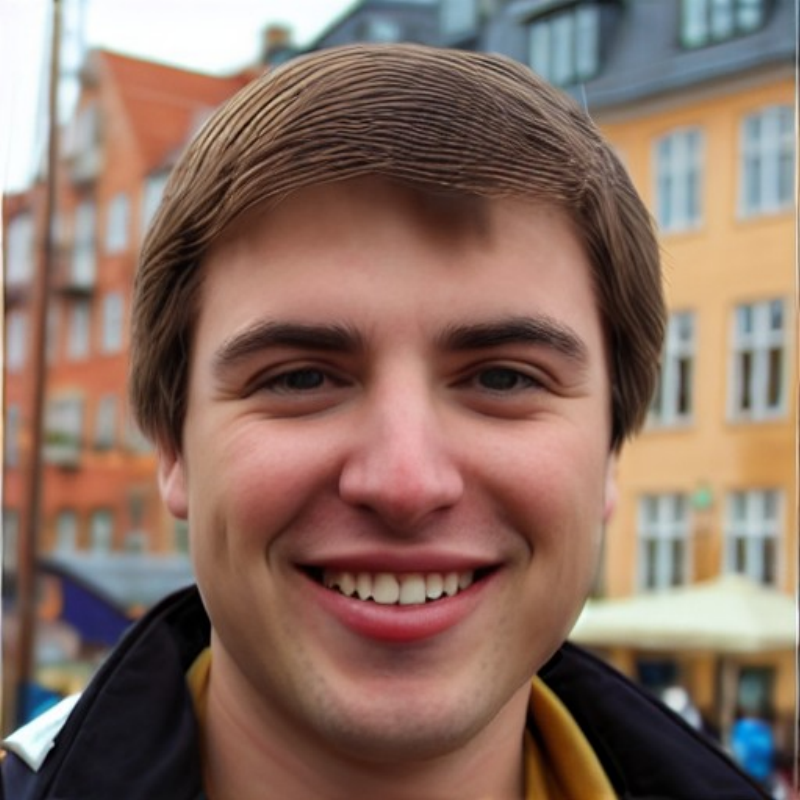}{-0.2,0.1}{1.0,0.5} \\[-1mm]
    \zoomedImage[width=\linewidth]{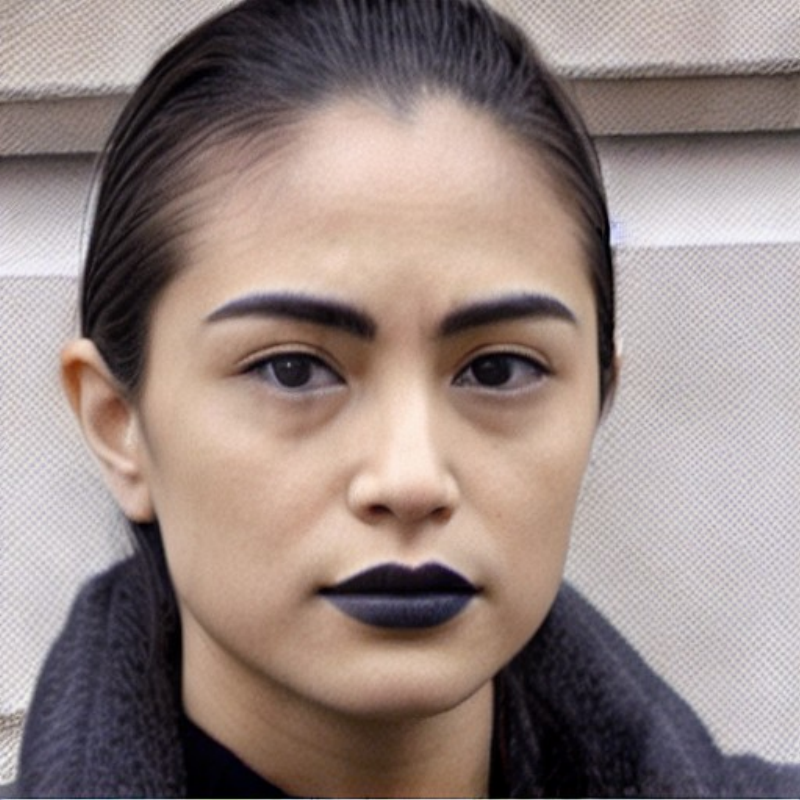}{-0.2,0.1}{1.0,0.5} \\[-1mm]
    \zoomedImage[width=\linewidth]{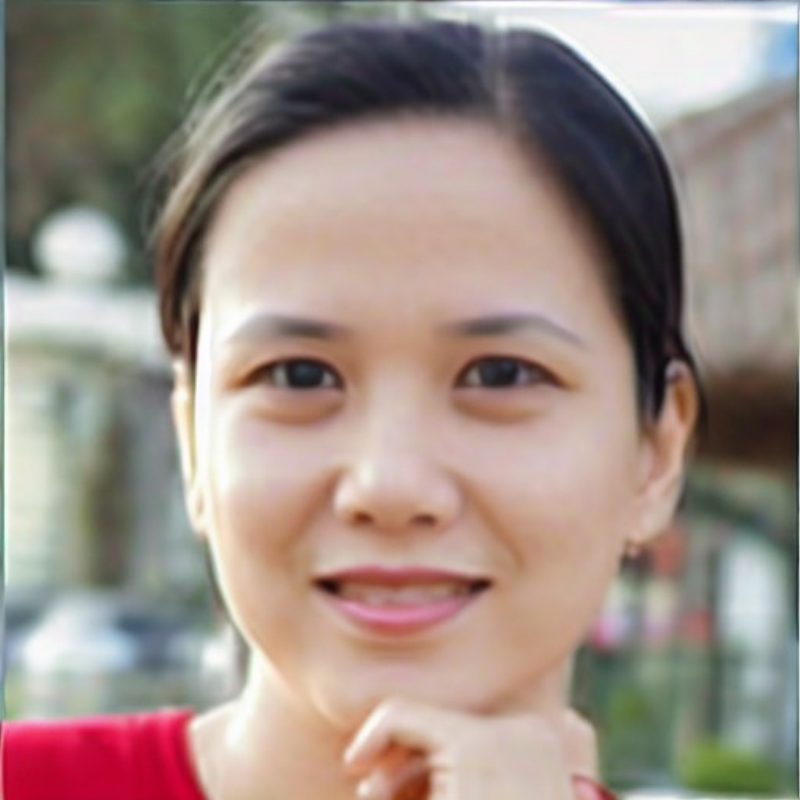}{-0.2,0.1}{1.0,-0.5} \\[-1mm]
\end{minipage}
\hfill
\begin{minipage}{0.13\textwidth}
    \centering \tiny\textbf{TREG} \\[-1mm]
    \zoomedImage[width=\linewidth]{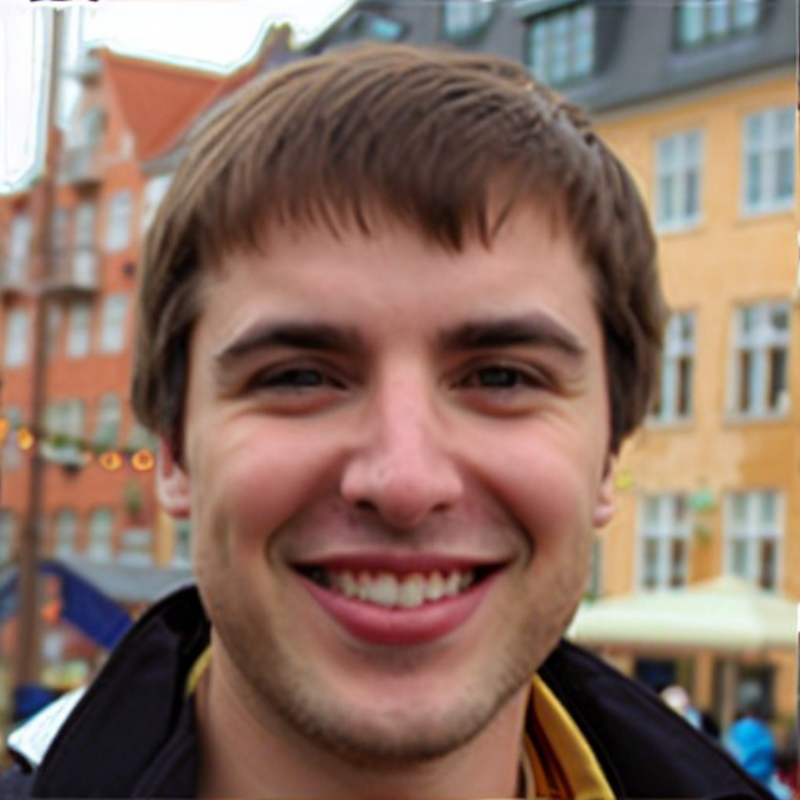}{-0.2,0.1}{1.0,0.5} \\[-1mm]
    \zoomedImage[width=\linewidth]{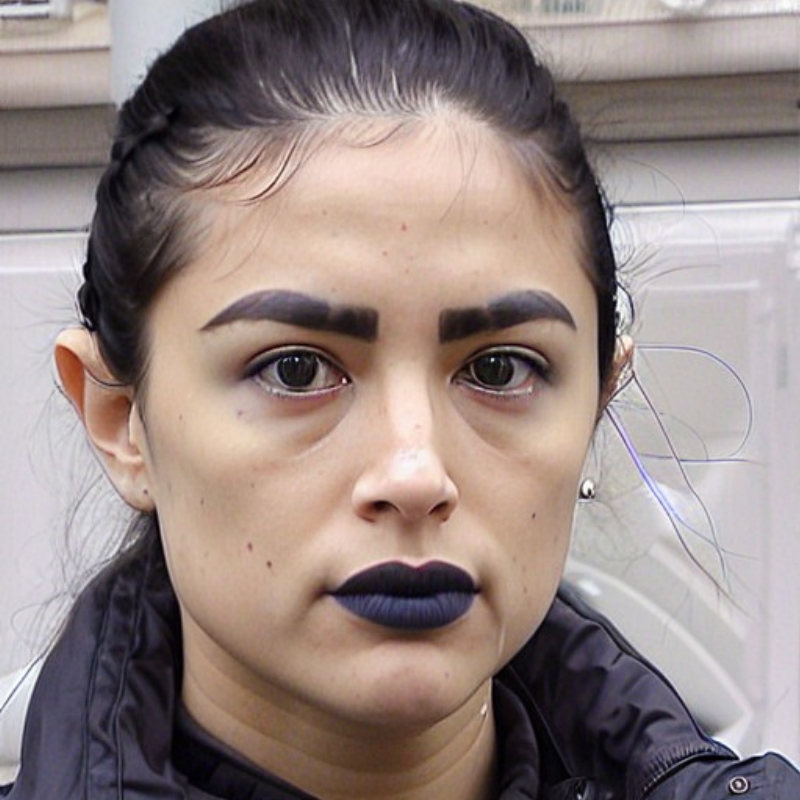}{-0.2,0.1}{1.0,0.5} \\[-1mm]
    \zoomedImage[width=\linewidth]{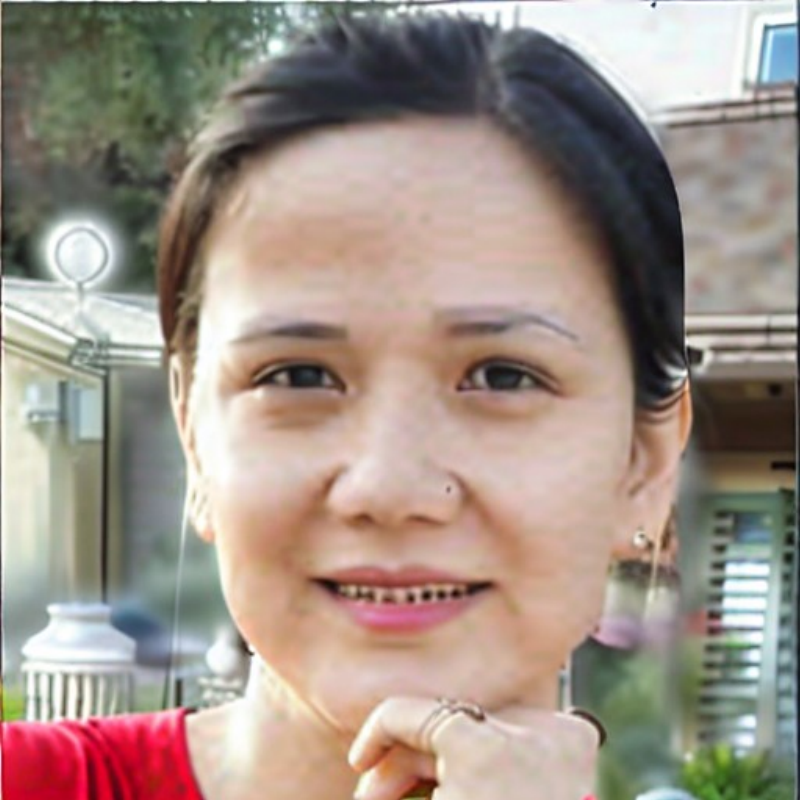}{-0.2,0.1}{1.0,-0.5} \\[-1mm]
\end{minipage}
\hfill
\begin{minipage}{0.13\textwidth}
    \centering \tiny\textbf{P2L} \\[-1mm]
    \zoomedImage[width=\linewidth]{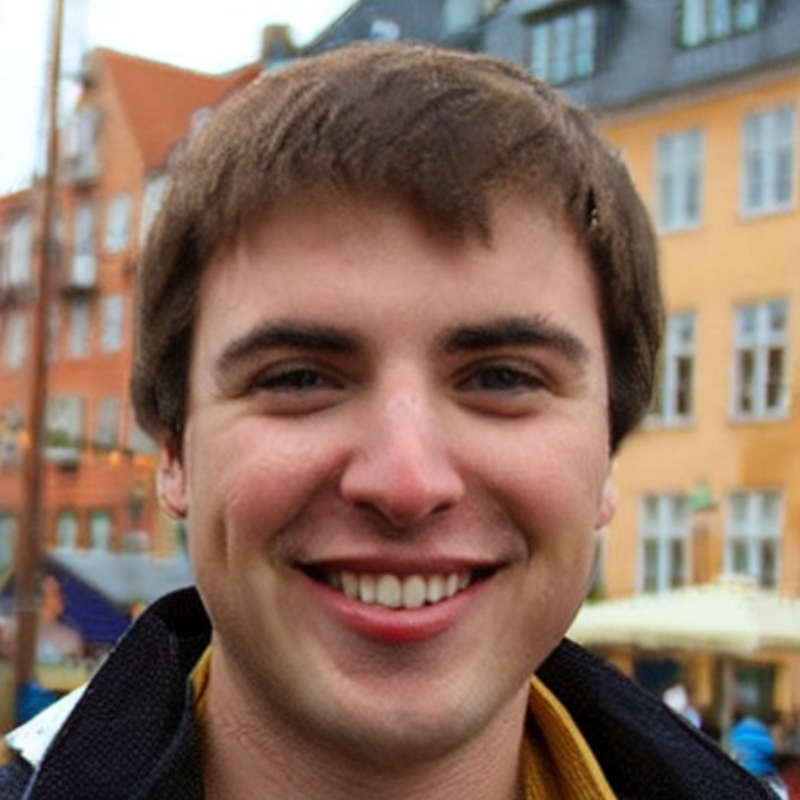}{-0.2,0.1}{1.0,0.5} \\[-1mm]
    \zoomedImage[width=\linewidth]{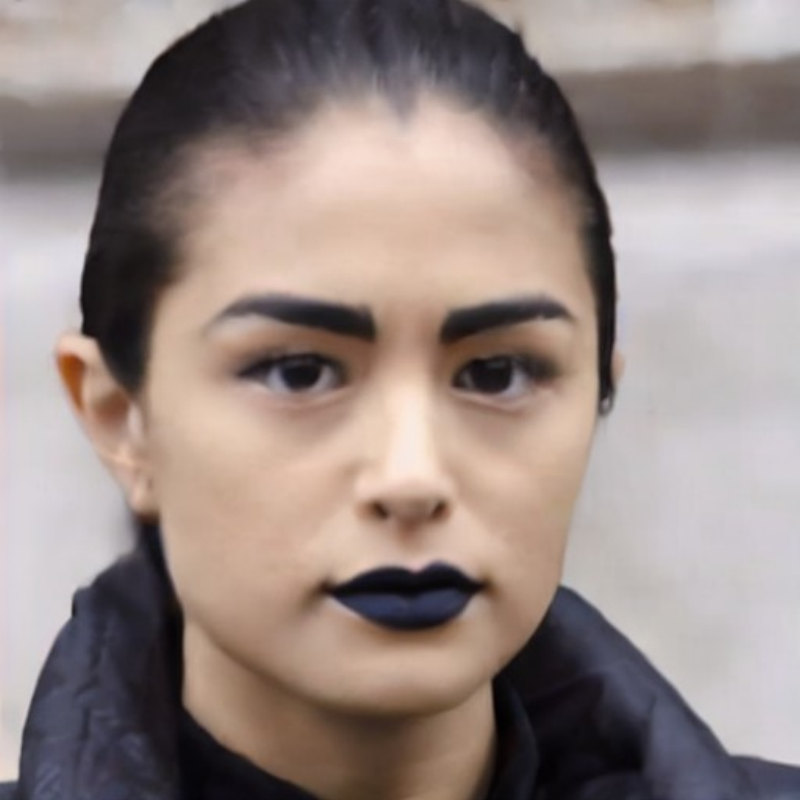}{-0.2,0.1}{1.0,0.5} \\[-1mm]
    \zoomedImage[width=\linewidth]{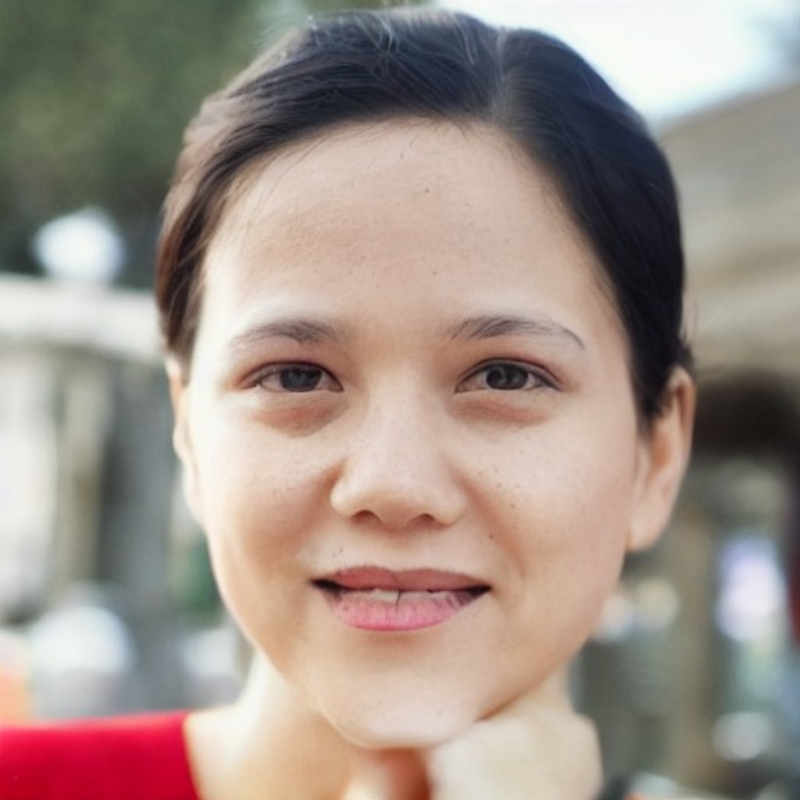}{-0.2,0.1}{1.0,-0.5} \\[-1mm]
\end{minipage}
\hfill
\begin{minipage}{0.13\textwidth}
    \centering \tiny\textbf{PSLD} \\[-1mm]
    \zoomedImage[width=\linewidth]{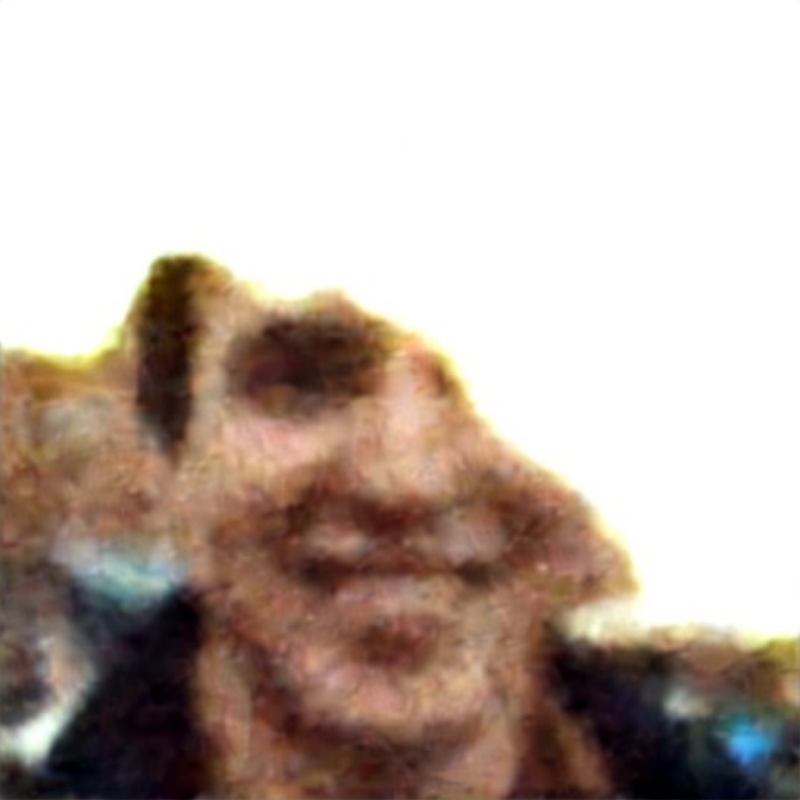}{-0.2,0.1}{1.0,0.5} \\[-1mm]
    \zoomedImage[width=\linewidth]{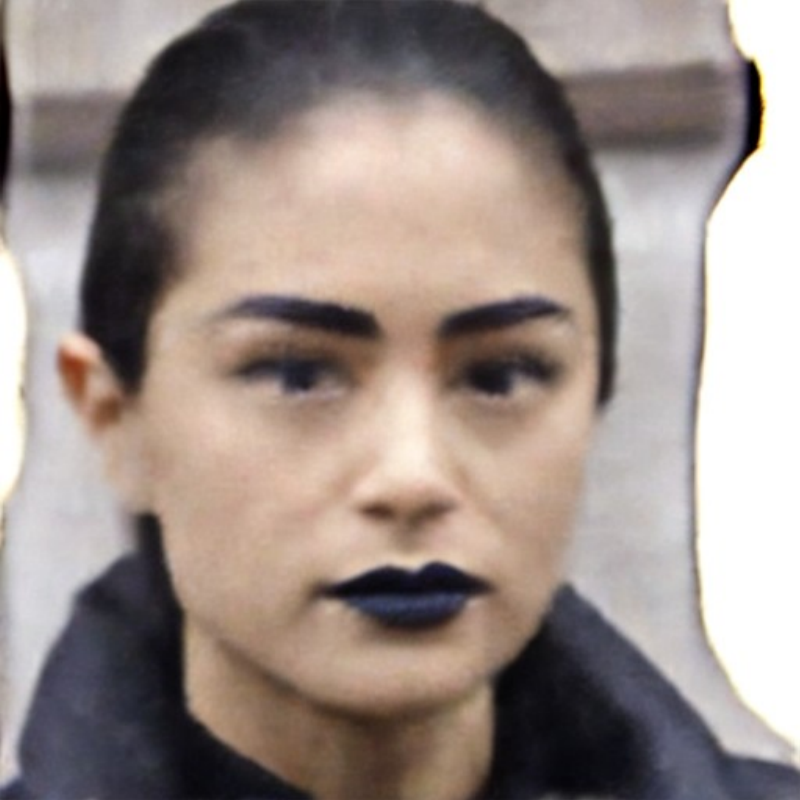}{-0.2,0.1}{1.0,0.5} \\[-1mm]
    \zoomedImage[width=\linewidth]{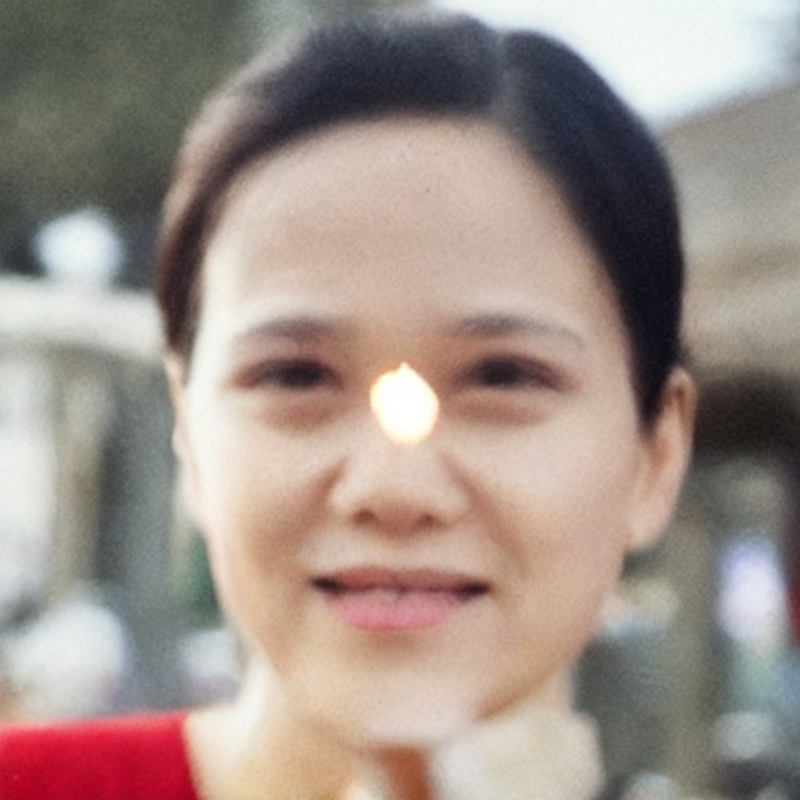}{-0.2,0.1}{1.0,-0.5} \\[-1mm]
\end{minipage}
\vspace{-2mm}
\caption{Qualitative comparison of image restoration results. Samples taken from FFHQ-512. Prompt: \texttt{A photo of a face}.}
\label{fig:qualitative_FFHQ}
\end{figure*}

\begin{figure*}[htbp]
\centering
\begin{minipage}{0.03\textwidth}
  \centering
  \begin{tabular}{c}
    \\[-8mm]
    \rotatebox{90}{\scalebox{.6}{\textbf{Gaussian deblur}}} \\[1mm]
    \rotatebox{90}{\scalebox{.6}{\textbf{Motion deblur}}}  \\[7mm]
    \rotatebox{90}{\scalebox{.6}{\textbf{SRx8}}}
  \end{tabular}
\end{minipage}
\hfill
\begin{minipage}{0.13\textwidth}
    \centering \tiny\textbf{GT} \\[-1mm]
    \zoomedImage[width=\linewidth]{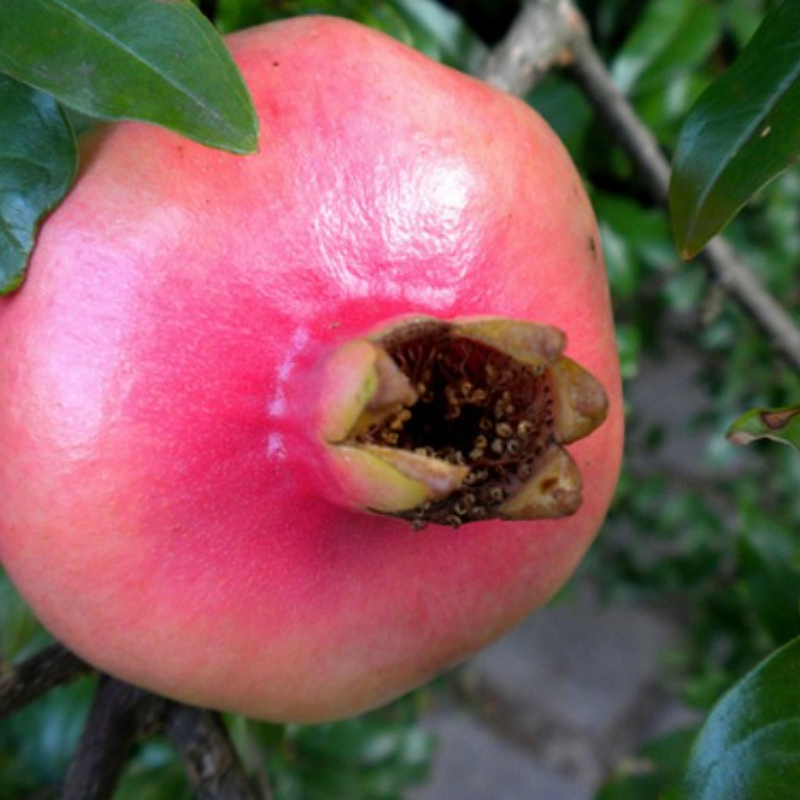}{-0.3,0.6}{1.0,0.5} \\[-1mm]
    \zoomedImage[width=\linewidth]{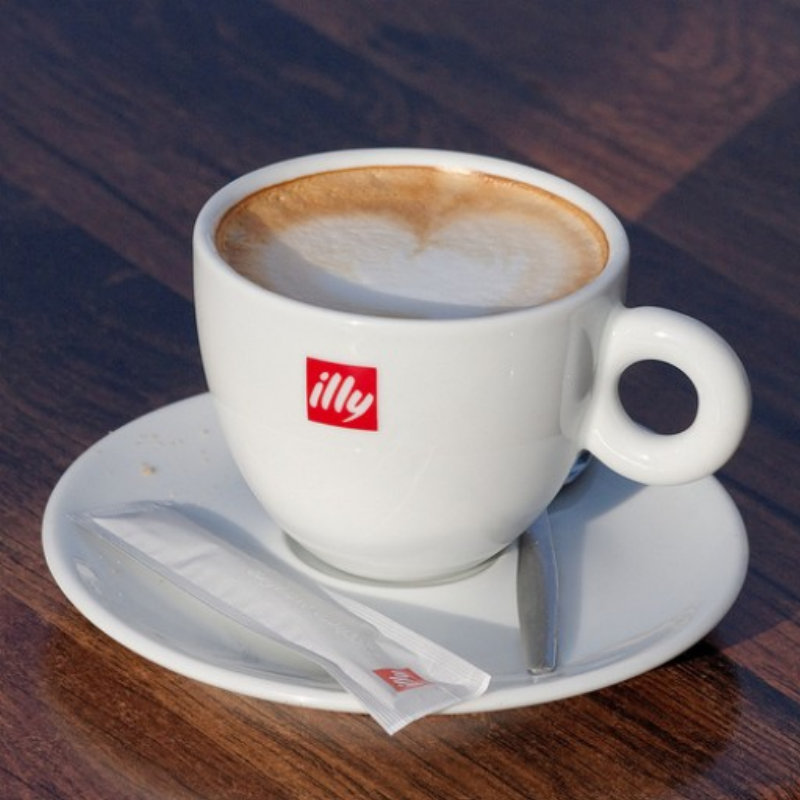}{-0.2,0.1}{1.0,0.5} \\[-1mm]
    \zoomedImage[width=\linewidth]{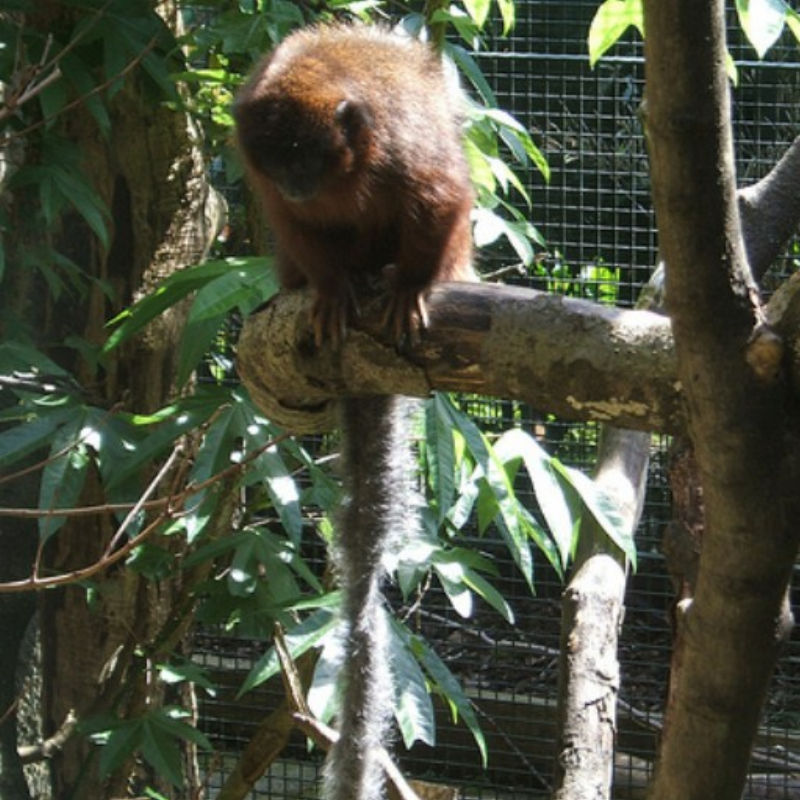}{-0.2,0.6}{1.0,-0.5} \\[-1mm]
\end{minipage}
\hfill
\begin{minipage}{0.13\textwidth}
  \centering \tiny\textbf{Measurement} \\[-1mm]
    \zoomedImage[width=\linewidth]{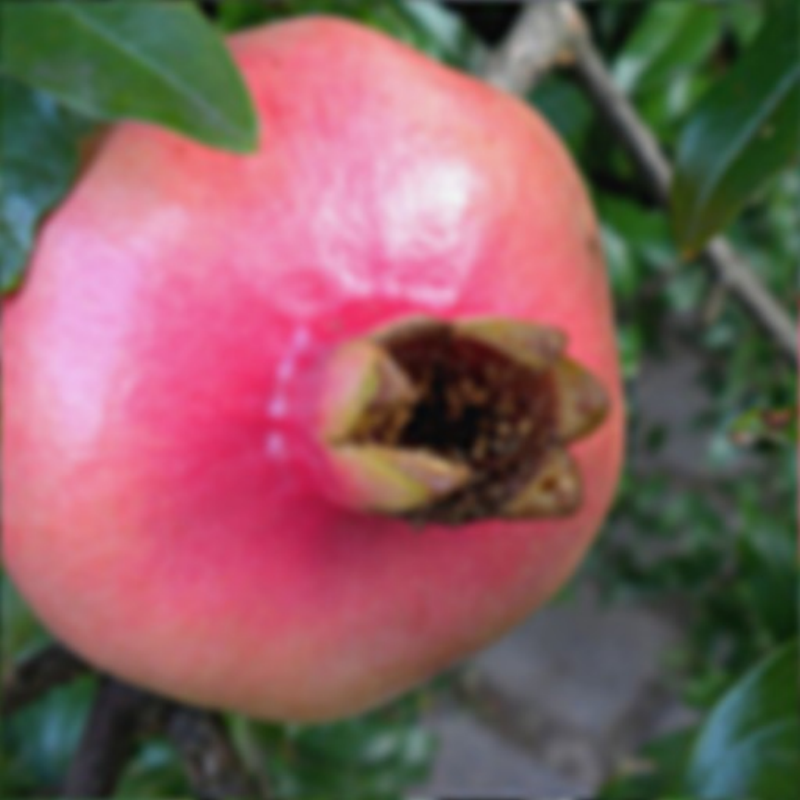}{-0.3,0.6}{1.0,0.5}\\[-1mm]
    \zoomedImage[width=\linewidth]{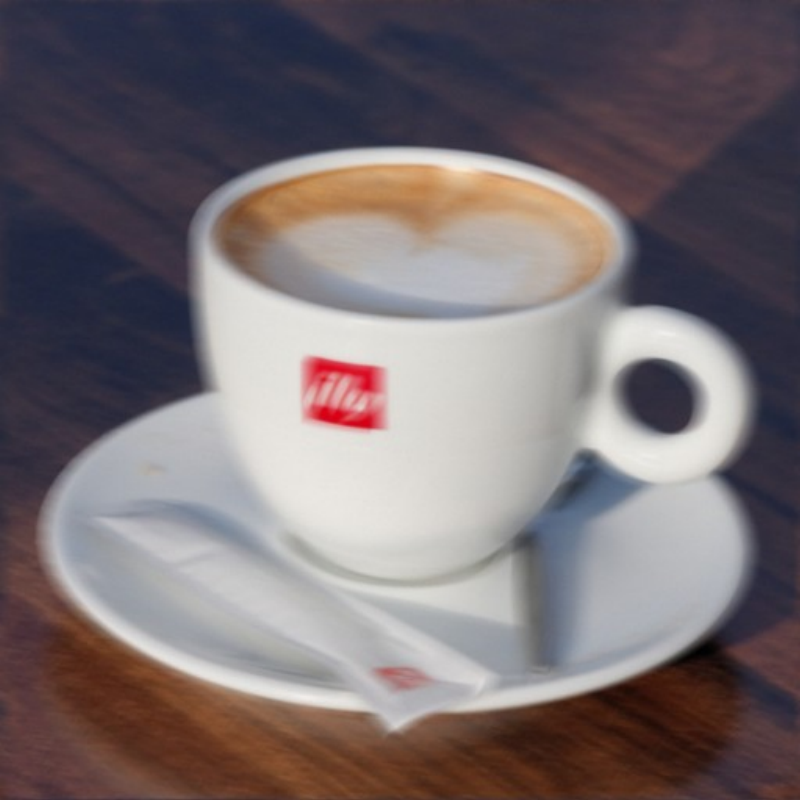}{-0.2,0.1}{1.0,0.5}\\[-1mm]
    \zoomedImage[width=\linewidth]{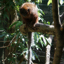}{-0.2,0.6}{1.0,-0.5} \\[-1mm]
\end{minipage}
\hfill
\begin{minipage}{0.13\textwidth}
    \centering \tiny\textbf{CWGF} \\[-1mm]
    \zoomedImage[width=\linewidth]{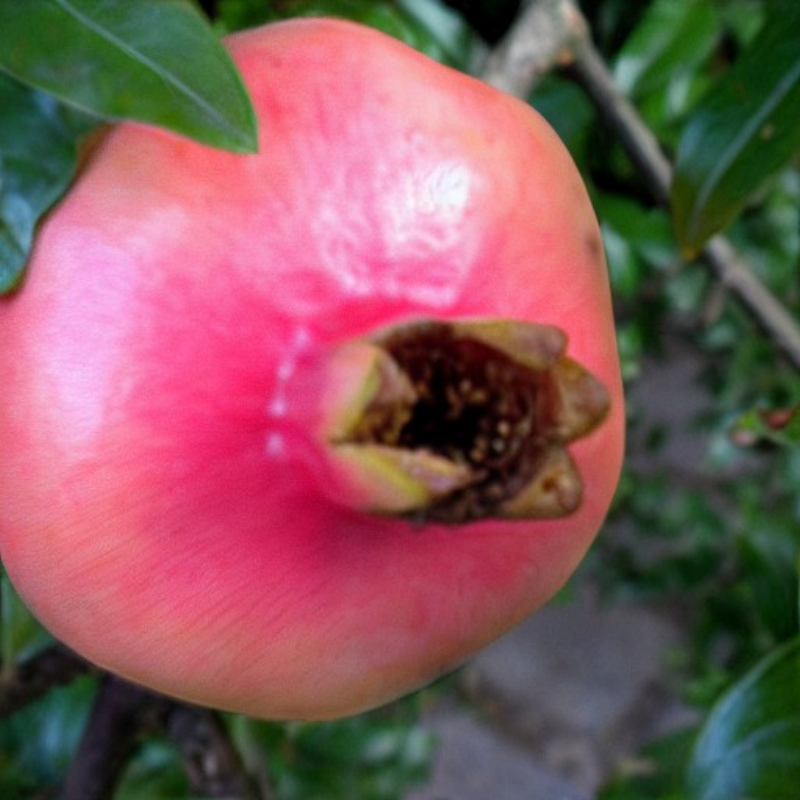}{-0.3,0.6}{1.0,0.5}\\[-1mm]
    \zoomedImage[width=\linewidth]{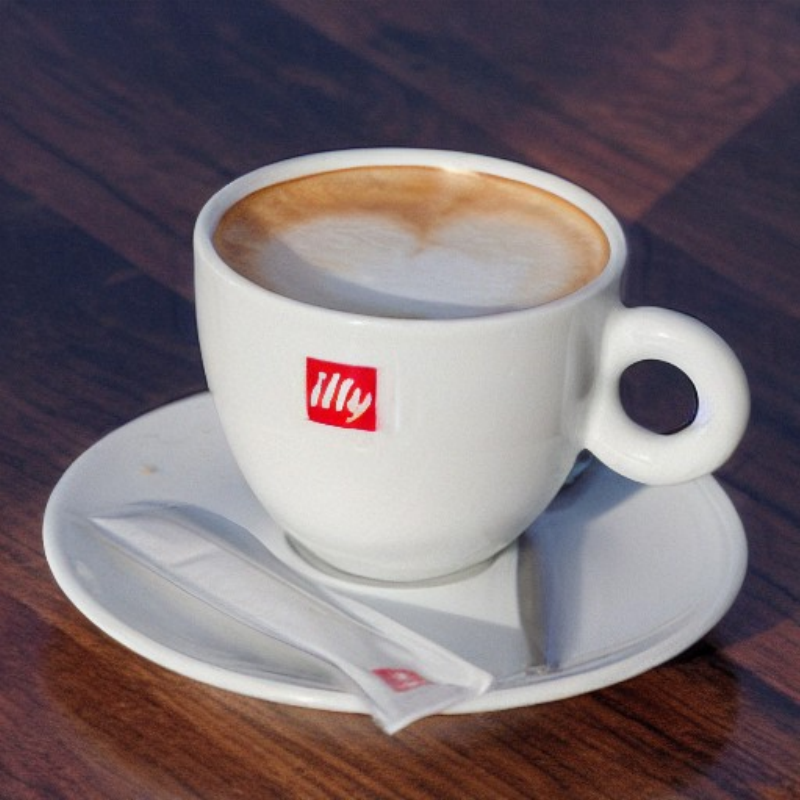}{-0.2,0.1}{1.0,0.5}\\[-1mm]
    \zoomedImage[width=\linewidth]{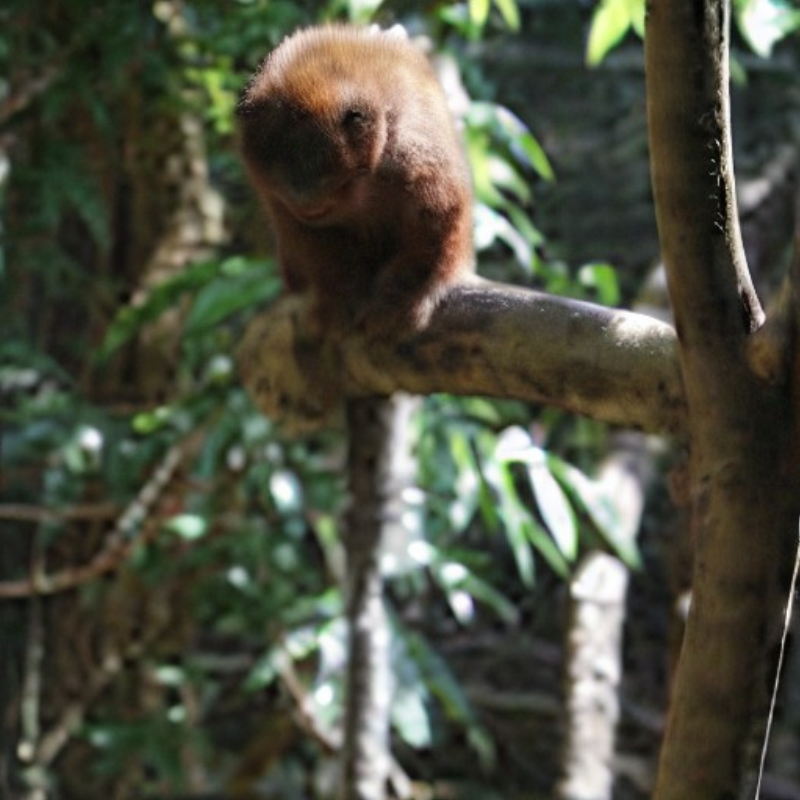}{-0.2,0.6}{1.0,-0.5} \\[-1mm]
\end{minipage}
\hfill
\begin{minipage}{0.13\textwidth}
    \centering \tiny\textbf{LATINO} \\[-1mm]
    \zoomedImage[width=\linewidth]{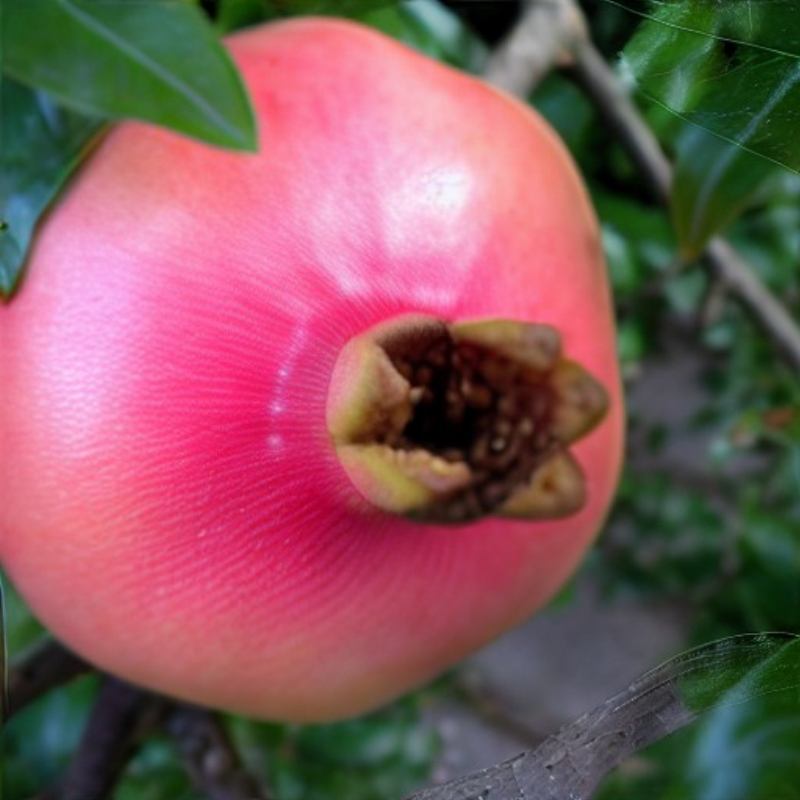}{-0.3,0.6}{1.0,0.5} \\[-1mm]
    \zoomedImage[width=\linewidth]{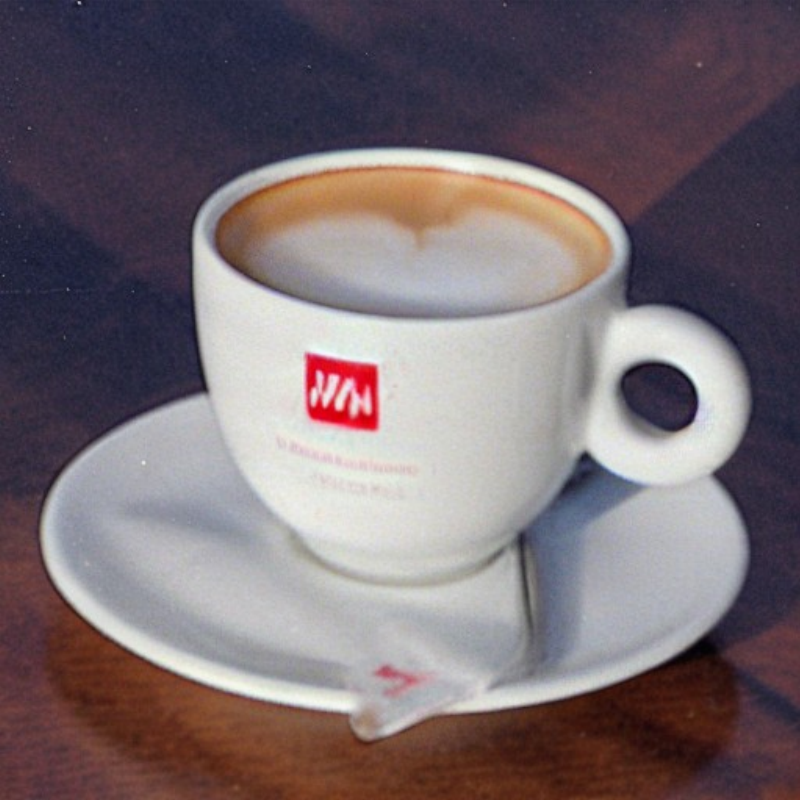}{-0.2,0.1}{1.0,0.5} \\[-1mm]
    \zoomedImage[width=\linewidth]{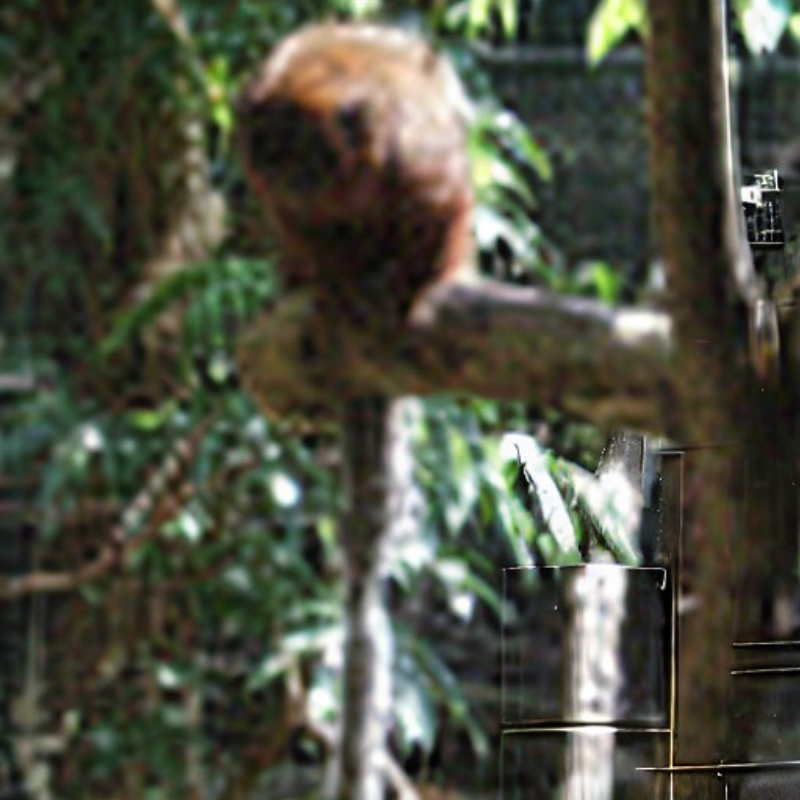}{-0.2,0.6}{1.0,-0.5} \\[-1mm]
\end{minipage}
\hfill
\begin{minipage}{0.13\textwidth}
    \centering \tiny\textbf{LATINO-P} \\[-1mm]
    \zoomedImage[width=\linewidth]{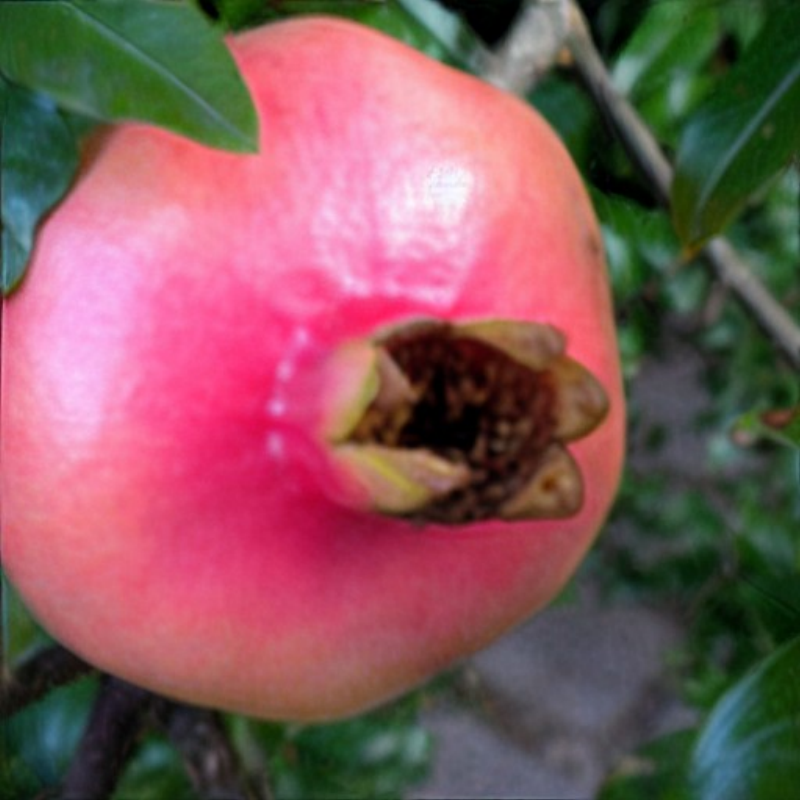}{-0.3,0.6}{1.0,0.5} \\[-1mm]
    \zoomedImage[width=\linewidth]{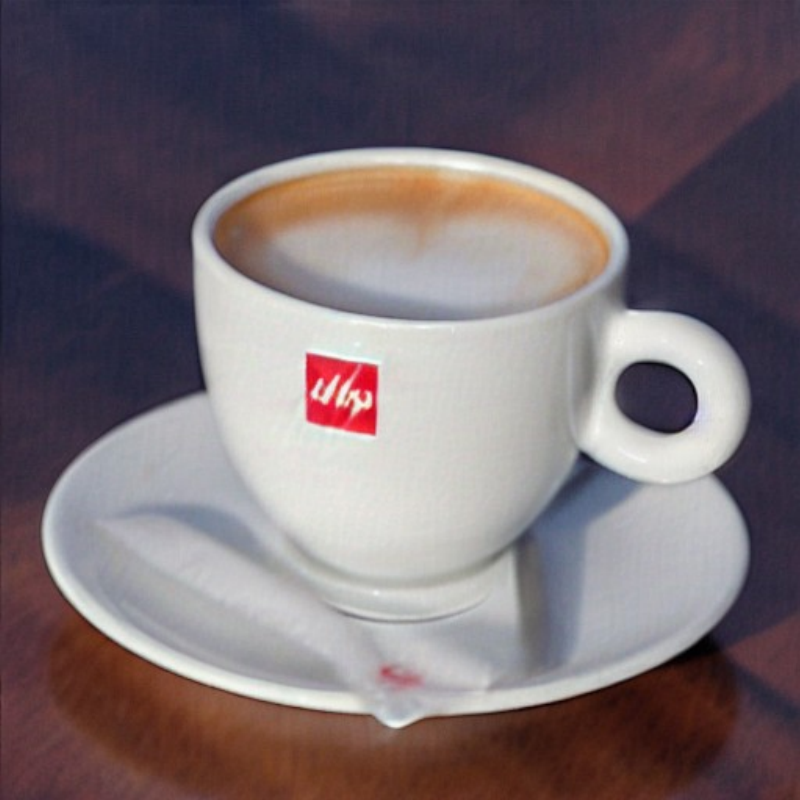}{-0.2,0.1}{1.0,0.5} \\[-1mm]
    \zoomedImage[width=\linewidth]{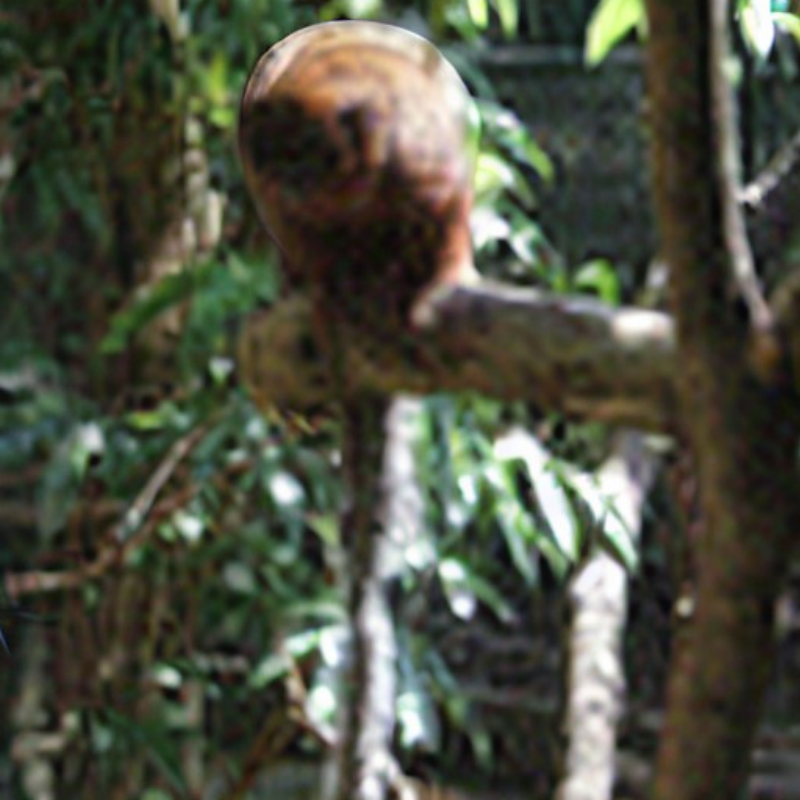}{-0.2,0.6}{1.0,-0.5} \\[-1mm]
\end{minipage}
\hfill
\begin{minipage}{0.13\textwidth}
    \centering \tiny\textbf{TREG} \\[-1mm]
    \zoomedImage[width=\linewidth]{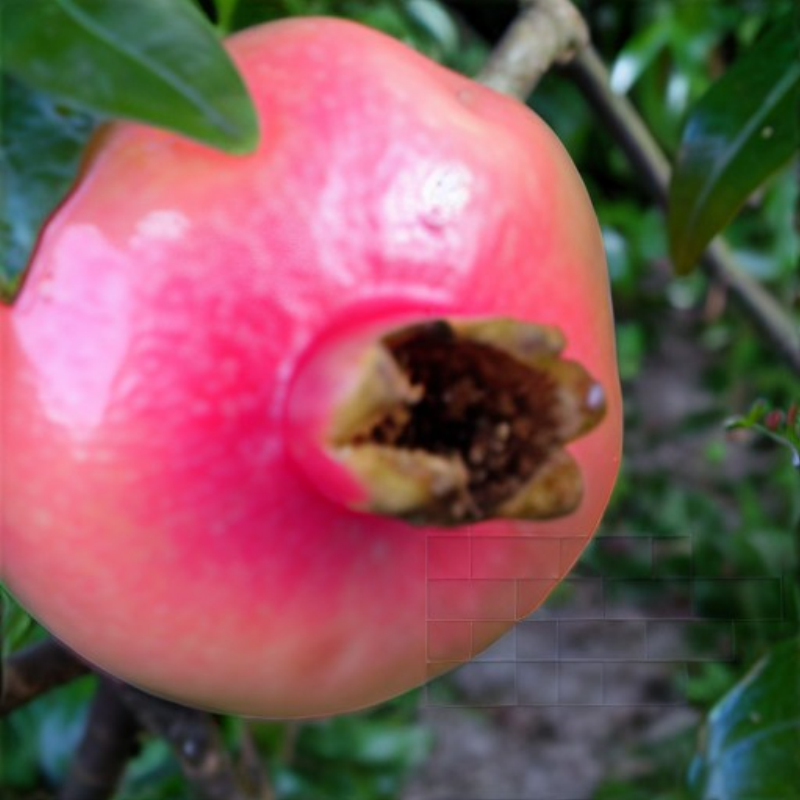}{-0.3,0.6}{1.0,0.5} \\[-1mm]
    \zoomedImage[width=\linewidth]{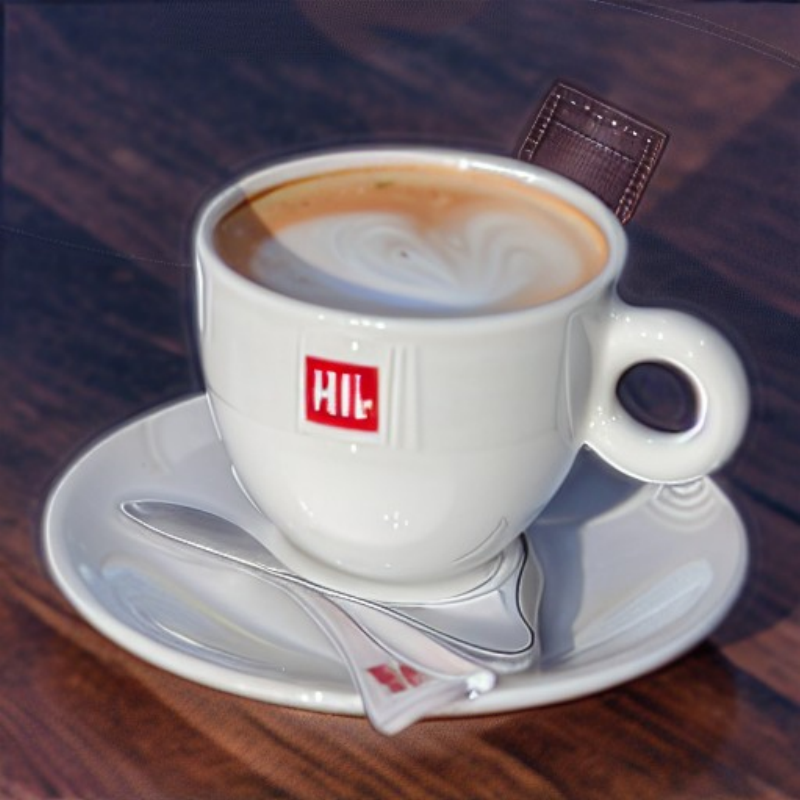}{-0.2,0.1}{1.0,0.5} \\[-1mm]
    \zoomedImage[width=\linewidth]{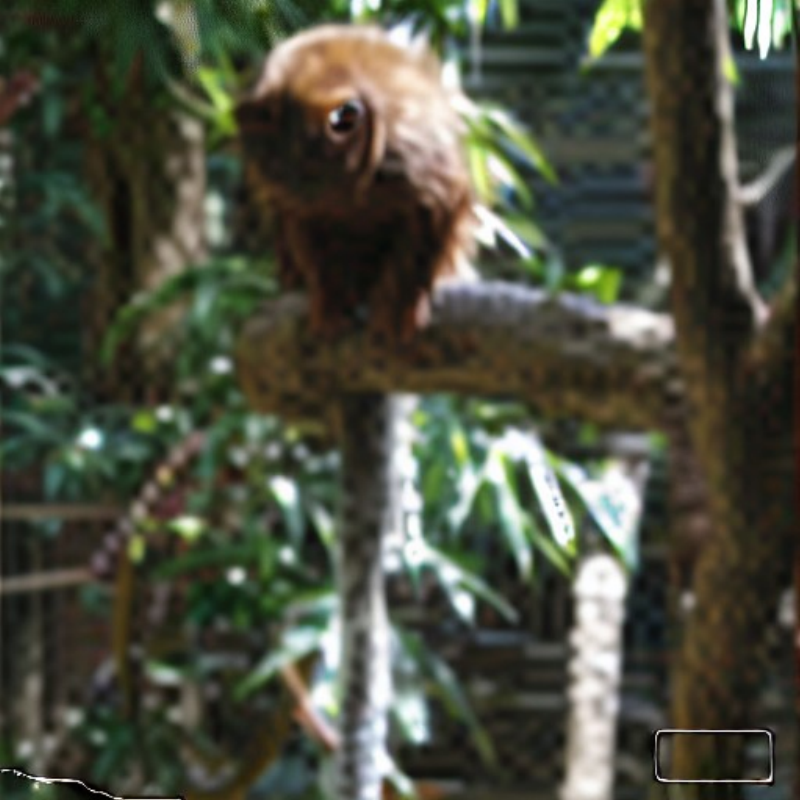}{-0.2,0.6}{1.0,-0.5} \\[-1mm]
\end{minipage}
\hfill
\begin{minipage}{0.13\textwidth}
    \centering \tiny\textbf{P2L} \\[-1mm]
    \zoomedImage[width=\linewidth]{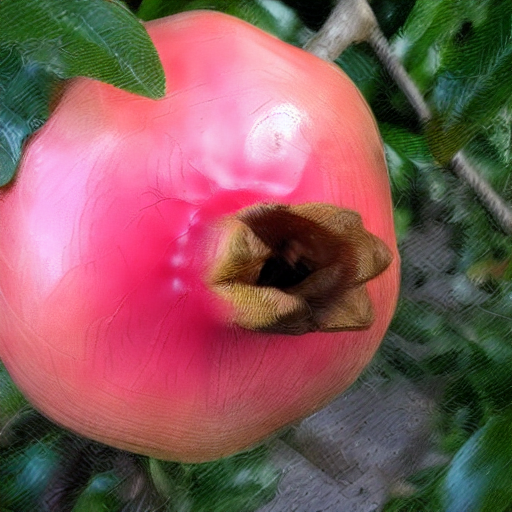}{-0.3,0.6}{1.0,0.5} \\[-1mm]
    \zoomedImage[width=\linewidth]{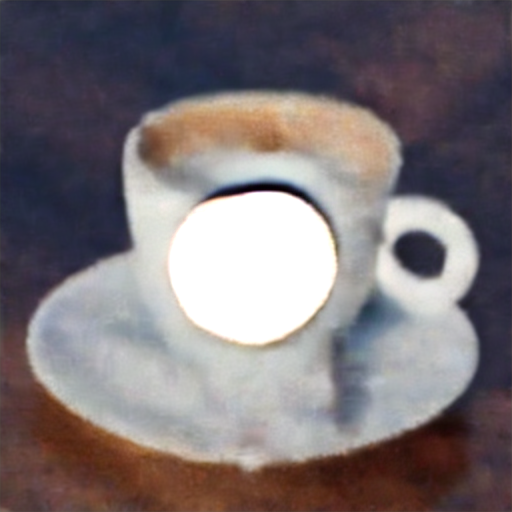}{-0.2,0.1}{1.0,0.5} \\[-1mm]
    \zoomedImage[width=\linewidth]{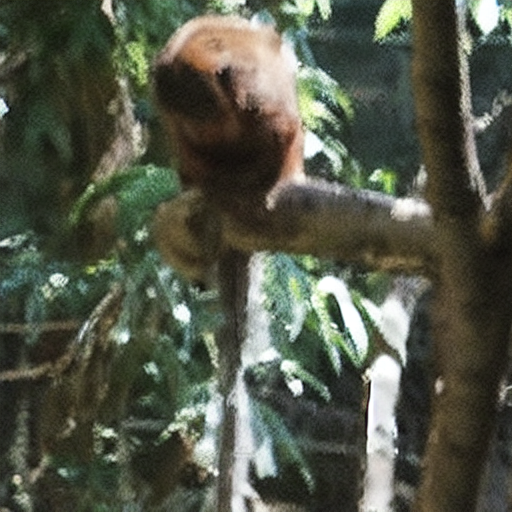}{-0.2,0.6}{1.0,-0.5} \\[-1mm]
\end{minipage}
\vspace{-2mm}
\caption{Qualitative comparison of image restoration results. Samples taken from ImageNet-512. Prompt: \texttt{a photo in high resolution}.}
\label{fig:qualitative_ImageNet}
\end{figure*}

\paragraph{Drifting a wrong prompt.}
To verify that our joint posterior-prompt optimisation is effective, we conduct experiments on the FFHQ test dataset, in which the initial prompt is misinitialised as \texttt{A photo of a cat}. We compare the method's results when prompt optimisation is disabled (CWGF (no prompt opt.) row in Table~\ref{tab:comparison_FFHQ_neg_prompt}) with those of the full CWGF algorithm.
\begin{table}[!htb]
\centering
\scriptsize
\setlength{\tabcolsep}{2.2pt}
\renewcommand{\arraystretch}{1.05}
\rowcolors{3}{gray!10}{white}
\caption{
Negative prompt scenario on FFHQ-512 with prompt \texttt{A photo of a cat}. Results are reported for Gaussian deblurring, motion deblurring, and $\times8$ super-resolution with $\sigma_\vy=0.01$. \textbf{Bold}: best; \underline{underline}: second best.
}
\label{tab:comparison_FFHQ_neg_prompt}
\resizebox{\textwidth}{!}{%
\begin{tabular}{l c ccc ccc ccc}
\toprule
& & \multicolumn{3}{c}{\textbf{Gaussian}} 
& \multicolumn{3}{c}{\textbf{Motion}} 
& \multicolumn{3}{c}{\textbf{SR $\times8$}} \\
\cmidrule(lr){3-5}\cmidrule(lr){6-8}\cmidrule(lr){9-11}
\textbf{Method} & \textbf{NFE}
& \textbf{FID} & \textbf{PSNR} & \textbf{LPIPS}
& \textbf{FID} & \textbf{PSNR} & \textbf{LPIPS}
& \textbf{FID} & \textbf{PSNR} & \textbf{LPIPS} \\
\midrule
CWGF (no prompt opt.)
& \textbf{16} 
& \underline{32.91} & \underline{28.64} & 0.354
& \underline{32.20} & \underline{27.86} & \underline{0.357}
& 92.43 & \underline{26.24} & 0.456 \\
LATINO-PRO 
& \underline{65} 
& 44.67 & 28.63 & \underline{0.331} 
& 44.40 & 25.31 & 0.422 
& 85.58 & 26.17 & \underline{0.452} \\
P2L 
& 400 
& 46.40 & 26.06 & 0.476 
& 65.22 & 23.79 & 0.542 
& \underline{80.68} & 23.43 & 0.540 \\
TReg 
& 200 
& 95.96 & 22.89 & 0.463 
& 164.4 & 20.87 & 0.518 
& 137.0 & 21.29 & 0.528 \\
\midrule
\rowcolor{orange!20}
CWGF
& \textbf{16} 
& \textbf{27.58} & \textbf{29.59} & \textbf{0.322} 
& \textbf{25.99} & \textbf{28.16} & \textbf{0.339} 
& \textbf{50.07} & \textbf{27.39} & \textbf{0.415} \\
\bottomrule
\end{tabular}
}
\end{table}
We also compare our method with LATINO, LATINO-PRO, P2L, and TReg, showing improved semantic alignment. Indeed, while other methods fail to yield correct high-resolution features, CWGF improves prompt retrieval, which is necessary to correctly leverage the prior's capabilities to generate such details.

\begin{figure*}[htbp]
\centering
\begin{minipage}{0.03\textwidth}
  \centering
  \begin{tabular}{c}
    \\[-8mm]
    \rotatebox{90}{\scalebox{.6}{\textbf{Gaussian deblur}}} \\[5mm]
    \rotatebox{90}{\scalebox{.6}{\textbf{Motion deblur}}}  \\[10mm]
    \rotatebox{90}{\scalebox{.6}{\textbf{SRx8}}}
  \end{tabular}
\end{minipage}
\hfill
\begin{minipage}{0.15\textwidth}
    \centering \tiny\textbf{GT} \\[-1mm]
    \zoomedImage[width=\linewidth]{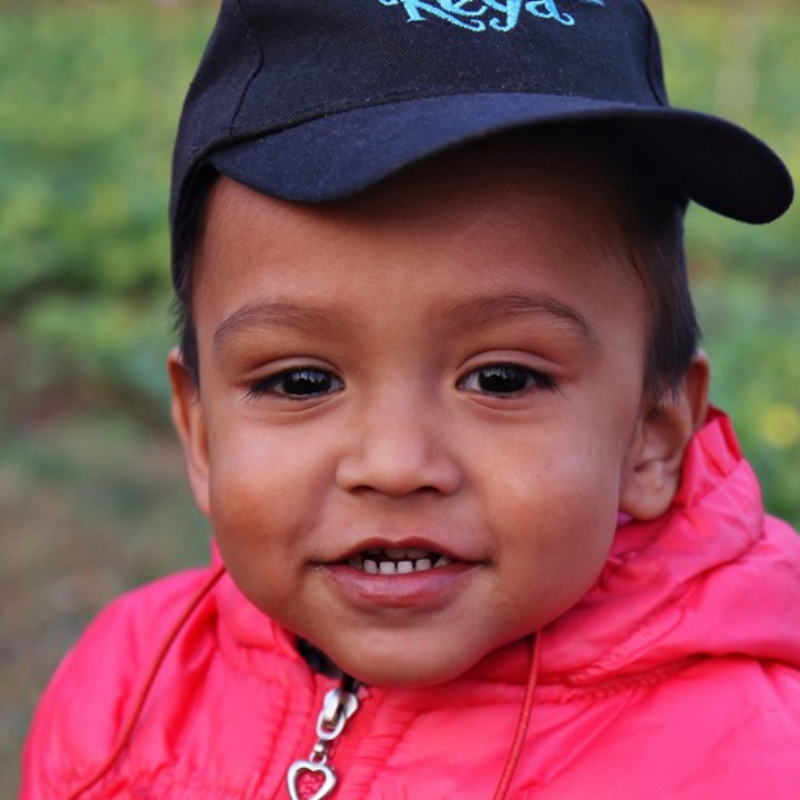}{-0.1,-0.4}{1.0,0.6} \\[-1mm]
    \zoomedImage[width=\linewidth]{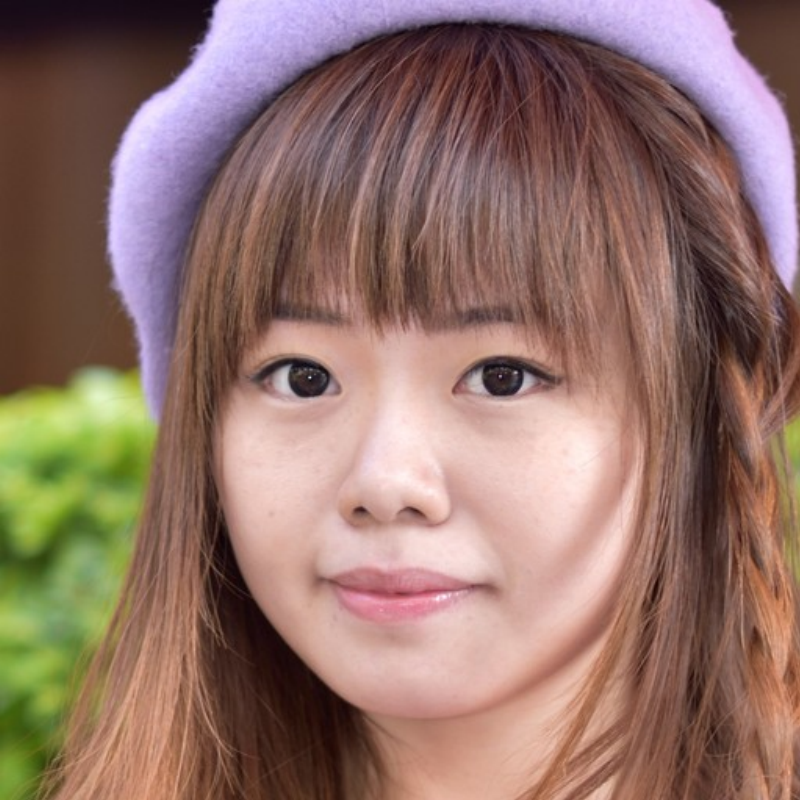}{-0.2,0.1}{1.0,0.5} \\[-1mm]
    \zoomedImage[width=\linewidth]{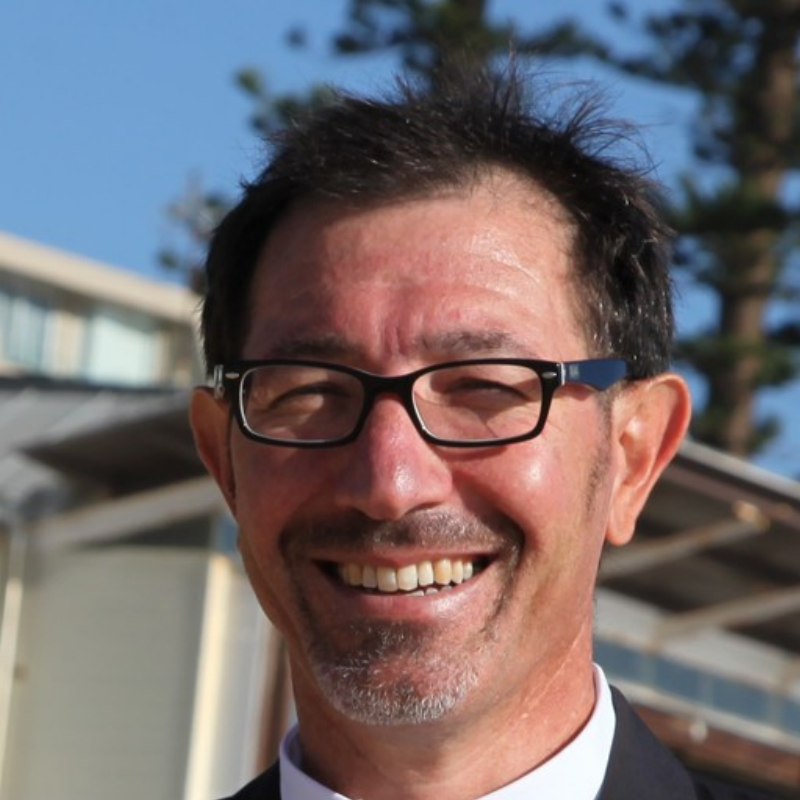}{-0.2,0.1}{1.0,-0.5} \\[-1mm]
\end{minipage}
\hfill
\begin{minipage}{0.15\textwidth}
  \centering \tiny\textbf{Measurement} \\[-1mm]
    \zoomedImage[width=\linewidth]{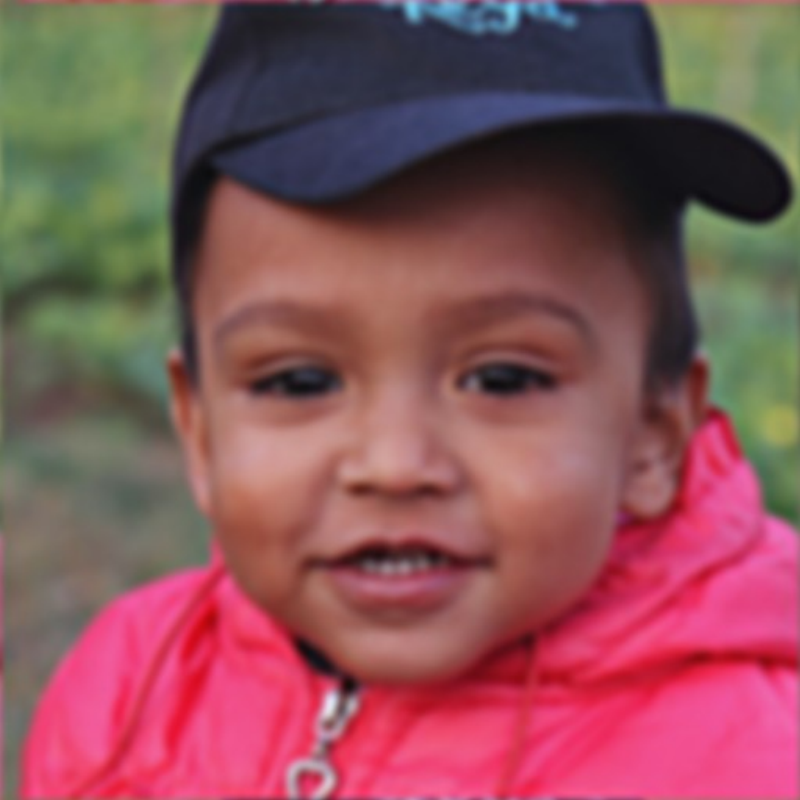}{-0.1,-0.4}{1.0,0.6}\\[-1mm]
    \zoomedImage[width=\linewidth]{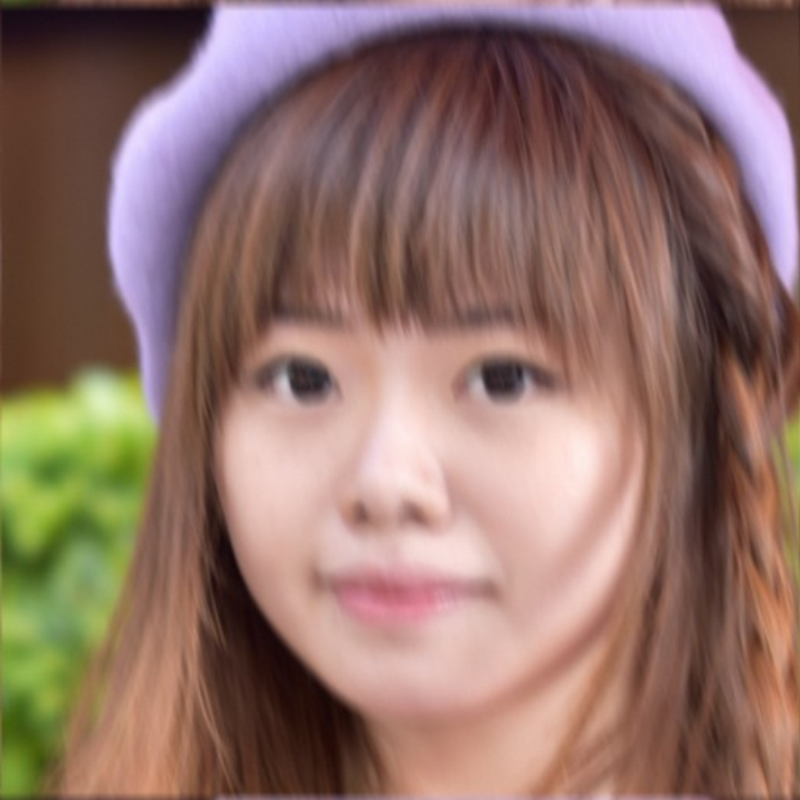}{-0.2,0.1}{1.0,0.5}\\[-1mm]
    \zoomedImage[width=\linewidth]{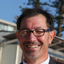}{-0.2,0.1}{1.0,-0.5} \\[-1mm]
\end{minipage}
\hfill
\begin{minipage}{0.15\textwidth}
    \centering \tiny\textbf{CWGF} \\[-1mm]
    \zoomedImage[width=\linewidth]{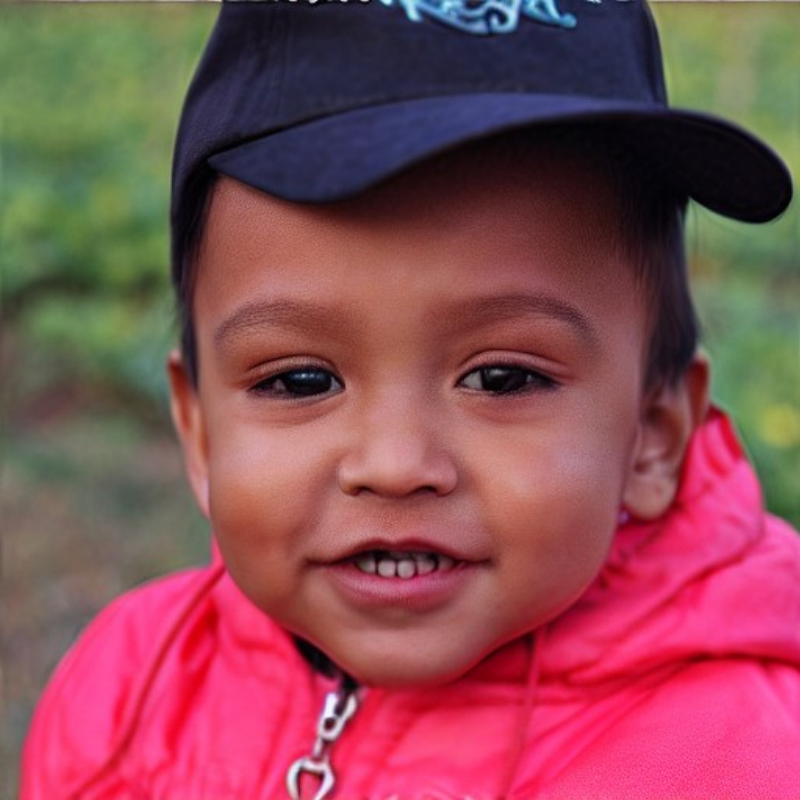}{-0.1,-0.4}{1.0,0.6}\\[-1mm]
    \zoomedImage[width=\linewidth]{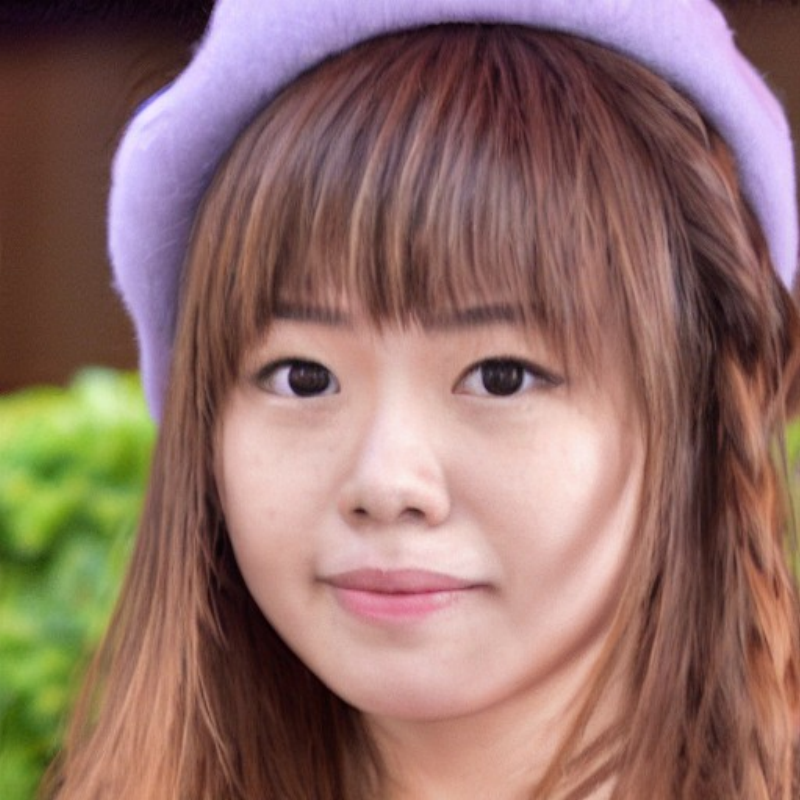}{-0.2,0.1}{1.0,0.5}\\[-1mm]
    \zoomedImage[width=\linewidth]{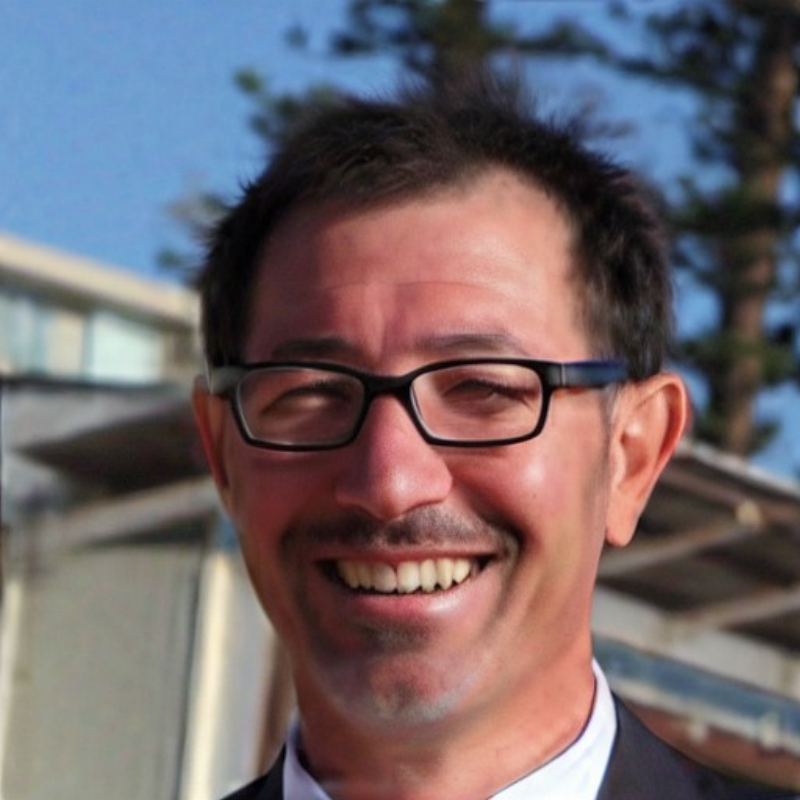}{-0.2,0.1}{1.0,-0.5} \\[-1mm]
\end{minipage}
\hfill
\begin{minipage}{0.15\textwidth}
    \centering \tiny\textbf{LATINO-PRO} \\[-1mm]
    \zoomedImage[width=\linewidth]{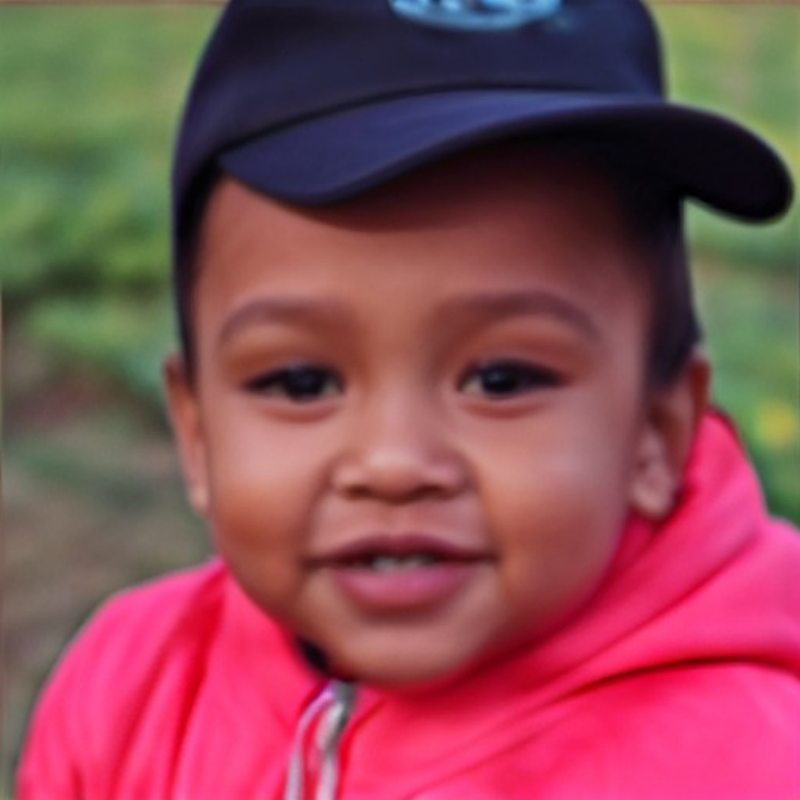}{-0.1,-0.4}{1.0,0.6} \\[-1mm]
    \zoomedImage[width=\linewidth]{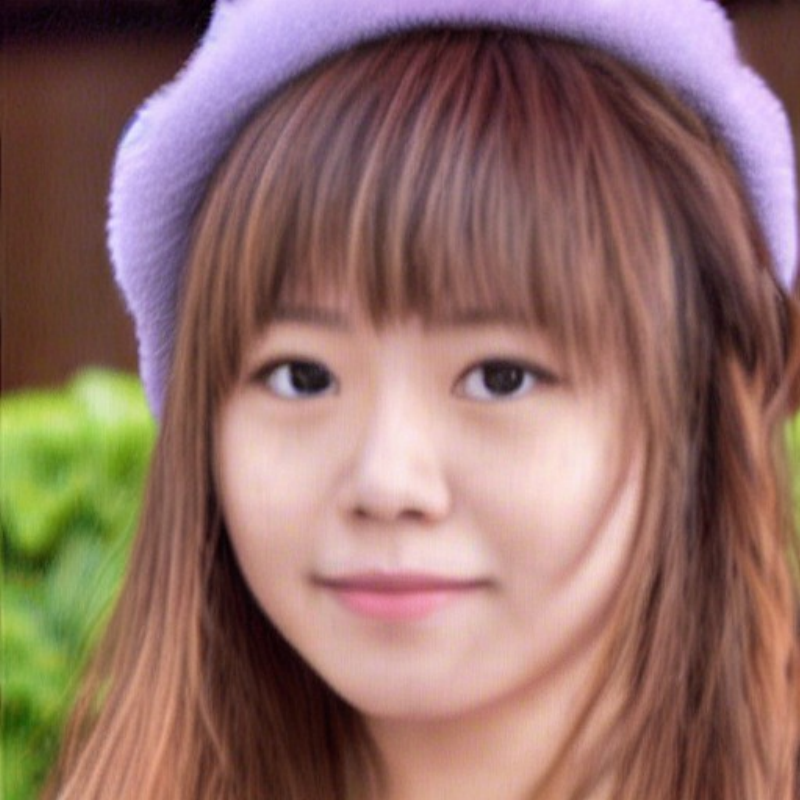}{-0.2,0.1}{1.0,0.5} \\[-1mm]
    \zoomedImage[width=\linewidth]{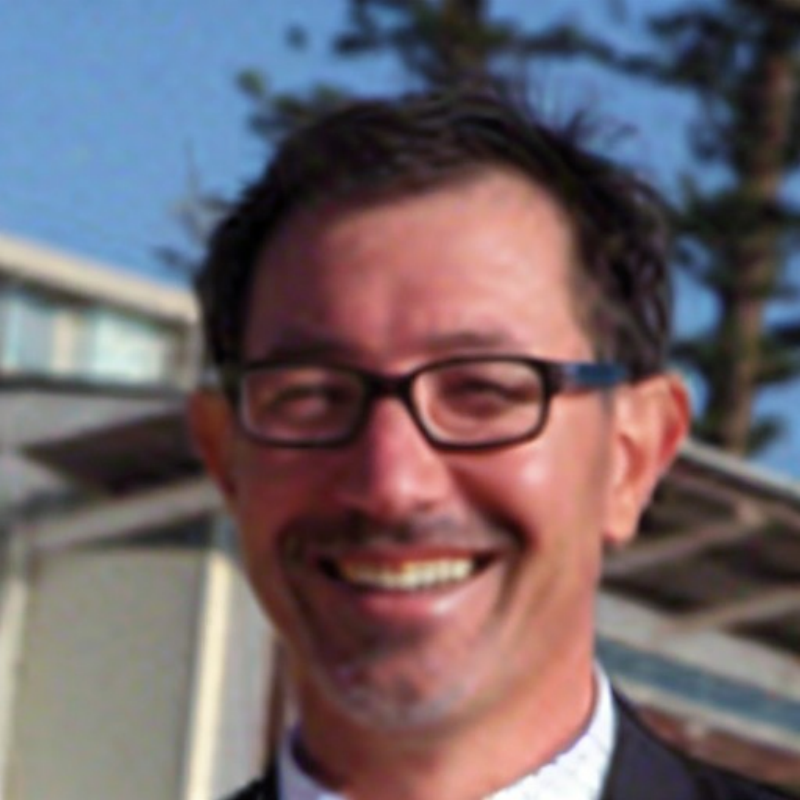}{-0.2,0.1}{1.0,-0.5} \\[-1mm]
\end{minipage}
\hfill
\begin{minipage}{0.15\textwidth}
    \centering \tiny\textbf{TREG} \\[-1mm]
    \zoomedImage[width=\linewidth]{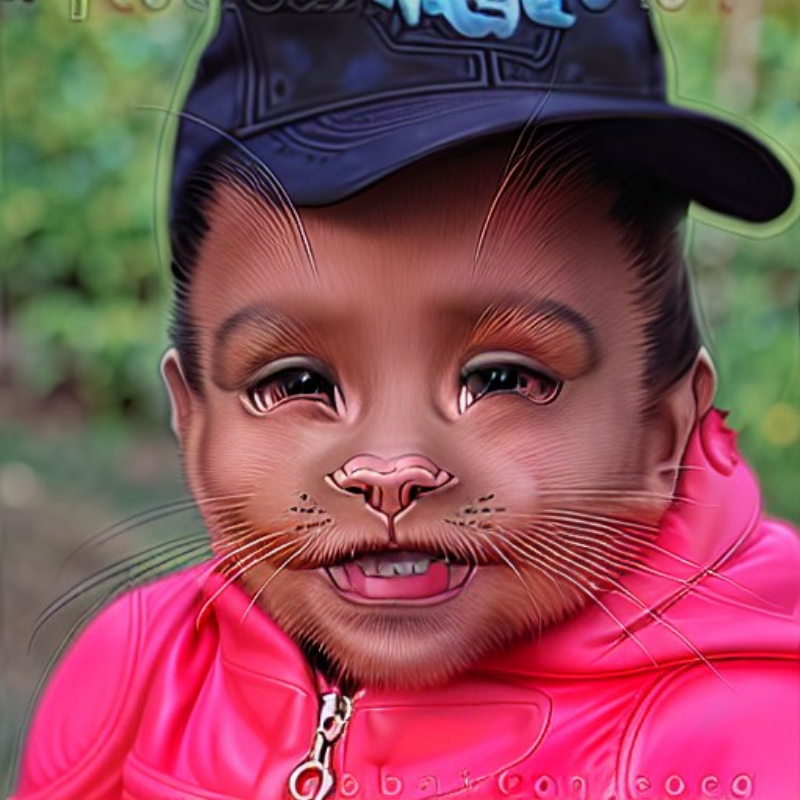}{-0.1,-0.4}{1.0,0.6} \\[-1mm]
    \zoomedImage[width=\linewidth]{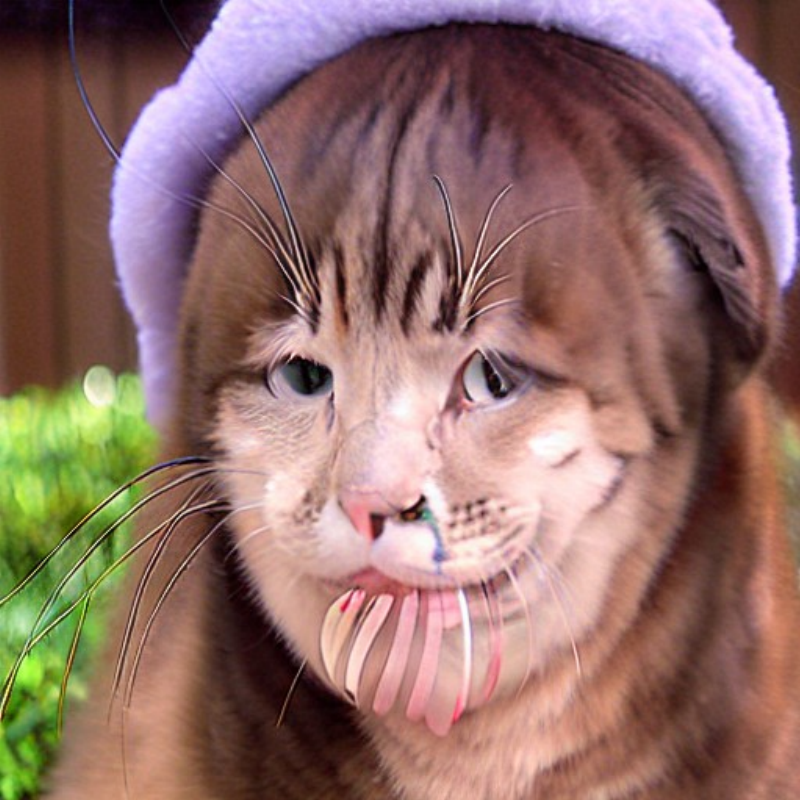}{-0.2,0.1}{1.0,0.5} \\[-1mm]
    \zoomedImage[width=\linewidth]{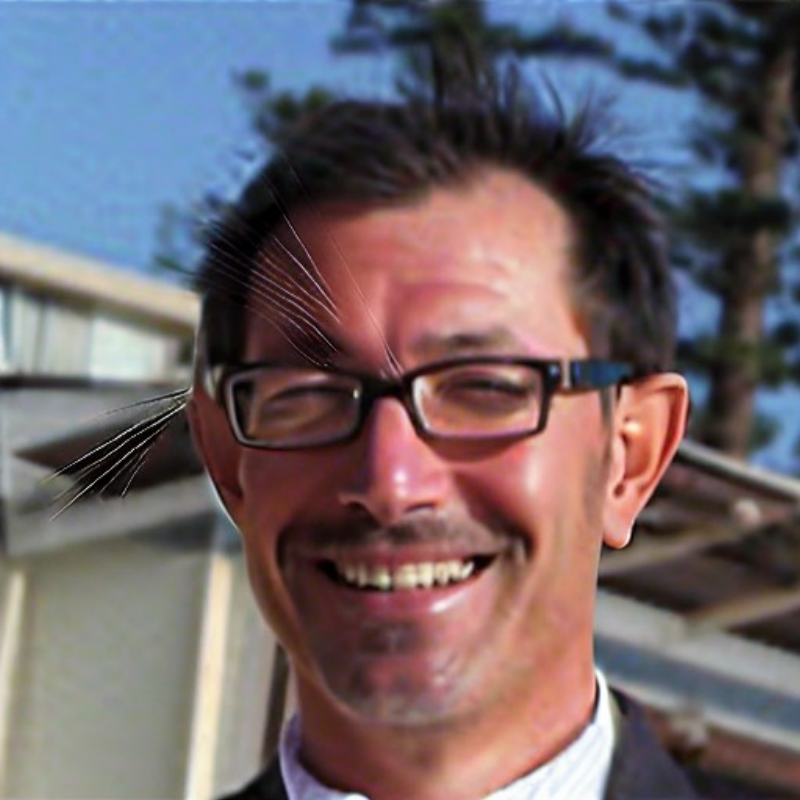}{-0.2,0.1}{1.0,-0.5} \\[-1mm]
\end{minipage}
\hfill
\begin{minipage}{0.15\textwidth}
    \centering \tiny\textbf{P2L} \\[-1mm]
    \zoomedImage[width=\linewidth]{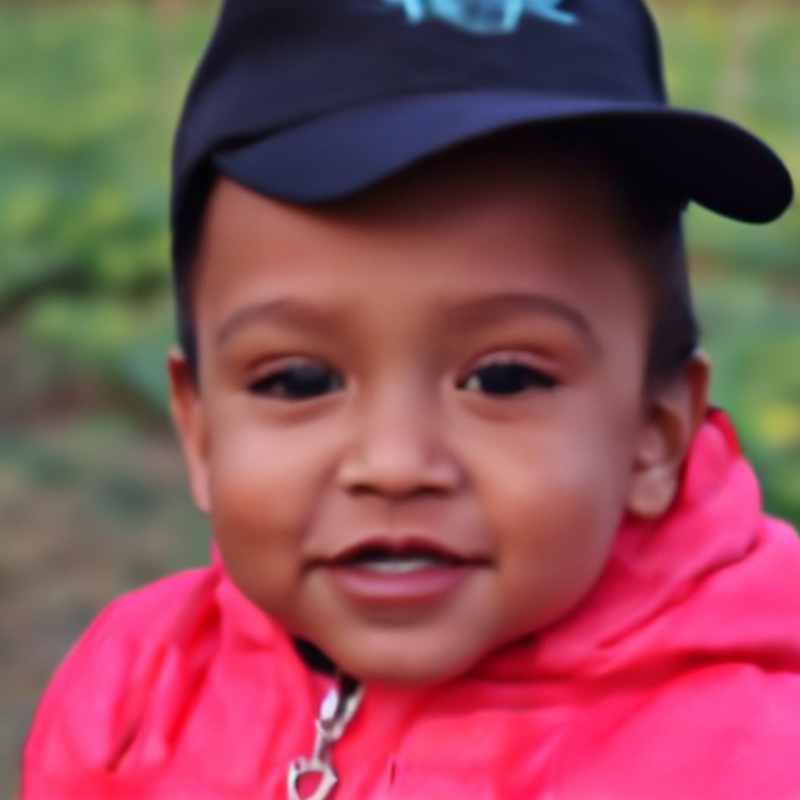}{-0.1,-0.4}{1.0,0.6} \\[-1mm]
    \zoomedImage[width=\linewidth]{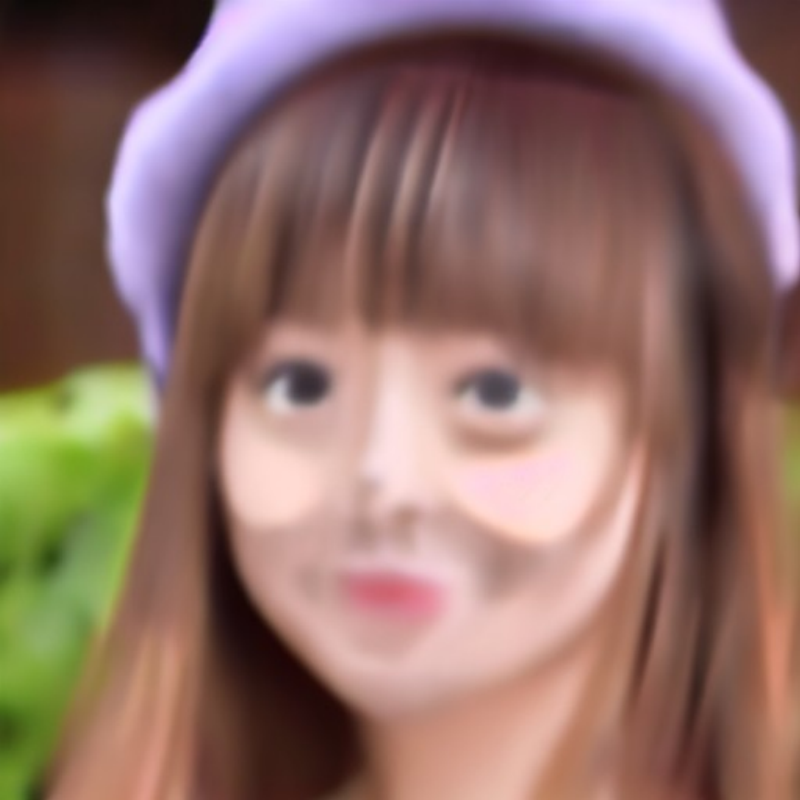}{-0.2,0.1}{1.0,0.5} \\[-1mm]
    \zoomedImage[width=\linewidth]{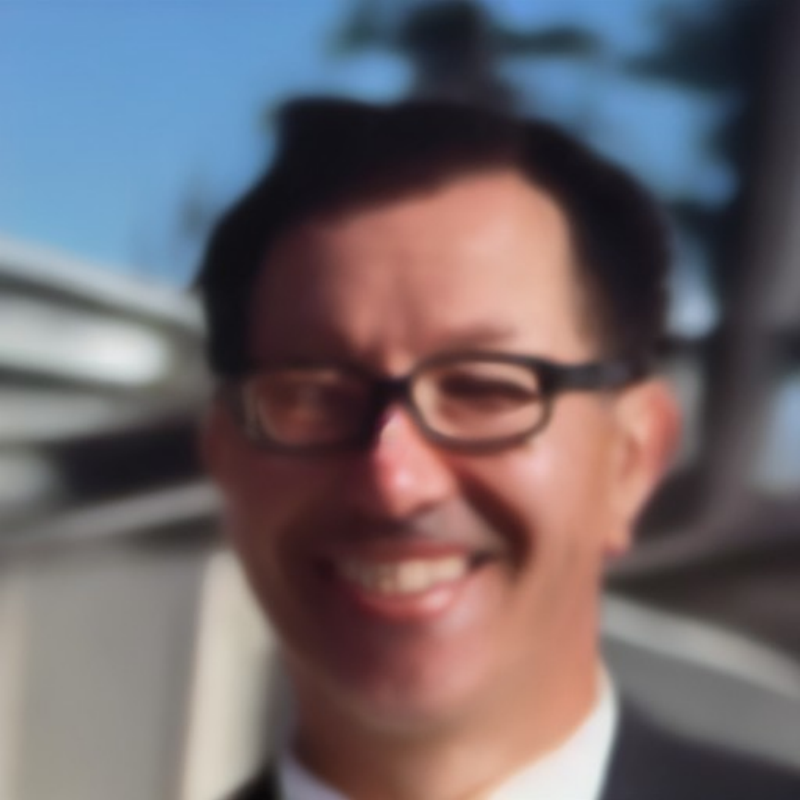}{-0.2,0.1}{1.0,-0.5} \\[-1mm]
\end{minipage}
\vspace{-2mm}
\caption{Qualitative comparison of image restoration results. Samples taken from FFHQ-512. Prompt: \texttt{A photo of a cat}.}
\label{fig:qualitative_FFHQ_cat}
\end{figure*}

%% file: sections/6-conclusion.tex
\section{Conclusion}\label{sec:conclusion}
We propose CWGF, a novel method for solving inverse problems with pretrained text-to-image LCM priors. CWGF leverages a joint Euclidean-Wasserstein gradient flow over prompt embeddings and latent distributions. Extensive experiments show that CWGF delivers state-of-the-art performance with a significantly reduced computational cost. However, there remain some limitations that future work should address. First, while we have assumed a Euclidean geometry for the prompt embeddings, their geometry is not fully understood and future work should explore non-Euclidean constructions~\citep{karrisWhichWayRole2025,leviDoubleEllipsoidGeometryCLIP2025}. Second, CWGF is based on a consistency model prior, leaving extensions to other few-step models as future work~\citep{boffi2025buildconsistencymodellearning,boffi2025flowmapmatchingstochastic,geng2025meanflowsonestepgenerative,dengGenerativeModelingDrifting2026}.

%% file: sections/7-acknow.tex
\section*{Acknowledgments and Disclosure of Funding}
MP acknowledges support by UKRI Engineering and Physical Sciences Research Council (EPSRC) (EP/Z534481/1). AS acknowledges support from the France 2030 research program on artificial intelligence via the PEPR PDE-AI grant (ANR-23-PEIA-0004). HPC resources provided by GENCI-IDRIS Jean-Zay (Grant 2024-AD011014557).
TW is supported by the Roth Scholarship from the Department of Mathematics, Imperial College London. The authors acknowledge computational resources and support provided by the Department of Mathematics and the Imperial College Research Computing Service, DOI: 10.14469/hpc/2232. The authors acknowledge the use of resources provided by the Isambard-AI National AI Research Resource (AIRR)~\citep{mcintoshsmith2024isambardaileadershipclasssupercomputer}. Isambard-AI is operated by the University of Bristol and is funded by the UK Government's Department for Science, Innovation and Technology (DSIT) via UK Research and Innovation; and the Science and Technology Facilities Council [ST/AIRR/I-A-I/1023].

%% file: appendix.tex
\section{Derivations and Proofs}
\label{app:proofs}
\subsection{Background on \texorpdfstring{$\mathsf{C}\times\mathcal{P}_2(\sZ)$}{CxP2(Z)}'s Geometry}\label{app-sec:background-w2}
\paragraph{Geometric Constructions.}~We endow the product space $\mathsf{C}\times\mathcal{P}_2(\sZ)$ with the geometry induced by the Wasserstein-2 metric on $\mathcal{P}_2(\sZ)$ and the Euclidean metric on $\mathsf{C}$ similarly to~\citet{kuntz23a}. The tangent space at a point $(\vc, \mu)$ is given by:
\begin{equation}
    T_{(\vc, \mu)}(\mathsf{C}\times\mathcal{P}_2(\sZ)) = T_{\vc}(\mathsf{C}) \times T_{\mu}(\mathcal{P}_2(\sZ)),
\end{equation}
where $T_{\vc}(\mathsf{C})$ is the usual Euclidean space and $T_{\mu}(\mathcal{P}_2(\sZ))$ is the tangent space at $\mu$ defined by the Wasserstein-2 geometry~\citep{otto2001,figalli2021invitation}:
\begin{equation}
    T_{\mu}(\mathcal{P}_2(\sZ)) = \left\{
                \partial_t \mu_t|_{t=0} \ \middle|\ \partial_t \mu_t + \nabla \cdot (\mu_t v_t) = 0, \mu_0 = \mu, \; v_t=\nabla\psi_t, \psi_t\in C_c^\infty(\sZ)
            \right\}.
\end{equation}
We define the Riemannian metric on the tangent space as a product metric:
\begin{equation}
    \langle (u_1, v_1), (u_2, v_2) \rangle_{(\vc, \mu)} = \langle u_1, u_2 \rangle + \langle v_1, v_2 \rangle_\mu,
\end{equation}
where $\langle u_1, u_2 \rangle$ is the standard Euclidean inner product and $\langle v_1, v_2 \rangle_\mu$ is the Wasserstein-2 inner product defined as:
\begin{equation}
    \langle v_1, v_2 \rangle_\mu = \int \nabla \psi_{v_1}(\vz) \cdot \nabla \psi_{v_2}(\vz) \mu(\vz) \dd \vz,
\end{equation}
where $\psi_{v_1}$ and $\psi_{v_2}$ are the potentials associated with the vector fields $v_1$ and $v_2$ satisfying $v_i = -\nabla \cdot (\mu \nabla \psi_{v_i})$ for $i=1,2$.
In what follows, we restrict ourselves to the subspace of distributions that have probability densities w.r.t. the Lebesgue measure and we thus use $\mathcal{P}_{2}(\sZ)$ interchangeably with $\mathcal{P}_{2,a.c.}(\sZ)$.

\paragraph{The Gradient.}~The Wasserstein-2 gradient at $\mu \in \mathcal{P}_2(\mathsf{Z})$ is the unique tangent vector in $ \mathrm{grad}_{W_2} \mathcal{L}[\mu] \in T_\mu \mathcal{P}_2(\mathsf{Z})$ such that for any tangent vector $v \in T_\mu \mathcal{P}_2(\mathsf{Z})$, we have:
\begin{equation}\label{eq:w2grad}
    \langle \mathrm{grad}_{W_2} \mathcal{L}[\mu], v \rangle_{\mu} = \lim_{\varepsilon \to 0} \frac{\mathcal{L}[\mu_\varepsilon] - \mathcal{L}[\mu]}{\varepsilon},
\end{equation}
for any smooth curve $\mu_\varepsilon: (-\varepsilon_0, \varepsilon_0) \to \mathcal{P}_2(\mathsf{Z})$ satisfying $\mu_0 = \mu$ and $\left.\partial_\varepsilon \mu_\varepsilon\right|_{\varepsilon=0} = v$.

For a regular functional $\mathcal{L}:\mathcal{P}_2(\sZ)\to \R$ satisfying the assumptions discussed in~\citet[10.4.1]{ambrosio2008gradient}, we define its first variation at $\mu$ as the unique function $\delta \mathcal{L}/\delta \mu:\sZ\to \R$ satisfying:
\begin{equation}
\lim_{\varepsilon \searrow 0} \frac{\mathcal{L}[\mu + \varepsilon \chi] - \mathcal{L}[\mu]}{\varepsilon} = \int \frac{\delta\mathcal{L}}{\delta \mu}(\vz) \, \chi(\dd \vz) ,
\end{equation}
for all signed measures $\chi$ such that $\mu+\varepsilon \chi\in\mathcal{P}_2(\sZ)$ for all sufficiently small $\varepsilon$. The Wasserstein-2 gradient of $\mathcal{L}$ at $\mu$ is given by~\citep[Chapter 4]{figalli2021invitation}:
\begin{equation}
    \mathrm{grad}_{W_2} \mathcal{L}[\mu] = -\nabla \cdot \left(\mu \nabla \frac{\delta \mathcal{L}}{\delta \mu}\right).
\end{equation}
Therefore, for a functional $\mathcal{F}: \mathsf{C}\times \mathcal{P}_2(\sZ)\to \R$ defined on the product manifold that is differentiable in $c$ and regular in $\mu$, the gradient at $(c,\mu)$ is given by:
\begin{equation}
    \nabla \mathcal{F}[c,\mu] := \left(\nabla_c \mathcal{F}[c,\mu], \mathrm{grad}_{W_2} \mathcal{F}[c,\mu]\right).
\end{equation}

We first state a standard result on the first variation of the KL divergence. (cf.~\citet{villani2003,santambrogio2016euclideanmetricwasserstein,figalli2021invitation})
\begin{lemma}[First Variation of KL Divergence]\label{lemma:first_var_kl}
    For $p, q \in \mathcal{P}_2(\sZ)$, the first variation of the KL divergence is given by:
    \begin{equation}
        \frac{\delta \KL(q\|p)}{\delta q} (z) = \log q(z) - \log p(z) +1, \quad \forall z\in \sZ
    \end{equation}
\end{lemma}
\begin{proof}
    Using the equivalent definition of first variation, we can write for $m \in T\mathcal{P}_2(\sZ)$ and $t>0$:
    \begin{equation}
        \KL(q+tm\|p) = \KL(q\|p) + t\left\langle m, \frac{\delta \KL(q\|p)}{\delta q} \right\rangle  + o(t),
    \end{equation}
    where the inner product is defined $\langle m,f\rangle:=\int_{\sZ} f(z)m(z) \mathrm{d} z$ for all $f,m\in T\mathcal{P}_2(\sZ)$.
    Using Taylor expansion of $(z+t)\log(z+t)=z\log z+t(\log z+1)+o(t)$, we can write $\KL(q+tm\|p)$ as:
    \begin{align}
        \KL(q+tm\|p) &= \int (q(z)+tm(z)) [\log (q(z)+tm(z))-\log p(z)] \dd z \\
        &=\int q(z)\log q(z) \dd z - \int q(z) \log p(z) \dd z \\
        &+ t\int (\log q(z) - \log p(z) +1) m(z)\dd z   + o(t),
    \end{align}
    which shows the desired result by matching the terms.
\end{proof}
As a consequence, the Wasserstein-2 gradient of the KL divergence is given by:
\begin{equation}
    \mathrm{grad}_{W_2} \KL(q\|p) = -\nabla \cdot \left(q \nabla \log \frac{q}{p}\right).
\end{equation}

\begin{lemma}\label{lemma:semi-implicit}
    For $\mathcal{L}: \mathcal{P}_2(\sZ)\to \R$ with $\mathcal{L}[\mu] :=\int w(z)\mu(z) \dd z$ for any fixed $w(z):\R^d \to \R$, we have:
    \begin{equation}
        \frac{\delta\mathcal{L}[\mu]}{\delta \mu}(z) = w(z)
    \end{equation}
\end{lemma}
\begin{proof}
    Similar to the proof above, for any $m \in T\mathcal{P}_2(\sZ)$ and $t\in \R$, we have:
    \begin{align}
        \mathcal{L}[\mu+tm] &= \int w(z) (\mu(z) + tm(z)) \dd z \\
        &= \int w(z) \mu(z) \dd z + t \int w(z) m(z) \dd z.
    \end{align}
    From the definition of the functional derivative, this implies the desired result.
\end{proof}

\subsection{Proof of \texorpdfstring{Theorem~\ref{thm:1st-opt-cond}}{first order optimality conditions}}\label{app:proof-prop-opt-cond}
\begin{proof}
    By the definition of $\mathrm{grad}_{W_2}$, $\mathrm{grad}_{W_2}\mathcal{F}[c,\mu]=0$ if and only if $\mu=p_c(\vz|\vy)$, which is the optimal solution to $\argmin \KL(\mu\|p_c) - \mathbb{E}_\mu[\log p(\vy|\vz)]$. Moreover, we have from $p_c(\vy, \vz)=p(\vy|\vz)p_c(\vz)$ that
    \begin{align*}
        \nabla_c p_c(\vy) &=  \int \nabla_c p_c(\vy,\vz)\dd \vz = \int \nabla_c \log p_c(\vz) p_c(\vz|\vy) p_c(\vy) \dd \vz \\
        &= p_c(\vy)\nabla_c \KL(\mu\|p_c) 
    \end{align*}
    Thus, $\nabla_c p_c(\vy)=0$ if and only if $\nabla_c \mathcal{R}[c,\mu]=\nabla_c \KL(\mu\|p_c)=0$. The result follows by combining the two conditions.
\end{proof}

\subsection{Exponential Convergence to the Minimisers}\label{app:proof-exp-conv}
To prove the exponential convergence, we first need two standard assumptions on the regularity of the model and the strong log-concavity of the likelihood and the prior.
\begin{assumption}[Model Regularity]\label{assump:app-model-regularity}
i) For all $\vz\in \sZ$, the maps $c\mapsto \log p_c(\vz|\vy)$ and $c\mapsto \log p_c(\vy)$ are differentiable; ii) for all $c\in \sC$, $c\mapsto p_c(\vz|\vy)$ is twice continuously differentiable; iii) for all $c\in \sC$ and $\vz\in \sZ$, the joint density $p_c(\vy,\vz)$ is positive everywhere.
\end{assumption}
\begin{assumption}[Strong Log-Concavity]\label{assump:app-log-concavity}
    i) The map $(c,\vz)\mapsto\log p_c(\vz)$ is jointly $\lambda$-strongly concave; ii) $p(\vy|\vx)p_{\phi^-}(\vz|\vy)$ is $\lambda'$-strongly log-concave.
\end{assumption}
\begin{remark}
    We note that in the case of a linear inverse problem, $p(\vy|\vx)$ is Gaussian and thus log-concave. Moreover, if the decoder is linear Gaussian, then $p_{\phi^-}(\vz|\vy)$ is also strongly log-concave.
\end{remark}
\begin{theorem}[Exponential convergence]
    Under Assumptions~\ref{assump:app-model-regularity} and~\ref{assump:app-log-concavity}, the functional $\mathcal{F}[c, \mu]$ has a unique minimiser $(c_\star, \mu_\star)$ and for $\tilde{\lambda}:=\min(\lambda, \lambda')$ the flow in~\eqref{eq:prompt-cont-flow}-\eqref{eq:distr-cont-flow} converges exponentially fast to $(c_\star, \mu_\star)$
    \begin{equation}
        \mathrm{d}_{W_2}(\mu_\tau, \mu_\star)^2 + \|c_\tau - c_\star\|^2_2 \leq 2\tilde{\lambda}^{-1} e^{-2\tilde{\lambda} \tau} (\mathcal{F}[c_0, \mu_0] - \mathcal{F}[c_\star, \mu_\star]).
    \end{equation}
\end{theorem}
\begin{proof}
    The proof follows the same strategy as in~\citet{caprioerrorboundsparticle2024}.
    Define $\ell(c,\vz):=\log p_c(\vy,\vz)=\log \int p(\vy|\vx) p(\vx|\vz) p_c(\vz) d\vx$. We first show that $\ell(c,\vz)$ is strongly concave in $(c,\vz)$. Indeed, we have $\ell(c,\vz)=\log p_c(\vz) + \log \int p(\vy|\vx)p(\vx|\vz) d\vx$, where the first term is $\lambda$-strongly concave by Assumption~\ref{assump:app-log-concavity}, and the second term is $\lambda'$-strongly concave in $\vz$ by the strong log-concavity of $p(\vy|\vx)p(\vx|\vz)$ and the Pr\'ekopa-Leindler inequality~\citep[Theorem 7.1]{gardner2002brunn}. It follows that for any $t\in[0,1]$ and $(c,\vz), (c', \vz')\in \sC\times \sZ$,
    \begin{align*}
        &\ell((1-t)c + tc', (1-t)\vz + t\vz') \geq (1-t)\ell(c,\vz) + t\ell(c', \vz') \\
        &+ \frac{\lambda}{2}t(1-t)(\|c-c'\|^2_2 + \|\vz-\vz'\|^2_2) + \frac{\lambda'}{2}t(1-t)\|\vz-\vz'\|^2_2 \\
        &\geq (1-t)\ell(c,\vz) + t\ell(c', \vz')
        + \frac{\min(\lambda, \lambda')}{2}t(1-t)(\|c-c'\|^2_2 + \|\vz-\vz'\|^2_2).
    \end{align*}
    Thus, under the log-concavity of $\ell(c,\vz)$ and Assumption~\ref{assump:app-model-regularity}, we can apply a similar argument to~\citet[Theorem 4, Appendix B.3]{kuntz23a} to obtain the existence and uniqueness of the minimiser $(c_\star, \mu_\star)$ of $\log p_c(\vy)$. The result follows directly from~\citet[Corollary 5]{caprioerrorboundsparticle2024} given the log-concavity of $\ell(c,\vz)$ and the regularity of the model.
\end{proof}

\subsection{Proof of Proposition~\ref{prop:exact-flow-map-identity}}
\label{app:proof-prop-exact-flow-map}

\begin{proof}
Fix $r\in\mathcal P_2(\sZ)$ and let $r_s=Q_s\ast r$. Along the probability-flow characteristic associated with $r_s$, we have
\[
\frac{d\vz_s^r}{ds}
=
-\frac{\beta_s}{2}\vz_s^r
-\frac{\beta_s}{2}\nabla\log r_s(\vz_s^r),
\qquad s\in[0,t],
\]
with terminal condition $\vz_t^r=\vz_t$ and initial point $\vz_0^r=G_t^r(\vz_t)$. Using $\frac{d}{ds}\alpha_s^{-1/2}=\frac{\beta_s}{2}\alpha_s^{-1/2}$, we obtain
\begin{align*}
\frac{d}{ds}\left(\frac{\vz_s^r}{\sqrt{\alpha_s}}\right)
&=
\frac{1}{\sqrt{\alpha_s}}\frac{d\vz_s^r}{ds}
+
\frac{\beta_s}{2\sqrt{\alpha_s}}\vz_s^r \\
&=
-\frac{\beta_s}{2\sqrt{\alpha_s}}\nabla\log r_s(\vz_s^r).
\end{align*}
Integrating from $0$ to $t$ and using $\alpha_0=1$ gives
\[
\frac{\vz_t}{\sqrt{\alpha_t}}-G_t^r(\vz_t)
=
-\int_0^t
\frac{\beta_s}{2\sqrt{\alpha_s}}
\nabla\log r_s(\vz_s^r)\,ds.
\]
Equivalently,
\[
G_t^r(\vz_t)
=
\frac{\vz_t}{\sqrt{\alpha_t}}
+
\int_0^t
\frac{\beta_s}{2\sqrt{\alpha_s}}
\nabla\log r_s(\vz_s^r)\,ds.
\]
Multiplying by $\sqrt{\alpha_t}$, subtracting $\vz_t$, and dividing by $\sigma_t^2=1-\alpha_t$ yields
\[
\frac{\sqrt{\alpha_t}G_t^r(\vz_t)-\vz_t}{\sigma_t^2}
=
\frac{\sqrt{\alpha_t}}{\sigma_t^2}
\int_0^t
\frac{\beta_s}{2\sqrt{\alpha_s}}
\nabla\log r_s(\vz_s^r)\,ds.
\]
Multiplying both sides by $1+\sqrt{\alpha_t}$ gives
\[
(1+\sqrt{\alpha_t})
\frac{\sqrt{\alpha_t}G_t^r(\vz_t)-\vz_t}{\sigma_t^2}
=
\int_0^t
\gamma_t(s)\nabla\log r_s(\vz_s^r)\,ds,
\]
where
\[
\gamma_t(s)
:=
\frac{(1+\sqrt{\alpha_t})\sqrt{\alpha_t}}{\sigma_t^2}
\frac{\beta_s}{2\sqrt{\alpha_s}}=
\frac{\sqrt{\alpha_t}}{1-\sqrt{\alpha_t}}
\frac{\beta_s}{2\sqrt{\alpha_s}}
\]
It remains to check that $\gamma_t$ is a probability density on $[0,t]$. Since $\beta_s\geq0$ and $\alpha_s>0$, we have $\gamma_t(s)\geq0$. Moreover,
\[
\int_0^t \frac{\beta_s}{2\sqrt{\alpha_s}}\,ds
=
\frac{1}{\sqrt{\alpha_t}}-1,
\]
and therefore, using $\sigma_t^2=1-\alpha_t$,
\[
\int_0^t\gamma_t(s)\,ds
=
\frac{(1+\sqrt{\alpha_t})\sqrt{\alpha_t}}{1-\alpha_t}
\left(
\frac{1}{\sqrt{\alpha_t}}-1
\right)
=
\frac{(1+\sqrt{\alpha_t})(1-\sqrt{\alpha_t})}{1-\alpha_t}
=
1.
\]
Thus the right-hand side is a convex average of positive-time scores along the probability-flow characteristic. Under continuity of $s\mapsto\nabla\log r_s(\vz_s^r)$ at $s=0$, this average converges to $\nabla\log r(\vz_0^r)$ as $t\to0$, with the quantitative Gaussian behaviour discussed in Appendix~\ref{app:gaussian-cm-score-error}.
\end{proof}

\subsection{Proof of Proposition~\texorpdfstring{\ref{prop:approx-delta-L}}{L}}\label{app:proof-prop-approx-delta-L}
\begin{proof}
First note that for a fixed decoder \(p_\phi(\vx \mid \vz)\), the optimal encoder distribution associated with the VAE objective is given by
\[
q_\phi^\star(\vz \mid \vx)
\propto
p_0(\vz)\exp\!\left(\frac{1}{\lambda}\log p_\phi(\vx \mid \vz)\right).
\]
Taking the logarithm and differentiating with respect to \(\vz\), we obtain
\[
\nabla_\vz \log q_\phi^\star(\vz \mid \vx)
=
\nabla_\vz \log p_0(\vz)
+
\frac{1}{\lambda}\nabla_\vz \log p_\phi(\vx \mid \vz),
\]
and therefore $\nabla_\vz \log p_\phi(\vx \mid \vz)
=
\lambda \nabla_\vz \log q_\phi^\star(\vz \mid \vx)
-
\lambda \nabla_\vz \log p_0(\vz)$. By Fisher's identity,
\[
g(\vz;\vy)
=
\E_{p_c(\vx \mid \vy,\vz)}\!\left[\nabla_\vz \log p_\phi(\vx \mid \vz)\right].
\]
Substituting the previous identity into this expression gives
\[
g(\vz;\vy)
=
\lambda \E_{p_c(\vx \mid \vy,\vz)}\!\left[\nabla_\vz \log q_\phi^\star(\vz \mid \vx)\right]
-
\lambda \nabla_\vz \log p_0(\vz).
\]
On the other hand, by definition,
\[
\hat g(\vz;\vy)
:=
\lambda \E_{p_c(\vx \mid \vy,\vz)}\!\left[\nabla_\vz \log q_\phi(\vz \mid \vx)\right]
-
\lambda \nabla_\vz \log p_0(\vz).
\]
Hence,
\begin{align*}
g(\vz;\vy)-\hat g(\vz;\vy)
&=
\lambda \E_{p_c(\vx \mid \vy,\vz)}\!\left[
\nabla_\vz \log q_\phi^\star(\vz \mid \vx)
-
\nabla_\vz \log q_\phi(\vz \mid \vx)
\right].
\end{align*}
Taking Euclidean norms and using Jensen's inequality, we obtain
\begin{align*}
\|g(\vz;\vy)-\hat g(\vz;\vy)\|_2^2
&\le
\lambda^2
\E_{p_c(\vx \mid \vy,\vz)}\!\left[
\|\nabla_\vz \log q_\phi^\star(\vz \mid \vx)
-
\nabla_\vz \log q_\phi(\vz \mid \vx)
\right\|^2_2].
\end{align*}
\end{proof}

\subsection{Proof of the construction of the Gaussian Posterior}\label{app:proof-prop-posterior-gauss}
\begin{proposition}\label{prop:posterior_gauss}
    For observation model $p(\vy|\vx) = \mathcal{N}(\vy; A\vx, \sigma_\vy^2 I)$ and Gaussian decoder $p_{\phi^-}(\vx|\vz) = \mathcal{N}(\vx; \mathcal{D}_{\phi^-}(\vz), \sigma_{dec}^2 I)$, we have $p_c(\vx|\vz,\vy) = \mathcal{N}(\vx; m(\vz), \Sigma)$, where $m(\vz)$ and $\Sigma$ are given by:
    \begin{align}\label{eq:mean}
        m(\vz) = \Sigma \left(\sigma_{dec}^{-2}\mathcal{D}_{\phi^-}(\vz) + \sigma_\vy^{-2}A^\top \vy\right), \qquad \Sigma = \left({\sigma_{dec}^{-2}}I + {\sigma_\vy^{-2}} A^\top A
        \right)^{-1}.
    \end{align}
\end{proposition}
\begin{proof}
    This is the standard conjugate Gaussian update. By Bayes' rule,
    \begin{equation*}
        p_c(\vx|\vz,\vy) \propto p(\vy|\vx)p_{\phi^-}(\vx|\vz),
    \end{equation*}
    and therefore
    \begin{equation*}
        p_c(\vx|\vz,\vy)
        \propto
        \exp\left(
            -\frac{1}{2\sigma_\vy^2}\|\vy-A\vx\|_2^2
            -\frac{1}{2\sigma_{dec}^2}\|\vx-\mathcal{D}_{\phi^-}(\vz)\|_2^2
        \right).
    \end{equation*}
    Expanding the quadratic terms in $\vx$ and collecting only the terms that depend on $\vx$, we get
    \begin{align*}
        \log p_c(\vx|\vz,\vy)
        &=
        -\frac{1}{2}\vx^\top
        \left(
            \sigma_{dec}^{-2}I+\sigma_\vy^{-2}A^\top A
        \right)\vx \\
        &\quad +
        \vx^\top
        \left(
            \sigma_{dec}^{-2}\mathcal{D}_{\phi^-}(\vz)
            +
            \sigma_\vy^{-2}A^\top \vy
        \right)
        + \mathrm{const},
    \end{align*}
    where $\mathrm{const}$ does not depend on $\vx$. Define
    \begin{equation*}
        \Sigma^{-1}:=\sigma_{dec}^{-2}I+\sigma_\vy^{-2}A^\top A,
        \qquad
        m(\vz):=\Sigma\left(
            \sigma_{dec}^{-2}\mathcal{D}_{\phi^-}(\vz)
            +
            \sigma_\vy^{-2}A^\top \vy
        \right).
    \end{equation*}
    Since $\sigma_{dec}>0$, the matrix $\Sigma^{-1}$ is positive definite and hence invertible. Completing the square gives:
    \begin{equation*}
        \log p_c(\vx|\vz,\vy)
        =
        -\frac{1}{2}
        (\vx-m(\vz))^\top \Sigma^{-1}(\vx-m(\vz))
        + \mathrm{const}.
    \end{equation*}
    Hence $p_c(\vx|\vz,\vy)=\mathcal{N}(\vx; m(\vz), \Sigma)$, where $\Sigma$ and $m(\vz)$ are defined as above.
\end{proof}

\section{Gaussian error analysis for the CM-induced prior score}
\label{app:gaussian-cm-score-error}

We provide a closed-form analysis of the CM-induced score surrogate in the Gaussian case. This setting is useful because the clean score, the positive-time score, and the exact probability-flow map are all available explicitly. It also reflects the regime targeted by latent diffusion models, whose VAE latent spaces are regularised toward a standard Gaussian prior.

Let $\vz_0\sim p_0=\mathcal N(0,\Sigma)$ with $\Sigma\succ 0$, and consider the VP forward process $\vz_t=\sqrt{\alpha_t}\vz_0+\sigma_t\varepsilon$, where $\varepsilon\sim\mathcal N(0,I)$ and $\sigma_t^2=1-\alpha_t$. Then $\vz_t\sim p_t=\mathcal N(0,M_t)$, with $M_t:=\alpha_t\Sigma+\sigma_t^2 I=I+\alpha_t(\Sigma-I)$, and the positive-time score is $s_t^\star(\vz_t)=\nabla_{\vz_t}\log p_t(\vz_t)=-M_t^{-1}\vz_t$. The clean score is $s_0^\star(\vz_0)=-\Sigma^{-1}\vz_0$.

Let $g_t^\star$ denote the exact probability-flow map from time $t$ to time $0$. Motivated by~\Cref{prop:exact-flow-map-identity}, we study the normalised CM-induced score
$$
\bar s_t(\vz_t)
:=
(1+\sqrt{\alpha_t})
\frac{\sqrt{\alpha_t}g_t^\star(\vz_t)-\vz_t}{\sigma_t^2}
=
\frac{\sqrt{\alpha_t}g_t^\star(\vz_t)-\vz_t}{1-\sqrt{\alpha_t}}.
$$
When $g_\theta=g_t^\star$, this is the exact normalised counterpart of the surrogate induced by the CM score formula. In practice, $g_t^\star$ is replaced by the consistency model $g_\theta$.

\begin{lemma}[Gaussian probability-flow map]
\label{lem:gaussian-pf-map}
For $p_0=\mathcal N(0,\Sigma)$ and $p_t=\mathcal N(0,M_t)$, the exact probability-flow map from time $t$ to time $0$ is
$$
g_t^\star(\vz_t)=\Sigma^{1/2}M_t^{-1/2}\vz_t.
$$
\end{lemma}

\begin{proof}
Since $M_t=\alpha_t\Sigma+\sigma_t^2 I$, the matrices $M_t$ and $\Sigma$ commute and are diagonalisable in the same orthonormal basis. The linear map $\Sigma^{1/2}M_t^{-1/2}$ transports $\mathcal N(0,M_t)$ to $\mathcal N(0,\Sigma)$ because
$$
\Sigma^{1/2}M_t^{-1/2}M_tM_t^{-1/2}\Sigma^{1/2}=\Sigma.
$$
For Gaussian marginals, the probability-flow characteristic is exactly this deterministic Gaussian transport map.
\end{proof}

We compare $\bar s_t(\vz_t)$ with the clean score evaluated at the endpoint $g_t^\star(\vz_t)$, namely $s_0^\star(g_t^\star(\vz_t))=-\Sigma^{-1}g_t^\star(\vz_t)$.

\begin{proposition}[Gaussian score-surrogate error]
\label{prop:gaussian-score-surrogate-error}
Let $\Sigma=U\operatorname{diag}(\lambda_1,\dots,\lambda_d)U^\top$ and set $m_{t,i}:=1+\alpha_t(\lambda_i-1)$. In the eigenbasis of $\Sigma$, the $i$-th coordinate of the error $e_t(\vz_t):=\bar s_t(\vz_t)-s_0^\star(g_t^\star(\vz_t))$ is
$$
(e_t)_i
=
a_{t,i}(\vz_t)_i,
\qquad
a_{t,i}
=
\frac{
1+\sqrt{\alpha_t}(\lambda_i-1)
-
\sqrt{\lambda_i}\sqrt{m_{t,i}}
}{
\sqrt{\lambda_i}(1-\sqrt{\alpha_t})\sqrt{m_{t,i}}
}.
$$
Consequently, for $\vz_t\sim\mathcal N(0,M_t)$,
$$
\mathbb E\|e_t(\vz_t)\|^2
=
\sum_{i=1}^d
\frac{
\left[
1+\sqrt{\alpha_t}(\lambda_i-1)
-
\sqrt{\lambda_i}\sqrt{m_{t,i}}
\right]^2
}{
\lambda_i(1-\sqrt{\alpha_t})^2
}.
$$
Moreover, $\mathbb E\|s_0^\star(g_t^\star(\vz_t))\|^2=\operatorname{Tr}(\Sigma^{-1})=\sum_{i=1}^d\lambda_i^{-1}$, so the relative mean-square error is
$$
\mathcal E_{\mathrm{rel}}(t)
=
\frac{
\sum_{i=1}^d
\frac{
\left[
1+\sqrt{\alpha_t}(\lambda_i-1)
-
\sqrt{\lambda_i}\sqrt{m_{t,i}}
\right]^2
}{
\lambda_i(1-\sqrt{\alpha_t})^2
}
}{
\sum_{i=1}^d \lambda_i^{-1}
}.
$$
\end{proposition}

\begin{proof}
In the eigenbasis of $\Sigma$, Lemma~\ref{lem:gaussian-pf-map} gives $(g_t^\star(\vz_t))_i=\sqrt{\lambda_i/m_{t,i}}\,(\vz_t)_i$. Hence
$$
(\bar s_t(\vz_t))_i
=
\frac{
\sqrt{\alpha_t}\sqrt{\lambda_i/m_{t,i}}-1
}{
1-\sqrt{\alpha_t}
}(\vz_t)_i,
$$
whereas
$$
(s_0^\star(g_t^\star(\vz_t)))_i
=
-\frac{1}{\lambda_i}\sqrt{\frac{\lambda_i}{m_{t,i}}}(\vz_t)_i
=
-\frac{1}{\sqrt{\lambda_i}\sqrt{m_{t,i}}}(\vz_t)_i.
$$
Subtracting gives the coefficient $a_{t,i}$. Since $(\vz_t)_i\sim\mathcal N(0,m_{t,i})$, the mean-square error is obtained by summing $a_{t,i}^2m_{t,i}$ over the eigen-directions. Finally, since $g_t^\star(\vz_t)\sim\mathcal N(0,\Sigma)$, the clean-score energy is $\mathbb E\|\Sigma^{-1}g_t^\star(\vz_t)\|^2=\operatorname{Tr}(\Sigma^{-1})$.
\end{proof}

The expression above immediately shows that the approximation is exact for the VP reference Gaussian. Indeed, if $\Sigma=I$, then $\lambda_i=1$ and $m_{t,i}=1$ for every $i$, so every numerator in Proposition~\ref{prop:gaussian-score-surrogate-error} vanishes. Therefore $\bar s_t(\vz_t)=s_0^\star(g_t^\star(\vz_t))$ for all $t$.

We also obtain the small-time order of the error.

\begin{corollary}[Small-time behaviour]
\label{cor:gaussian-small-time-order}
Assume that the spectrum of $\Sigma$ is contained in $[\lambda_{\min},\lambda_{\max}]$ with $\lambda_{\min}>0$. Then, as $t\to0$,
$$
\mathbb E\left\|
\bar s_t(\vz_t)-s_0^\star(g_t^\star(\vz_t))
\right\|^2
=
O\!\left((1-\sqrt{\alpha_t})^2\right)
=
O(\sigma_t^4).
$$
More precisely, in eigen-direction $i$,
$$
\mathbb E[(e_t)_i^2]
=
\frac{(\lambda_i-1)^2}{4\lambda_i^3}
(1-\sqrt{\alpha_t})^2
+
o\!\left((1-\sqrt{\alpha_t})^2\right)
=
\frac{(\lambda_i-1)^2}{16\lambda_i^3}
\sigma_t^4
+
o(\sigma_t^4).
$$
\end{corollary}

\begin{proof}
Let $\delta_t:=1-\sqrt{\alpha_t}$. Since $\alpha_t=(1-\delta_t)^2$, we have $\sigma_t^2=2\delta_t-\delta_t^2$. A Taylor expansion gives
$$
1+\sqrt{\alpha_t}(\lambda_i-1)
-
\sqrt{\lambda_i}\sqrt{m_{t,i}}
=
-\frac{\lambda_i-1}{2\lambda_i}\delta_t^2
+
O(\delta_t^3).
$$
Substituting this expansion into the exact expression of Proposition~\ref{prop:gaussian-score-surrogate-error} yields $\mathbb E[(e_t)_i^2]=\frac{(\lambda_i-1)^2}{4\lambda_i^3}\delta_t^2+o(\delta_t^2)$. Since $\delta_t=\sigma_t^2/(1+\sqrt{\alpha_t})$, this is equivalently $\mathbb E[(e_t)_i^2]=\frac{(\lambda_i-1)^2}{16\lambda_i^3}\sigma_t^4+o(\sigma_t^4)$.
\end{proof}

This analysis supports the CM-induced replacement used in the prior subflow. If $g_\theta\simeq g_t^\star$, then
$$
(1+\sqrt{\alpha_t})s_t^{\mathrm{CM}}(\vz_t;c)
=
(1+\sqrt{\alpha_t})
\frac{\sqrt{\alpha_t}g_\theta(\vz_t,t,c)-\vz_t}{\sigma_t^2}
$$
approximates the clean prior score evaluated at the reverse-flow endpoint. The approximation is exact for the standard Gaussian prior and its bias scales with the deviation of the covariance eigenvalues $\lambda_i$ from $1$.

\begin{figure}[h]
    \centering
    \centering
    \includegraphics[width=\linewidth]{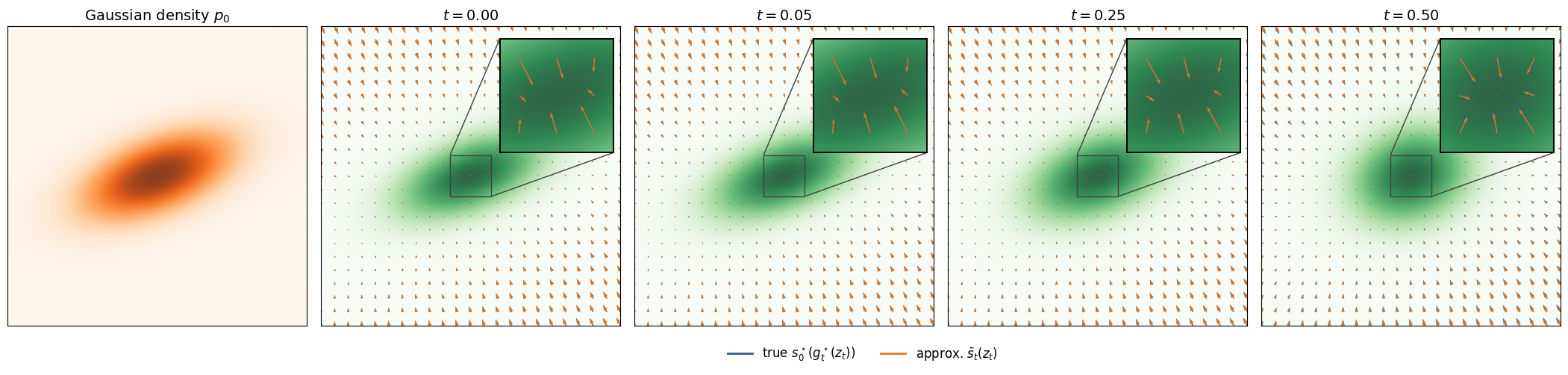}
    \caption{The plots show that, in the Gaussian prior setting, the normalised surrogate tracks the target score well for small and moderate diffusion times. Times $t\in[0,1]$}
    \label{fig:cm-score-gaussian-validation}

\end{figure}

This is consistent with the latent spaces used in latent diffusion models. The VAE is trained with a KL regularisation toward a standard Gaussian latent prior, and the latent scaling used by LDMs further encourages approximately whitened latent coordinates. Therefore, when the empirical latent covariance is close to $I$, the Gaussian calculation predicts a small bias for the CM-induced score surrogate. This provides a simple explanation for the stability of replacing the inaccessible clean prior score by a positive-time CM-induced score in latent-space posterior sampling.

\section{Algorithm}\label{sec:algorithm}
We now give the pseudocode for simulating the gradient flow. We use the notation $\vz_{k}$ to denote the clean particles, and $\vz_{t_k}$ for the ones diffused to the time $t_k\in [0,T]$ and at step $k \in \mathbb{N}$ of the Euler discretisation of the ODE.

\begin{algorithm}[h]
\caption{Consistency-regularised Wasserstein Gradient Flow (CWGF)}\label{algo:cwgf}
\begin{algorithmic}[1]
\State \textbf{Inputs:} observation $\vy$, forward operator $\mathcal A$, noise level $\sigma_\vy$, particles $\{\vz_0^{(n)}\}_{n=1}^N$, initial prompt $c_0$, CM flow map $g_\theta(\cdot,t,c)$, score $s_{\theta}(\cdot,t;c)$, VAE encoder-decoder $(\mathcal E_{\phi^-},\mathcal D_{\phi^-})$, decoder std. $\sigma_{\mathrm{dec}}$, encoder covariance $\Sigma_{\phi^-}$, VAE weight $\lambda$, timestep schedule $t(k)$, weights $w(t)$, step sizes $\eta_c,\eta_R,\eta_L$, iterations $K$.
\State Set $\Sigma_\vy \gets \left(\sigma_{\mathrm{dec}}^{-2}I+\sigma_\vy^{-2}A^\top A\right)^{-1}$.
\For{$k=0,\dots,K-1$}
    \State Set $t_k\gets t(k)$ and get $\alpha_{t_k},\sigma_{t_k}$.
    \For{$n=1,\dots,N$}
        \State Sample $\varepsilon_k^{(n)}\sim\mathcal N(0,I)$ and set
        $
        \vz_{t_k}^{(n)}
        \gets
        \sqrt{\alpha_{t_k}}\vz_k^{(n)}
        +
        \sigma_{t_k}\varepsilon_k^{(n)}.
        $
    \EndFor

    \State \textbf{Prompt step.} Estimate the prompt gradient by
    \begin{equation*}
    \widehat{\nabla_c\mathcal R}
    \gets
    \nabla_c
    \frac{1}{2N}
    \sum_{n=1}^N
    w(t_k)
    \left\|
    -\frac{\vz_{t_k}^{(n)}-\sqrt{\alpha_{t_k}}\vz_k^{(n)}}{\sigma_{t_k}^2}
    -
    s_{\theta}(\vz_{t_k}^{(n)},t_k;c_k)
    \right\|_2^2 .
    \end{equation*}
    \State Update
    $
    c_{k+1}\gets c_k-\eta_c\,\widehat{\nabla_c\mathcal R}.
    $

    \State \textbf{Prior subflow.}
    \For{$n=1,\dots,N$}
        \State Compute kernel weights
        \begin{equation*}
        \pi_{nm}(t_k)
        \gets
        \frac{
        \exp\!\left(
        -\frac{\|\vz_{t_k}^{(n)}-\sqrt{\alpha_{t_k}}\vz_k^{(m)}\|^2}{2\sigma_{t_k}^2}
        \right)
        }{
        \sum_{j=1}^N
        \exp\!\left(
        -\frac{\|\vz_{t_k}^{(n)}-\sqrt{\alpha_{t_k}}\vz_k^{(j)}\|^2}{2\sigma_{t_k}^2}
        \right)
        }.
        \end{equation*}
        \State Update
        \begin{equation*}
        \bar\vz_k^{(n)}
        \gets
        \vz_k^{(n)}
        +
        \eta_R\left(\,g_\theta(\vz_{t_k}^{(n)},t_k,c_k) - \vz_k^{(n)}\right)
        +
        \eta_{R}
        \left(
        \vz_k^{(n)}
        -
        \sum_{m=1}^N\pi_{nm}(t_k)\vz_k^{(m)}
        \right).
        \end{equation*}
    \EndFor

    \State \textbf{Likelihood subflow.}
    \For{$n=1,\dots,N$}
        \State Compute the Gaussian pixel-space posterior mean
        \begin{equation*}
        m(\bar\vz_k^{(n)})
        \gets
        \Sigma_\vy
        \left(
        \sigma_{\mathrm{dec}}^{-2}\mathcal D_{\phi^-}(\bar\vz_k^{(n)})
        +
        \sigma_\vy^{-2}A^\top\vy
        \right).
        \end{equation*}
        \State Compute the encoder-based likelihood drift
        \begin{equation*}
        \widehat g(\bar\vz_k^{(n)};\vy)
        \gets
        \lambda\Sigma_{\phi^-}^{-1}
        \left(
        \mathcal E_{\phi^-}(m(\bar\vz_k^{(n)}))
        -
        \bar\vz_k^{(n)}
        \right)
        +
        \lambda\bar\vz_k^{(n)}.
        \end{equation*}
        \State Update
        $
        \vz_{k+1}^{(n)}
        \gets
        \bar\vz_k^{(n)}
        +
        \eta_L\,\lambda^{-1}\Sigma_{\phi^-}\widehat g(\bar\vz_k^{(n)};\vy).
        $
    \EndFor
\EndFor
\State \Return restored samples $\{\mathcal D_{\phi^-}(\vz_K^{(n)})\}_{n=1}^N$ and final prompt $c_K$.
\end{algorithmic}
\end{algorithm}

\subsection{Choice of hyperparameters}
\label{app:hyperparameters}
We report the hyperparameters used for both the FFHQ-512 and ImageNet experiments in~\Cref{tab:comparison_ffhq_imagenet}. Across all three inverse problems, we use the LCM-LoRA prior~\cite{Luo2023LCMLoRAAU} based on Stable Diffusion 1.5, run CWGF for $K=16$ iterations, use $N=1$ particle, enable prompt optimisation, and use cyclic timestep sampling as discussed in Section~\ref{sec:experiments}. We set the measurement noise to $\sigma_\vy=0.01$, and use the prompt \texttt{a photo of a face}. The problem-dependent hyperparameters are reported in~\Cref{tab:hyperparams-ffhq-positive}. Here, $\eta_z$ denotes the particle step size, $\eta_c$ the prompt step size, $\sigma_{\mathrm{dec}}$ the decoder standard deviation used in the likelihood approximation, and $w_t$ the timestep-dependent weights used in the prior/likelihood update. The linear weights $w_t$ are defined by $w(t)=0.1+0.8(t/T)$ for $t\in\{\texttt{999,879,759,639,499,379,259,139}\}$ and $T=999$. $r_c$ is the radius of the ball centred in the starting prompt $c_0$, to which we project every time we apply the prompt update step. Additionally, similar to techniques used in training and distilling diffusion models~\citep{songMaximumLikelihoodTraining2021,kingmaUnderstandingDiffusionObjectives2023}, we weigh the objective in~\eqref{eq:score_kl} by the inverse signal-to-noise ratio $W(t)=\sigma_s^2/\alpha_s$~\citep{mardanivariationalperspectivesolving2023,wangUniinstructOnestepDiffusion2025,choiRethinkingDesignSpace2026}.
\begin{table}[htbp]
\centering
\caption{
Hyperparameters used for positive-prompt experiments. Common settings are $K=16$, $N=1$, $8$ LCM timesteps, cyclic timestep sampling, and $\sigma_\vy=0.01$.
}
\label{tab:hyperparams-ffhq-positive}
\resizebox{0.5\textwidth}{!}{%
\small
\setlength{\tabcolsep}{5pt}
\renewcommand{\arraystretch}{1.12}
\begin{tabular}{l c c c c c}
\toprule
\textbf{Problem} 
& $\boldsymbol{\sigma_{\mathrm{dec}}}$ 
& $\boldsymbol{\eta_z}$ 
& $\boldsymbol{\eta_c}$ 
& $\boldsymbol{r_c}$ 
& $\boldsymbol{w_t}$ \\
\midrule
SR $\times 8$
& 0.25
& 1.0
& 0.03
& 15.0
& linear \\
Gaussian deblur
& 0.08
& 1.0
& 0.66
& 15.0
& linear \\
Motion deblur
& 0.05
& 1.0
& 0.26
& 15.0
& linear \\
\bottomrule
\end{tabular}
}
\end{table}

For completeness, we also provide in Table~\ref{tab:hyperparams-ffhq-negative} the template for the negative-prompt setting, in which the prompt is \texttt{A photo of a cat}.

\begin{table}[htbp]
\centering
\caption{
Hyperparameters used for the FFHQ-512 negative-prompt experiments with prompt \texttt{A photo of a cat}. Common settings are $K=16$, $N=1$, $8$ LCM timesteps, cyclic timestep sampling,  and $\sigma_\vy=0.01$. Here, $r_c$ denotes the prompt projection radius.
}
\label{tab:hyperparams-ffhq-negative}
\resizebox{0.5\textwidth}{!}{%
\small
\setlength{\tabcolsep}{5pt}
\renewcommand{\arraystretch}{1.12}
\begin{tabular}{l c c c c c}
\toprule
\textbf{Problem} 
& $\boldsymbol{\sigma_{\mathrm{dec}}}$ 
& $\boldsymbol{\eta_z}$ 
& $\boldsymbol{\eta_c}$ 
& $\boldsymbol{r_c}$ 
& $\boldsymbol{w_t}$ \\
\midrule
SR $\times 8$
& 0.25
& 1.0
& 3.0
& 15.0
& linear \\
Gaussian deblur
& 0.08
& 1.0
& 3.0
& 15.0
& linear \\
Motion deblur
& 0.05
& 1.0
& 3.0
& 15.0
& linear \\
\bottomrule
\end{tabular}
}
\end{table}

\input{sections/appendix-relatedworks}

\section{Baselines}\label{app:baselines}

We briefly recall the main latent-diffusion baselines considered in our comparisons. All methods combine a pretrained LDM prior with a data-consistency correction, but they differ substantially in how the likelihood term is incorporated and how text conditioning is handled.

\subsection{LDPS}
LDPS is the direct latent-space counterpart of image-domain DPS~\citep{chung2023diffusion}. Let $\hat{\vz}_0 := \mathbb{E}[\vz_0 \mid \vz_t]$ denote the clean latent estimate associated with the current noisy iterate $\vz_t$. The update rule is given by
\[
\vz_{t-1}
=
\mathrm{DDIM}(\vz_t)
-
\rho \nabla_{\vz_t}\| \vy - A\mathcal D(\hat{\vz}_0)\|_2^2,
\]
where $\rho$ denotes the step size and $\mathrm{DDIM}(\cdot)$ represents a single DDIM~\citep{Song2020DenoisingDI} sampling step. In practice, a constant step size $\rho=1$ is typically used. Thus, LDPS injects the data-fidelity term via a single latent gradient correction applied to the decoded clean estimate.

\subsection{PSLD}
PSLD~\citep{routsolvinglinearinverse2023} refines LDPS by adding a latent-manifold stabilisation term. In the linear setting, it updates the latent variable by combining a standard DDIM step with two gradient corrections,
\[
\vz_{t-1}
=
\mathrm{DDIM}(\vz_t)
-
\rho \nabla_{\vz_t}\|\vy-A\mathcal D(\hat{\vz}_0)\|_2^2
-
\gamma \nabla_{\vz_t}
\Bigl\|
\hat{\vz}_0
-
\mathcal E\!\Bigl(
A^\top \vy + (I-A^\top A)\mathcal D(\hat{\vz}_0)
\Bigr)
\Bigr\|_2^2.
\]
The first term corresponds to the LDPS correction, while the second encourages the clean latent estimate to remain stable under the decode--project--re-encode cycle. In practice, the method uses a fixed step size $\rho=1$ together with a regularisation weight $\gamma=0.1$. This additional constraint is designed to keep the iterates close to the natural latent autoencoding manifold.

\subsection{P2L}
P2L~\citep{Chung2023PrompttuningLD} alternates between prompt adaptation and latent refinement. At a given diffusion step $t$, let $\hat{\vz}_0(c):=\mathbb{E}[\vz_0 \mid \vz_t, c]$ denote the current estimate of the clean latent variable conditioned on the text embedding $c$. The first stage updates the text embedding by approximately maximising the conditional posterior $p(c\mid \vz_t,\vy)$, leading to the gradient-based update
\[
\nabla_c \log p(c\mid \vz_t,\vy)
\approx
-\nabla_c \bigl\|A\mathcal D(\hat{\vz}_0(\vz_t,c))-\vy\bigr\|_2^2,
\]
which is typically optimised with Adam~\citep{kingma2017adammethodstochasticoptimization}. Once an updated embedding $c_t^\star$ is obtained, the second stage refines the latent variable by approximately maximising $p(\vz_t\mid \vy,c_t^\star)$ as done by LDPS, i.e.,
\[
\nabla_{\vz_t}\log p(\vz_t\mid \vy,c_t^\star)
\approx
s_\theta(\vz_t,c_t^\star)
+
\rho_t \nabla_{\vz_t}
\bigl\|A\mathcal D(\hat{\vz}_0(\vz_t,c_t^\star))-\vy\bigr\|_2^2,
\]
where $s_\theta(\cdot,c_t^\star)$ is the diffusion-model score and $\rho_t$ balances prior and likelihood contributions. Hence, P2L interleaves text-embedding optimisation with posterior-guided latent sampling.

\subsection{TReg}
TReg~\citep{kim2025regularizationtextslatentdiffusion} formulates the reconstruction as an ADMM-style~\citep{Boyd2011DistributedOA} proximal optimisation problem in latent space. Given the current clean-latent prediction $\hat{\vz}_0$, it considers
\[
\min_{\vx,\vz}\;
\ell_{\mathrm{MAP}}(\vz)+\gamma \ell_{\mathrm{TReg}}(\vz)
\qquad\text{s.t.}\qquad
\vx=\mathcal D(\vz),
\]
where
\[
\ell_{\mathrm{TReg}}(\vz)=\|\vz-\hat{\vz}_0\|_2^2
\]
and the MAP term is
\[
\ell_{\mathrm{MAP}}(\vz)
=
-\log p(\vz\mid \mathcal D(\vz),\vy)-\log p(\vy\mid \mathcal D(\vz))
=
\frac{\|\vz-\mathcal E(\mathcal D(\vz))\|_2^2}{2\sigma_E^2}
+
\frac{\|\vy-A(\mathcal D(\vz))\|_2^2}{2\sigma^2}.
\]
The method first solves the pixel-space subproblem
\[
\hat{\vx}_0(\vy)
=
\argmin_{\vx}\;
\frac{\|\vy-A(\vx)\|_2^2}{2\sigma^2}
+
\lambda \|\vx-\mathcal D(\hat{\vz}_0)\|_2^2,
\]
and then updates the latent representation through
\[
\hat{\vz}_0^{\mathrm{ema}}
=
\argmin_{\vz}\;
\zeta \|\vz-\mathcal E(\hat{\vx}_0(\vy))\|_2^2
+
\gamma \|\vz-\hat{\vz}_0\|_2^2
=
\alpha_{t-1}\mathcal E(\hat{\vx}_0(\vy))
+
(1-\alpha_{t-1})\hat{\vz}_0,
\]
with $\alpha_{t-1}=\zeta/(\zeta+\gamma)$. After this proximal correction, a DDIM step is performed. TReg can additionally adapt the null-text embedding via the so-called \emph{Adaptive Negation} step,
\[
c_{\emptyset}
\leftarrow
c_{\emptyset}
-
\eta \nabla_{\emptyset} T_{\mathrm{img}}(\hat{\vx}_0(\vy),c_{\emptyset}),
\]
where $T_{\mathrm{img}}$ denotes the CLIP image encoder~\citep{Radford2021LearningTV}. This allows the method to preserve a user-provided prompt while still adjusting the unconditional branch used in classifier-free guidance. 

\subsection{LATINO}\label{app:LATINO}

LATINO~\citep{Spagnoletti_2025_ICCV} is a plug-and-play inverse solver that combines a prompt-conditioned latent consistency prior with an exact likelihood correction in pixel space. Given a measurement $\vy$, a prompt embedding $c$, and a posterior of interest
\[
p(\vx\mid \vy,c)\propto p(\vy\mid \vx)\,p(\vx\mid c),
\]
LATINO may be interpreted as a split-step discretisation of an overdamped Langevin dynamics targeting $p(\vx\mid \vy,c)$: a \emph{prior step} transports the current iterate toward the prompt-conditioned prior by means of the stochastic auto-encoder of Definition~\ref{def:sae-latino}, and a \emph{likelihood step} enforces data consistency by applying the proximal map associated with $-\log p(\vy\mid \vx)$.

\begin{definition}[Stochastic auto-encoder]\label{def:sae-latino}
For a diffusion time $t\in(0,T]$ and a prompt embedding $c$, we define the stochastic auto-encoder (SAE) associated with the latent consistency prior as the pair $(\mathfrak E_t,\mathfrak D_{t,c})$, where:
\begin{itemize}
    \item the stochastic encoder $\mathfrak E_t$ maps $\vx$ to a noisy latent $\vz_t$ according to
    \[
    \vz_t \mid \vx \sim \mathcal N\!\bigl(\sqrt{\alpha_t}\,\mathcal E_{\phi^-}(\vx),\,\sigma_t^2 I\bigr),
    \qquad \sigma_t^2=1-\alpha_t,
    \]
    \item the stochastic decoder $\mathfrak D_{t,c}$ maps $\vz_t$ back to image space through
    \[
    \mathfrak D_{t,c}(\vz_t)
    :=
    \mathcal D_{\phi^-}\!\bigl(\hat g_{\theta^-}(\vz_t,t,c)\bigr).
    \]
\end{itemize}
Equivalently, the associated prior transport on image space is the Markov kernel
\[
\vx \mapsto \mathfrak D_{t,c}\circ \mathfrak E_t(\vx).
\]
\end{definition}

More precisely, starting from the current image iterate $\vx^{(k-1)}$, LATINO first maps it to latent space and perturbs it at diffusion time $t_k$,
\[
\vz_{t_k}^{(k)}
=
\sqrt{\alpha_{t_k}}\,\mathcal E_{\phi^-}\!\bigl(\vx^{(k-1)}\bigr)
+
\sigma_{t_k}\,\epsilon^{(k)},
\qquad
\epsilon^{(k)}\sim\mathcal N(0,I).
\]
The latent consistency model then predicts a cleaned latent
\[
\hat{\vz}_0^{(k)}=\hat g_{\theta^-}\!\bigl(\vz_{t_k}^{(k)},t_k,c\bigr),
\]
which is decoded back to pixel space as
\[
\vu^{(k)}=\mathcal D_{\phi^-}\!\bigl(\hat{\vz}_0^{(k)}\bigr)
=\mathfrak D_{t_k,c}\!\bigl(\vz_{t_k}^{(k)}\bigr).
\]
Finally, the measurement model is enforced through the proximal correction
\[
\vx^{(k)}
=
\operatorname{prox}_{\delta_k g_y}\!\bigl(\vu^{(k)}\bigr),
\qquad
g_y(\vx):=-\log p(\vy\mid \vx),
\]
where $\delta_k>0$ is an iteration-dependent parameter balancing the prior and likelihood contributions. The full method is described in Algorithm~\ref{alg:latino_appendix}.

\begin{algorithm}[H]
\caption{LATINO}
\label{alg:latino_appendix}
\begin{algorithmic}[1]
\State \textbf{Input:} measurement \(\vy\), prompt embedding \(c\), initialisation \(\vx^{(0)}=\mathcal A^\dagger \vy\), number of steps \(K\), schedules \(\{t_k\}_{k=1}^K\) and \(\{\delta_k\}_{k=1}^K\)
\For{$k=1,\dots,K$}
    \State Sample \(\epsilon^{(k)}\sim\mathcal N(0,I)\)
    \State \(\vz_{t_k}^{(k)} \gets \sqrt{\alpha_{t_k}}\,\mathcal E_{\phi^-}(\vx^{(k-1)})+\sigma_{t_k}\epsilon^{(k)}\)
    \State \(\vu^{(k)} \gets \mathcal D_{\phi^-}\!\bigl(\hat g_{\theta^-}(\vz_{t_k}^{(k)},t_k,c)\bigr)\)
    \State \(\vx^{(k)} \gets \operatorname{prox}_{\delta_k g_y}(\vu^{(k)})\)
\EndFor
\State \textbf{return} \(\vx^{(K)}\)
\end{algorithmic}
\end{algorithm}

In practice, the algorithm is warm-started from \(\mathcal A^\dagger \vy\) and uses an annealed schedule \(\{(t_k,\delta_k)\}_{k=1}^K\), with decreasing diffusion times and task-dependent proximal strengths.

\subsection{LATINO-PRO}\label{app:latino-pro}

LATINO-PRO extends LATINO by treating the prompt embedding \(c\) as an unknown parameter to be estimated from the observation. Rather than fixing \(c\) a priori, it adopts a maximum marginal likelihood estimation (MMLE) strategy and seeks
\[
\hat c(\vy)\in\arg\max_{c} p(\vy\mid c),
\qquad
p(\vy\mid c)=\int p(\vy\mid \vx)\,p(\vx\mid c)\,\dd \vx.
\]
The optimised prompt is then used to sample from the empirical-Bayes posterior \(p(\vx\mid \vy,\hat c(\vy))\). Through Fisher's identity we get
\[
\nabla_c \log p(\vy\mid c)
=
\mathbb E_{p(\vx\mid \vy,c)}
\bigl[\nabla_c \log p(\vx\mid c)\bigr],
\]
which shows that the marginal-likelihood gradient can be approximated by posterior samples. LATINO-PRO therefore alternates between:

\begin{enumerate}
    \item running a short LATINO chain targeting \(p(\vx\mid \vy,c_m)\), starting from the current state and using the current prompt \(c_m\);
    \item updating the prompt embedding through a stochastic approximation proximal gradient (SAPG) step,
    \[
    c_{m+1}
    =
    \Pi_{\mathcal C}\!\left(c_m+\gamma_m h_m\right),
    \]
    where \(\gamma_m>0\) is a step size, \(\Pi_{\mathcal C}\) denotes projection onto an admissible set \(\mathcal C\), and \(h_m\) is a Monte Carlo approximation of \(\nabla_c \log p(\vy\mid c_m)\).
\end{enumerate}

In the LATINO-PRO implementation, the gradient estimator is computed in latent space by differentiating the transition densities induced by the stochastic latent auto-encoding recursion. Concretely, after generating latent states
\[
\bigl\{\vz_{t_1}^{(1)},\dots,\vz_{t_{N_m}}^{(N_m)}\bigr\}
\]
along a short LATINO trajectory (i.e. using $N_m \sim K/2$ steps), one uses the approximation
\[
h_m
\approx
\nabla_c \log p\!\left(\vz_{t_1}^{(1)},\dots,\vz_{t_{N_m}}^{(N_m)}\mid c_m\right),
\]
which can be evaluated using automatic differentiation via the latent consistency model (see~\citep {Spagnoletti_2025_ICCV}[Appendix A] for details). Similar to CWGF, the prompt update acts in latent space and avoids differentiating through the full pixel-space proximal correction.

Overall, LATINO-PRO may thus be viewed as an alternating scheme:
\begin{center}
\emph{short posterior sampling with LATINO} \(\longrightarrow\) \emph{prompt update by SAPG} \(\longrightarrow\) \emph{repeat}.
\end{center}
After a small number of such outer iterations, one obtains an optimised prompt embedding \(\hat c(\vy)\), and a final LATINO pass can then be run with this calibrated prompt to produce a higher-quality reconstruction.

\section{Enhanced sampling variability on MNIST}
\label{app:mnist}
While the CWGF algorithm has proven to be SOTA on multiple image restoration tasks, the ablation study in Appendix~\ref{sec:ablations} shows how a single particle is enough to approximate $\mu(\vz \mid \vy)$ for the sake of the restoration task. This phenomenon seems counterintuitive, but it is actually explainable as the latent space of our SD1.5-based LCM is too big to appreciate significant interactions between particles when $N < 10$ as in our experiments. Indeed, as clarified in Section~\ref{sec:prior-subflow}, when $N>1$ in~\eqref{eq:denoiser-discrepancy-update} we get an extra interaction term $\vz_{0}^{(i)}-\sum_{j=1}^N \pi_{ij}(s)\,\vz_0^{(j)}$, where $\pi_{ij}$ can be read as a similarity matrix between the different particles at noise level $s$. Thus, when the particles are too similar (and we would like to enhance the differences), this extra term pushes away from the weighted mean, while a diagonal $\pi_{ij} \sim \Id$ annihilates the term, and the algorithm processes the $N$ samples in parallel, as LATINO does. Thus, observing $\pi_{ij} \sim \Id$, as in our experiments, indicates that the particles are already sufficiently distinct to require only limited reshuffling during sampling.

The situation may be different when the space in which the particles $\{\vz^{(i)}\}_i$ live is much smaller. To illustrate this different scenario, we trained a small LCM on the MNIST dataset~\citep{Deng2012TheMD}, and, to further reduce the dimensionality of the latent space, we embedded the samples in the latent space of a VAE with $d=10$ dimensions. The training of the VAE ($\mathcal{D}_{\phi^-}, \mathcal{E}_{\phi^-}$) followed~\citep{dubois2019disentangling,bickfordsmith2024making}, while the training of the LCM is inspired by~\citep{lu2025simplifyingstabilizingscalingcontinuoustime}, for which the timesteps are normalised to be $\in[0,\pi/2]$. The implementation uses CWGF (no prompt opt.) following Algorithm~\ref{algo:cwgf}, since the LCM is unconditional in this scenario. Concerning the inverse problem considered, we define a simple inpainting with a central mask of shape $16\times16$, while $\sigma_\vy=0.01$. Other hyperparameters are set: $K=8$, $s$ takes 4 equi-spaced values in the Karras~\citep{karrasElucidatingDesignSpace2022} time schedule (so we cycle through them, and we use only the first two for $k>4$), $\texttt{lr}_z = 0.5$, $w_t=0.5$, $\sigma_{dec}=0.6$.

We then proceed to solve $\vy = \mathcal{A}\vx + \vn$ on a sample $\vx$ from the MNIST test set, and we compare two ways of doing so: (i) sampling $M = 1024$ particles from $p(\vx \mid \vy)$ with CWGF (no prompt opt.) in parallel, meaning that $N = M = 1024$, or (ii) repeating $M$ times the CWGF (no prompt opt.) scheme for $N=1$.

\begin{figure*}[h]
    \centering
    \includegraphics[width=\textwidth]{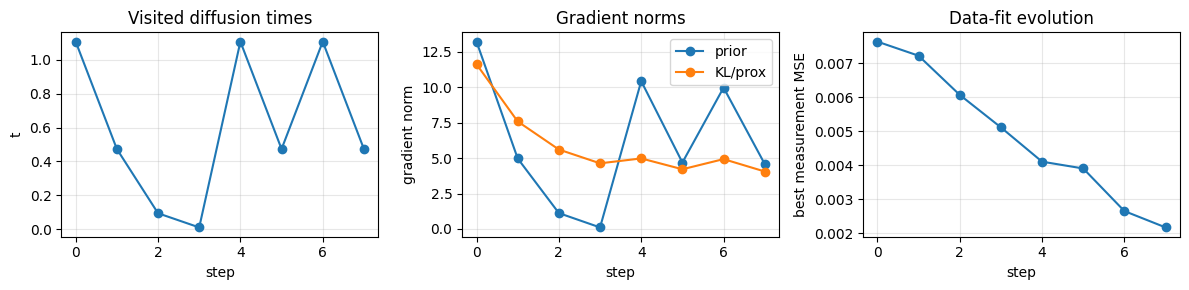}\\[-1mm]
    \includegraphics[width=\textwidth]{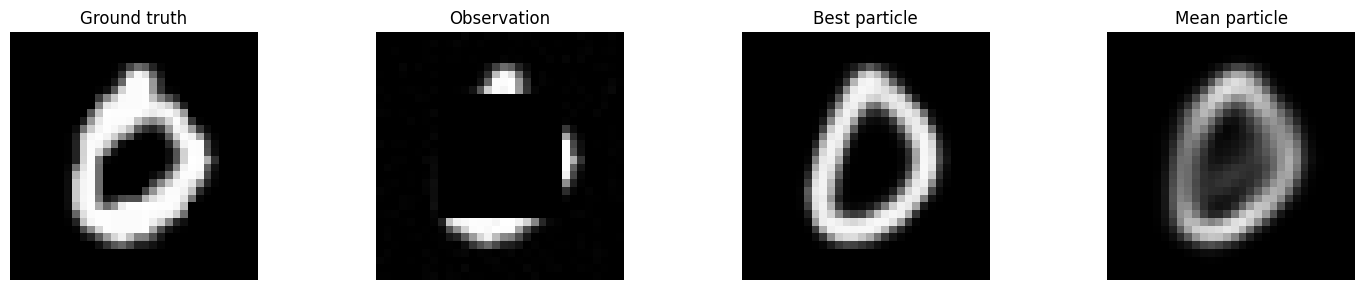}
    \caption{\textbf{MNIST box inpainting with CWGF (no prompt opt.).}
    \emph{Top:} visited diffusion times $t(k)\in[0,\pi/2]$, gradient norms (prior vs.\ likelihood/prox), and evolution of the best data-fit over iterations.
    \emph{Bottom:} ground truth, blurred measurement, best particle, and particle mean at the final iteration.}
    \label{fig:mnist_single_example}
    \vspace{-2mm}
\end{figure*}

We report the sampling dynamics and a representative reconstruction in Figure~\ref{fig:mnist_single_example}, confirming that the reconstruction runs successfully. To isolate the role of particle interactions, Figure~\ref{fig:mnist_1024_grids} compares joint evolution ($N=M$) against $M$ independent runs ($N=1$). A visual analysis of the samples shows greater variability in the $M=N$ case (i.e. ``all at once''), as expected from the discussion of the repulsive term. Figure~\ref{fig:mnist_1024_stats} summarises the resulting diversity in latent space and reconstruction quality, confirming the visual insights. Indeed, the first two Principal Components in the latent space show a more extensive coverage for the $M=N$ setting, proving a \emph{repulsion off the mean} phenomenon. Similarly, the PSNR histogram for this case shows a larger support and reaches a higher maximum PSNR. Finally, we trained a small classifier on the MNIST training dataset (achieving $\sim 98\%$ accuracy on the test set) to detect the digit distribution in the two cases. In Figure~\ref{fig:mnist_digit_dist_a}, we observe a larger number of classes retrieved when the algorithm samples the $1024$ particles in parallel. Indeed, in the ``batched'' case, we do not observe any digit labelled as $5$, while, as observed in Figure~\ref{fig:mnist_digit_dist_b}, it is a plausible solution to the inpainting problem, and we observe these cases when $M=N$.

\begin{figure*}[h]
    \centering
    \includegraphics[width=\textwidth]{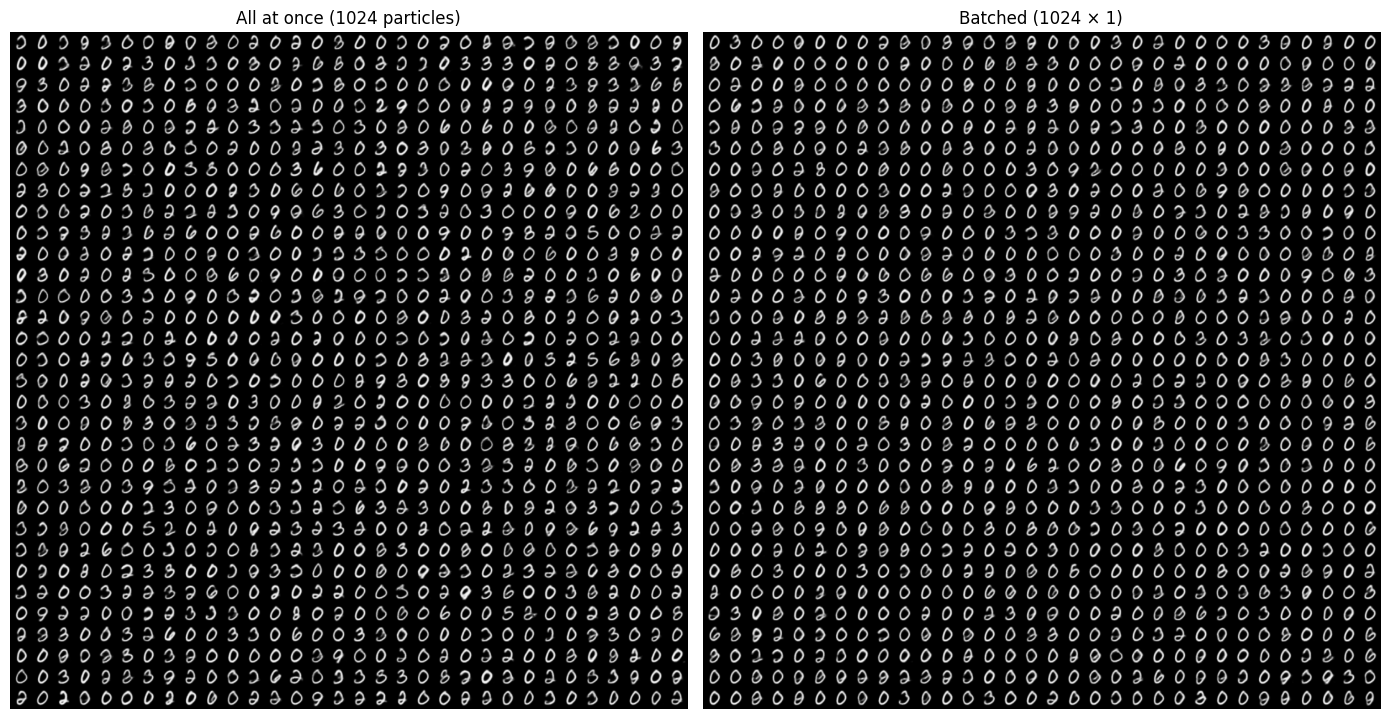}
    \caption{\textbf{Effect of particle interactions.}
    Samples obtained with $M=1024$ posterior draws using (left) CWGF (no prompt opt.) with $N=M$ particles evolved jointly, versus (right) $M$ independent runs with $N=1$ (no interaction term).}
    \label{fig:mnist_1024_grids}
    \vspace{-2mm}
\end{figure*}

\begin{figure*}[h]
    \centering
    \includegraphics[width=0.9\textwidth]{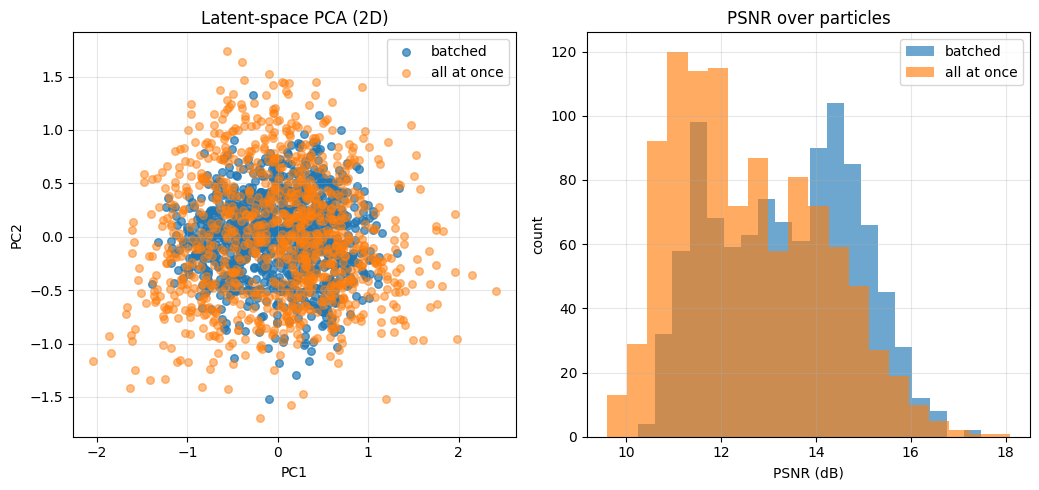}
    \caption{\textbf{Distributional comparison of posterior samples.}
    \emph{Left:} 2D PCA of latent samples, contrasting joint evolution ($N=M$) against independent runs ($N=1$ repeated).
    \emph{Right:} histogram of PSNR across the $M=1024$ samples for the two regimes.}
    \label{fig:mnist_1024_stats}
    \vspace{-2mm}
\end{figure*}

\begin{figure*}[h]
    \centering
    \begin{subfigure}[b]{0.45\textwidth}
        \centering
        \begin{minipage}[c][0.28\textheight][c]{\linewidth}
            \centering
            \includegraphics[width=\linewidth]{Experiments/MNIST/digits-dist.png}
        \end{minipage}
        \vspace{-1cm}
        \caption{Histogram of digit distribution.}
        \label{fig:mnist_digit_dist_a}
    \end{subfigure}\hfill
    \begin{subfigure}[b]{0.45\textwidth}
        \centering
        \begin{minipage}[c][0.28\textheight][c]{\linewidth}
            \centering
            \includegraphics[width=\linewidth]{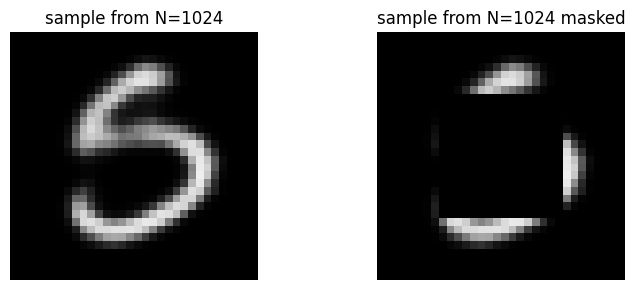}
        \end{minipage}
        \vspace{-1cm}
        \caption{Sample from class $5$ for $N=M=1024$.}
        \label{fig:mnist_digit_dist_b}
    \end{subfigure}
    \caption{\textbf{Classification of posterior samples.}}
    \label{fig:mnist_digit_dist}
    \vspace{-2mm}
\end{figure*}

Most importantly, we have experimental evidence that the matrix $\pi_{ij}$ does not collapse to $\Id$ except when $s$ is close to $0$, as shown in Figure~\ref{fig:mnist_pi_row}, where we plot the first $4$ iterations within CWGF (no prompt opt.) for $N=20$ particles.

\begin{figure*}[h]
    \centering
    \begin{minipage}{0.24\textwidth}
        \centering
        \includegraphics[width=\linewidth]{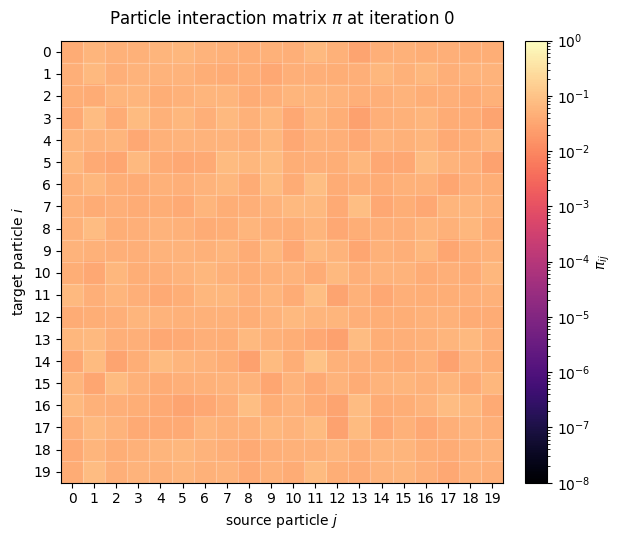}
    \end{minipage}\hfill
    \begin{minipage}{0.24\textwidth}
        \centering
        \includegraphics[width=\linewidth]{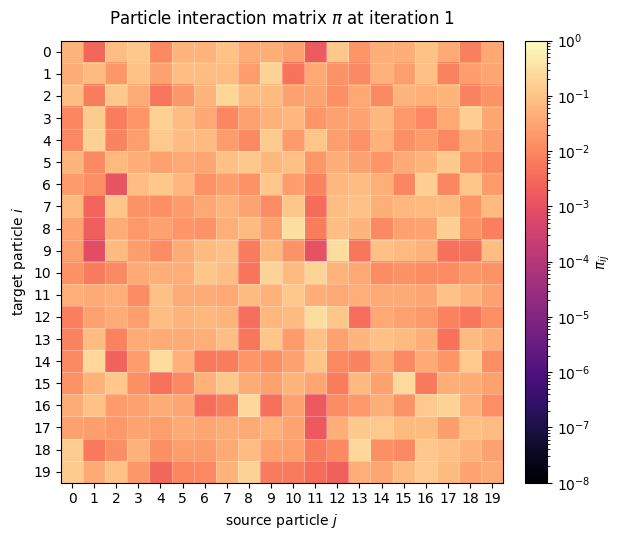}
    \end{minipage}\hfill
    \begin{minipage}{0.24\textwidth}
        \centering
        \includegraphics[width=\linewidth]{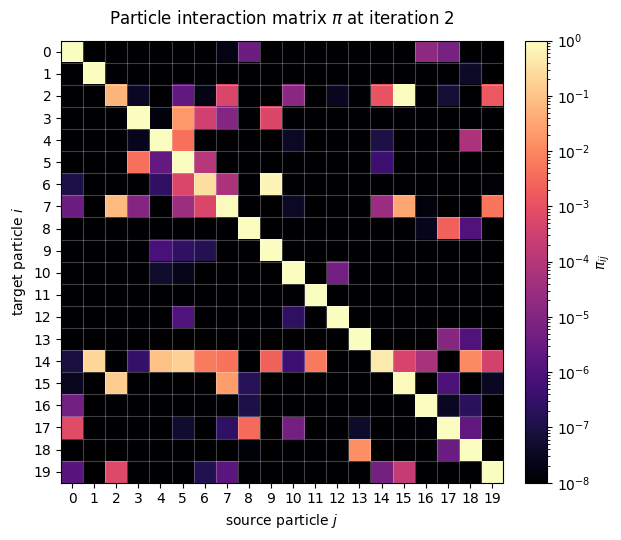}
    \end{minipage}\hfill
    \begin{minipage}{0.24\textwidth}
        \centering
        \includegraphics[width=\linewidth]{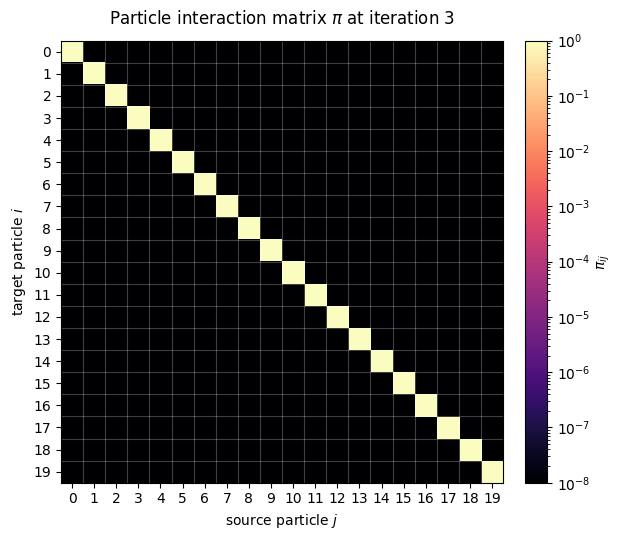}
    \end{minipage}
    \caption{\textbf{Evolution of the particle interaction matrix $\pi$.} Heatmaps of $\pi_{ij}$ at iterations $k=0,1,2,3$ (left to right).}
    \label{fig:mnist_pi_row}
    \vspace{-2mm}
\end{figure*}

We can then conclude that, in smaller latent spaces where the particles actually interact, the posterior reconstructions show a larger variability, allowing for non-obvious solutions that lie on the tails of the posterior distribution $p(\vx | \vy)$. We encourage future research to exploit this feature of the CWGF algorithm in contexts and applications where it is important to correctly sample the tail of the considered conditional distribution.

\section{Non-linear problems}\label{app:non-linear}

In addition to the linear Gaussian inverse problems considered in our theory, we also evaluate a more challenging \emph{nonlinear} restoration task based on JPEG compression. Given a clean image $\vx_0\in\mathsf{X}$, we generate the observation as
\begin{equation}
    \vy = \mathcal{J}_q(\vx_0 + \eta_0) + \vn,
    \qquad
    \eta_0 \sim \mathcal{N}(0,\sigma_{\vx_0}^2 I),
    \qquad
    \vn \sim \mathcal{N}(0,\sigma_\vy^2 I),
\end{equation}
where $\mathcal{J}_q$ denotes JPEG compression--decompression at quality factor $q$. This setting departs from the assumptions of \eqref{eq:simple_inv}, since $\mathcal{J}_q$ is nonlinear and non-differentiable due to quantisation, so the closed-form Gaussian posterior and the associated linear proximal operator are no longer available. In practice, we keep the same posterior-guided sampling scheme and replace the exact data-consistency step by an approximate proximal update based on the deterministic surrogate $\mathcal{A}_{\mathrm{det}}(\vx)=\mathcal{J}_q(\vx)$, while treating the stochastic perturbations $\eta_0$ and $\vn$ as part of the observed measurement. More precisely, at each iteration we approximately solve
\begin{equation}
    \mu(\vz)
    \approx
    \argmin_{\vu\in\mathsf{X}}
    \frac{1}{2}\|\vu-\vx\|_2^2
    + \frac{\sigma_{dec}^2}{2\sigma_\vy^2}\|\mathcal A_{\mathrm{det}}(\vu)-\vy\|_2^2.
\end{equation}
by a small number of projected gradient steps, using the identity as a practical approximation of the adjoint of $\mathcal{A}_{\mathrm{det}}$. This provides a simple and effective way to assess the robustness of our method beyond the strictly linear regime.

Table~\ref{tab:comparison_FFHQ_nonlinear} and Figure~\ref{fig:qualitative_FFHQ_nonlinear} show results on the FFHQ dataset when JPEG quality level is set to $q=20\%$, $\sigma_{\vx_0}=0.05$, and $\sigma_\vy = 0.01$.

\begin{table*}[h]
\centering
\footnotesize
\setlength{\tabcolsep}{4pt}
\renewcommand{\arraystretch}{1.1}%
\caption{Results on Noisy JPEG, all with noise $\sigma_\vy = 0.01$ on the FFHQ-512 val dataset. Prompt: \texttt{A photo of a face}. \textbf{Bold}:  best, \underline{underline}: second best.}
\label{tab:comparison_FFHQ_nonlinear}
\begin{tabular}{l c ccc ccc ccc}
\toprule
 &  & \multicolumn{3}{c}{\textbf{Noisy JPEG}} \\%
\cmidrule(lr){3-5}%
\textbf{Method} & \textbf{NFE$\downarrow$}
& \textbf{FID$\downarrow$} & \textbf{PSNR$\uparrow$} & \textbf{LPIPS$\downarrow$} \\
\midrule
LATINO & \textbf{8} & 34.91 & 25.87 & 0.448 \\%
LATINO-PRO & 65 & \underline{33.45} & \underline{26.28} & \underline{0.433} \\%
P2L & 400 & 75.57 & 18.28 & 0.529\\
TReg & 200 & 94.42 & 20.04 & 0.703 \\%
\midrule
\rowcolor{orange!20}
CWGF & \underline{16} & \textbf{30.71} & \textbf{26.33} & \textbf{0.355} \\%
\bottomrule
\end{tabular}
\end{table*}
\begin{figure*}[h]
\centering

\begin{minipage}{0.03\textwidth}
  \centering
  \begin{tabular}{c}
    \\[-2mm]
    \rotatebox{90}{\scalebox{.6}{\textbf{Noisy JPEG}}} \\[1mm]
  \end{tabular}
\end{minipage}%
\hfill
\begin{minipage}{0.15\textwidth}
    \centering \tiny\textbf{GT} \\[-1mm]
    \zoomedImage[width=\linewidth]{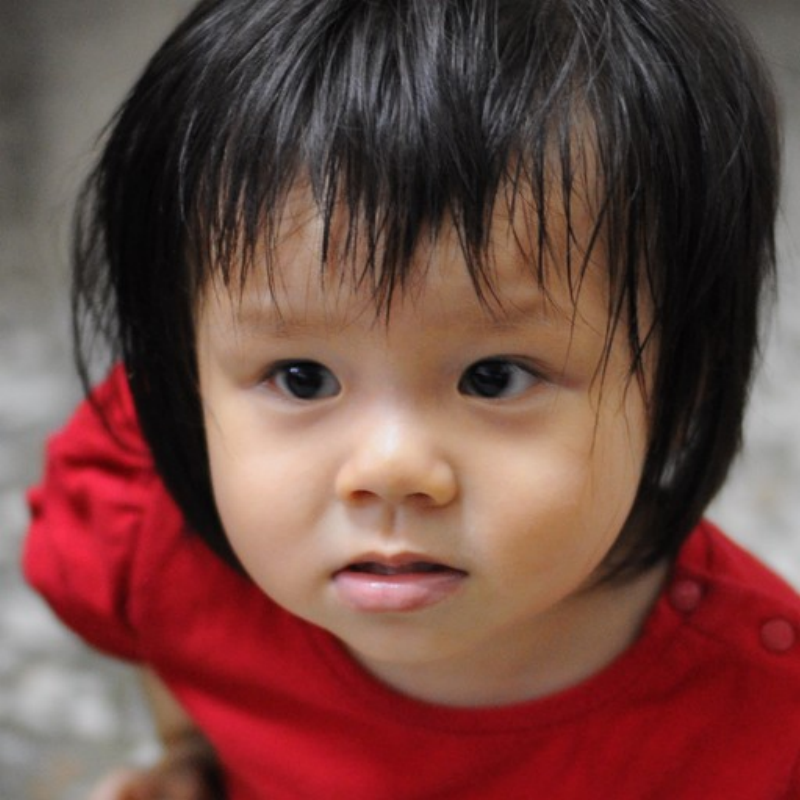}{-0.1,-0.4}{1.0,0.6} \\[-1mm]
\end{minipage}%
\hfill
\begin{minipage}{0.15\textwidth}
  \centering \tiny\textbf{Measurement} \\[-1mm]
    \zoomedImage[width=\linewidth]{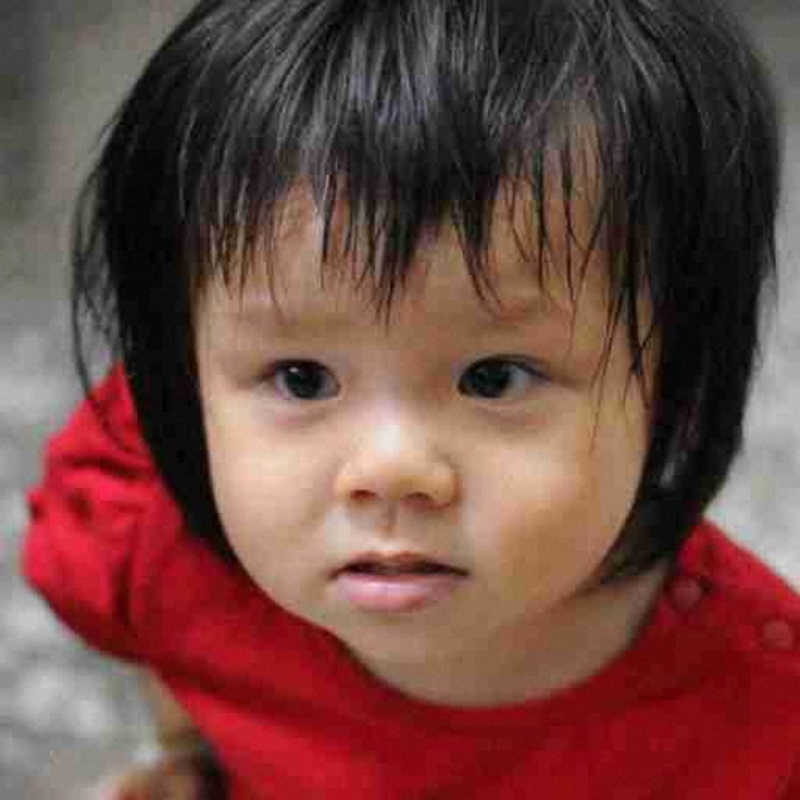}{-0.1,-0.4}{1.0,0.6}\\[-1mm]
\end{minipage}%
\hfill
\begin{minipage}{0.15\textwidth}
    \centering \tiny\textbf{CWGF} \\[-1mm]
    \zoomedImage[width=\linewidth]{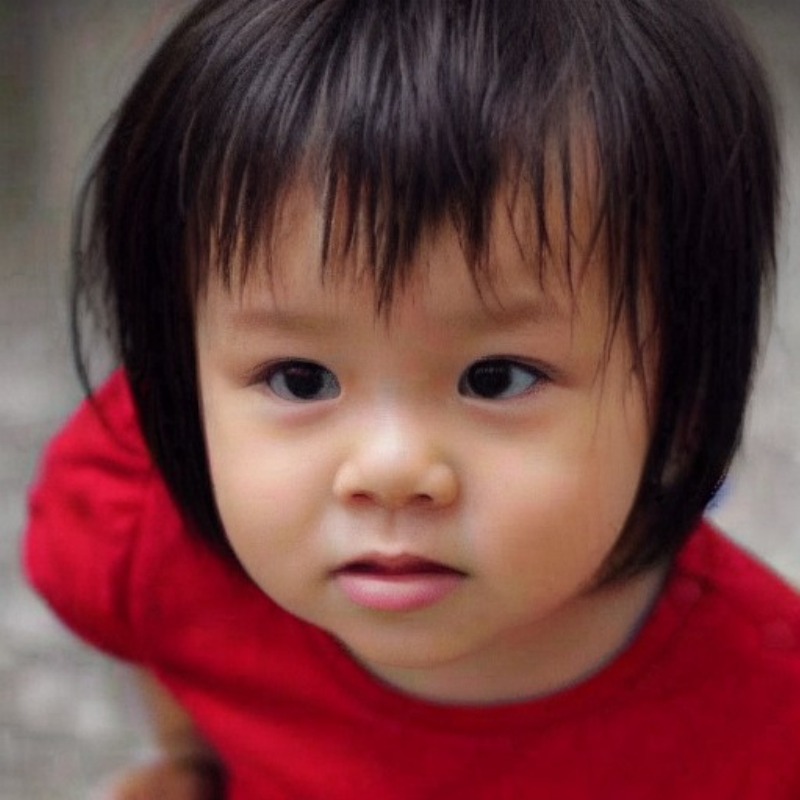}{-0.1,-0.4}{1.0,0.6}\\[-1mm]
\end{minipage}%
\hfill
\begin{minipage}{0.15\textwidth}
    \centering \tiny\textbf{LATINO} \\[-1mm]
    \zoomedImage[width=\linewidth]{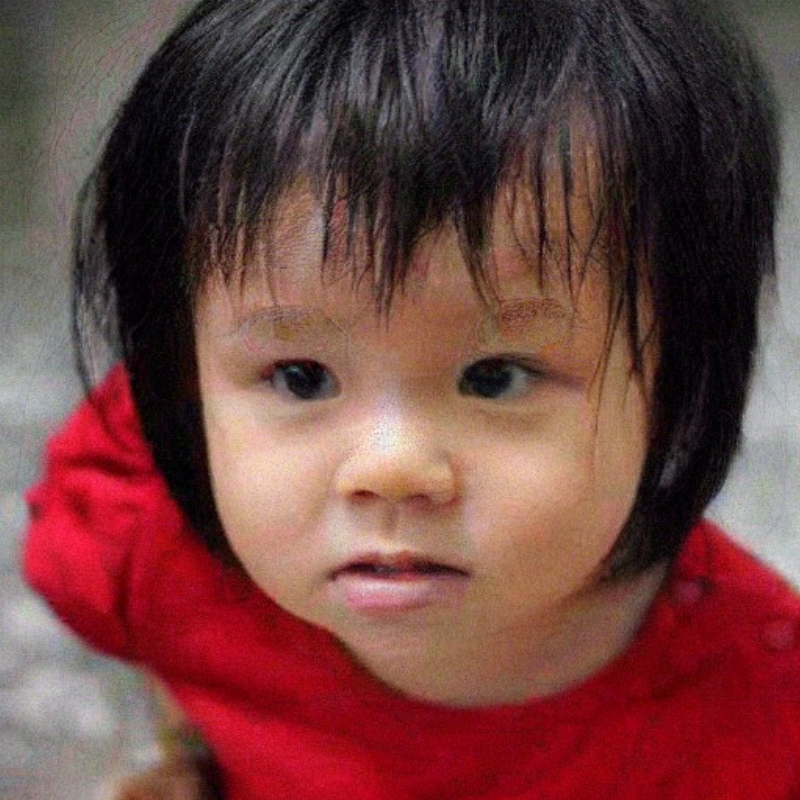}{-0.1,-0.4}{1.0,0.6} \\[-1mm]
\end{minipage}%
\hfill
\begin{minipage}{0.15\textwidth}
    \centering \tiny\textbf{LATINO-PRO} \\[-1mm]
    \zoomedImage[width=\linewidth]{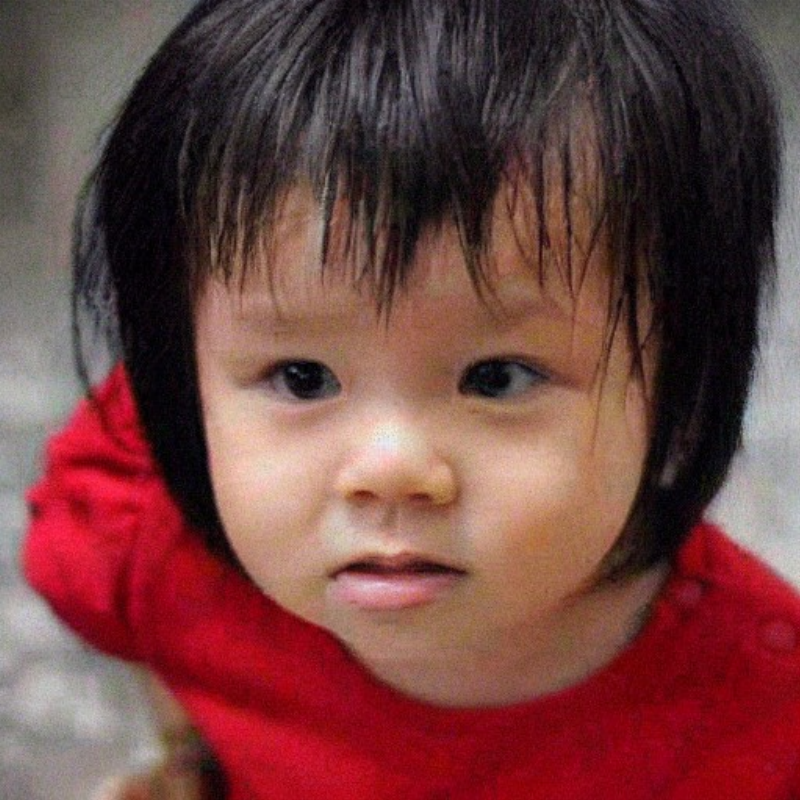}{-0.1,-0.4}{1.0,0.6} \\[-1mm]
\end{minipage}%
\hfill
\begin{minipage}{0.15\textwidth}
    \centering \tiny\textbf{TREG} \\[-1mm]
    \zoomedImage[width=\linewidth]{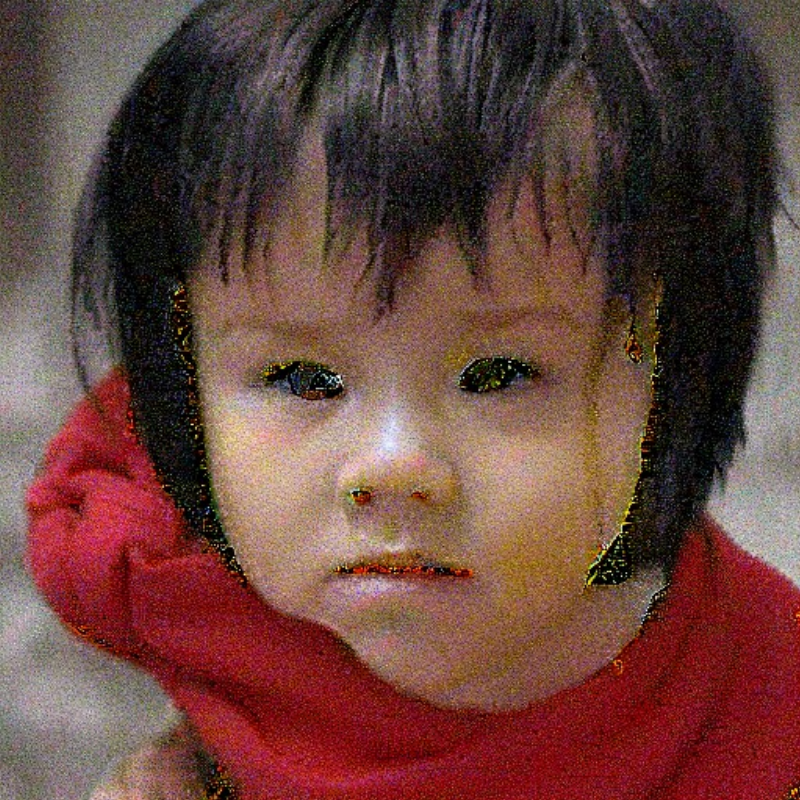}{-0.1,-0.4}{1.0,0.6} \\[-1mm]
\end{minipage}%
\vspace{-2mm}
\caption{Qualitative comparison of image restoration results for the non-linear problems considered. Samples taken from FFHQ-512. Prompt: \texttt{A photo of a face}.}
\label{fig:qualitative_FFHQ_nonlinear}
\end{figure*}

\section{Ablations}\label{sec:ablations}
\subsection{Ablation on Prompt Optimisation}
We perform ablations on the effect of prompt optimisation under the same hyperparameters as in Appendix~\ref{app:hyperparameters}. The results are shown in Tables~\ref{tab:prompt-ablation-matched-cfg-gauss}, \ref{tab:prompt-ablation-matched-cfg-motion}, and \ref{tab:prompt-ablation-matched-cfg-sr8} for Gaussian deblurring, motion deblurring, and SR$\times8$, respectively. We observe that prompt optimisation consistently improves the performance of the adversarial prompt, while it has less of an effect on the positive prompt, possibly due to the fact that the prompt is already well-aligned with the observations.
\begin{table}[htbp]
\centering
\small
\caption{
Prompt ablation study on Gaussian deblurring.
}
\label{tab:prompt-ablation-matched-cfg-gauss}
\begin{tabular}{lcccc}
\toprule
Setting & Prompt opt. & {FID $\downarrow$} & {LPIPS $\downarrow$} & {PSNR $\uparrow$} \\
\midrule
Positive prompt        & \xmark   & \textbf{21.28} & \textbf{0.306} & {29.37} \\
Positive prompt        & \cmark  & {21.95} & \textbf{{0.306}} & {29.23} \\
Adversarial prompt     & \xmark   & {32.91} & {0.354} & {28.64} \\
Adversarial prompt     & \cmark  & {27.58} & {0.322} & \textbf{29.59} \\
\bottomrule
\end{tabular}
\end{table}

\begin{table}[htbp]
\centering
\small
\caption{
Prompt ablation study on motion deblurring.
}\label{tab:prompt-ablation-matched-cfg-motion}
\begin{tabular}{lcccc}
\toprule
Setting & Prompt opt. & {FID $\downarrow$} & {LPIPS $\downarrow$} & {PSNR $\uparrow$} \\
\midrule
Positive prompt        & \xmark   & {23.12} & \textbf{0.338} & {27.83} \\
Positive prompt        & \cmark  & \textbf{23.03} & \textbf{0.338} & {27.82} \\
Adversarial prompt     & \xmark   & {32.20} & {0.357} & {27.86} \\
Adversarial prompt     & \cmark  & {25.99} & {0.339} & \textbf{{28.16}} \\
\bottomrule
\end{tabular}
\end{table}
\begin{table}[htbp]
\centering
\small
\caption{
Prompt ablation study on SR$\times8$.
}\label{tab:prompt-ablation-matched-cfg-sr8}
\begin{tabular}{lcccc}
\toprule
Setting & Prompt opt. & {FID $\downarrow$} & {LPIPS $\downarrow$} & {PSNR $\uparrow$} \\
\midrule
Positive prompt        & \xmark   & {35.67} & {0.373} & 26.97 \\
Positive prompt        & \cmark  & \textbf{35.52} & \textbf{{0.372}} & {26.95} \\
Adversarial prompt     & \xmark   & {92.43} & {0.456} & {26.24} \\
Adversarial prompt     & \cmark  & {50.07} & {0.415} & \textbf{27.39} \\
\bottomrule
\end{tabular}
\end{table}

\subsection{Ablation on Weights \texorpdfstring{$w_t$}{wt}}\label{sec:ablate-wt}
We provide a sensitivity analysis of CWGF to the step weighting $w(t)$ used in implementation. The weight $w(t)$ controls the strength of the prior-flow correction at each diffusion timestep. Specifically, we set $w(t)$ to be constant in $t$ and vary it in $[0.1,1.0]$ (\Cref{fig:sensitivity-wt}) to isolate the effect of its overall scale. The other hyperparameters are the same as in Appendix~\ref{app:hyperparameters}.
\begin{figure}[htbp]
    \centering
    \includegraphics[width=0.7\linewidth]{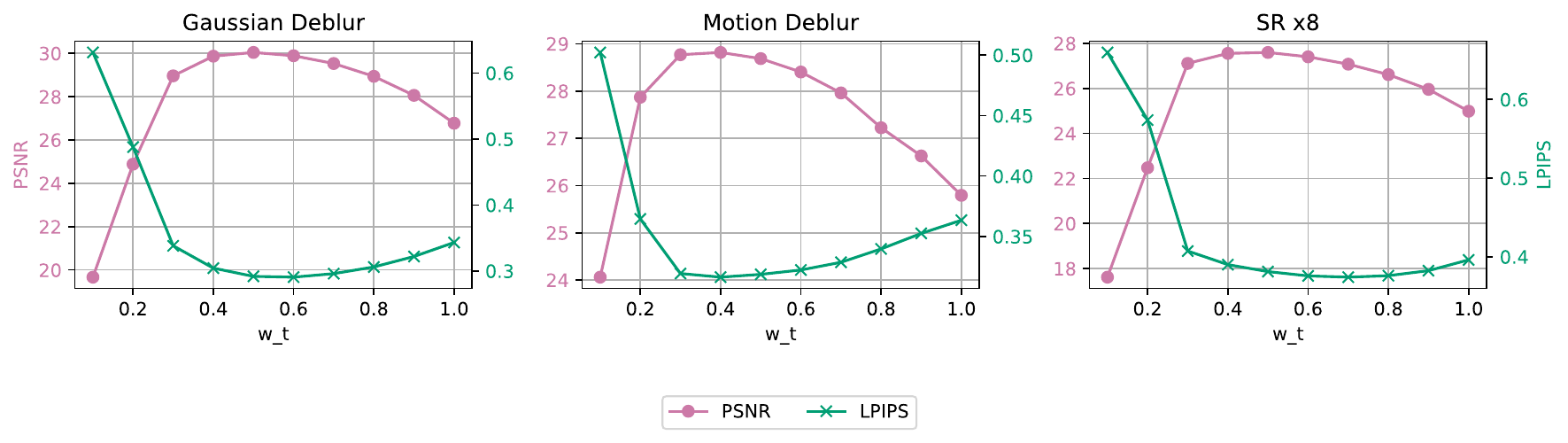}
    \caption{Sensitivity of CWGF to $w_t$ measured in PSNR and LPIPS averaged over 50 held-out images on FFHQ-512.}\label{fig:sensitivity-wt}
\end{figure}

\subsection{Ablation on Prompt Learning Rate \texorpdfstring{$\eta_c$}{eta_c}}\label{sec:ablate-prompt-lr}
Similarly we provide a sensitivity analysis of CWGF to the prompt learning rate in the positive and adversarial prompt settings. The other hyperparameters are the same as in Appendix~\ref{app:hyperparameters}. We vary $\eta_c$ in $[0.01,1.0]$ for the positive prompt (\Cref{fig:sensitivity-prompt}) and in $[0.01, 10.0]$ for the adversarial prompt (\Cref{fig:sensitivity-prompt-adv}). As shown in the figures, CWGF has a preference for smaller learning rates in the positive prompt setting.
\begin{figure}[htbp]
    \centering
    \includegraphics[width=0.7\linewidth]{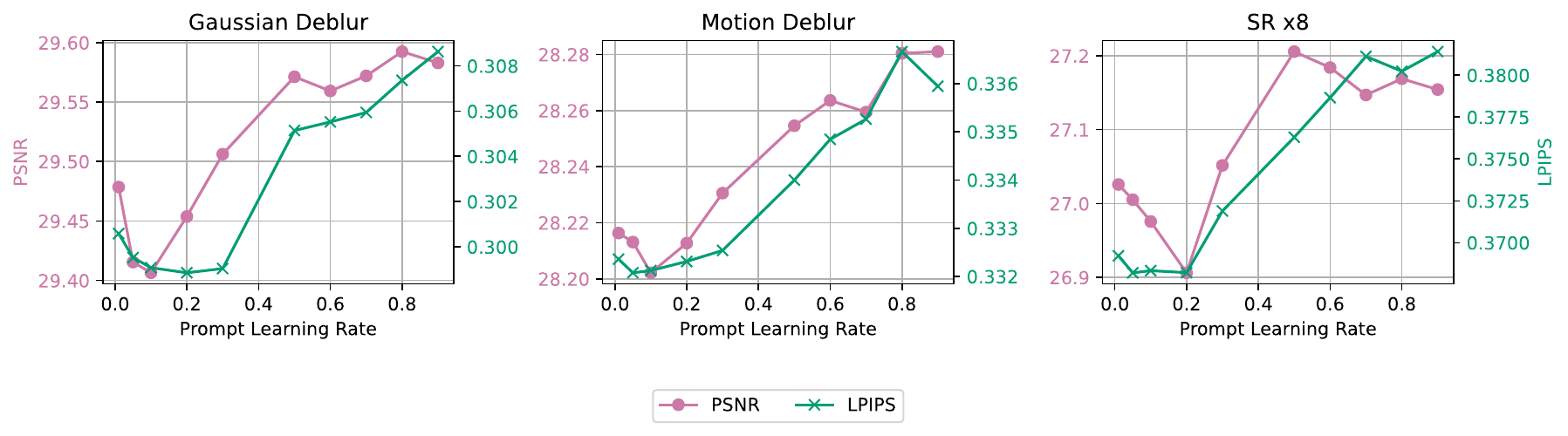}
    \caption{Sensitivity of CWGF to $\eta_c$ measured in PSNR and LPIPS averaged over 50 held-out images on FFHQ-512. Using positive prompt: \texttt{A photo of a face.}}\label{fig:sensitivity-prompt}
\end{figure}
\begin{figure}[htbp]
    \centering
    \includegraphics[width=0.7\linewidth]{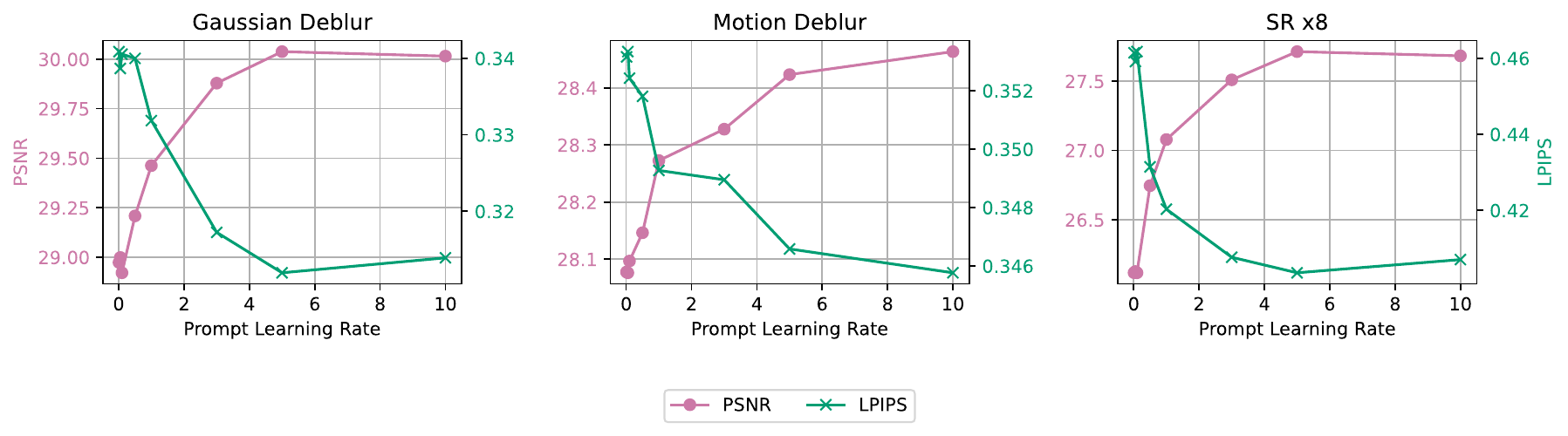}
    \caption{Sensitivity of CWGF to $\eta_c$ measured in PSNR and LPIPS averaged over 50 held-out images on FFHQ-512. Using adversarial prompt: \texttt{A photo of a cat.}}\label{fig:sensitivity-prompt-adv}
\end{figure}

\subsection{Ablation on Timestep Sampling}\label{sec:ablate-timestep}
We now provide an ablation study on different timestep sampling strategies used for our score approximation in \Cref{sec:prior-subflow} to isolate the effect of the timestep ordering.

We compare our cyclic timestep sampling with uniform sampling $t(k)\sim \mathrm{Unif}(\mathcal{T})$ for $\mathcal{T}=\{\texttt{999,879,759,639,499,379,259,139}\}$ and a decreasing schedule that repeats each element of $\mathcal{T}$ consecutively \textit{i.e.} \texttt{\{999,999,...,259,259,139,139\}}. We present the results for the positive prompt in~\Cref{tab:ablate-t-pos} and the adversarial prompt in~\Cref{tab:ablate-t-adv}. We note that uniform sampling performs worse than the other two strategies, aligning with the fact that the $t(k)$ used in approximating the score acts as an annealing parameter for coarse-to-fine refinement.

\begin{table*}[htbp]
\centering
\setlength{\tabcolsep}{2.5pt}
\renewcommand{\arraystretch}{1.08}
\caption{Ablation on the timestep sampling on FFHQ-512 using positive prompt: \texttt{A photo of a face}.}
\label{tab:ablate-t-pos}
\resizebox{0.8\textwidth}{!}{%
\begin{tabular}{l ccc ccc ccc}
\toprule
\cmidrule(lr){2-10}
 & \multicolumn{3}{c}{\textbf{Gaussian Deblur}} & \multicolumn{3}{c}{\textbf{Motion Deblur}} & \multicolumn{3}{c}{\textbf{SR $\times8$}} \\
\cmidrule(lr){2-4}\cmidrule(lr){5-7}\cmidrule(lr){8-10}
\textbf{Method}  & \textbf{FID$\downarrow$} & \textbf{PSNR$\uparrow$} & \textbf{LPIPS$\downarrow$}
& \textbf{FID$\downarrow$} & \textbf{PSNR$\uparrow$} & \textbf{LPIPS$\downarrow$}
& \textbf{FID$\downarrow$} & \textbf{PSNR$\uparrow$} & \textbf{LPIPS$\downarrow$} \\
\midrule
Uniform
& 24.77 & 29.47 & 0.323
& 27.33 & 27.98 & 0.348
& 48.44 & 27.29 & 0.403 \\
Decreasing
& 22.63  & 29.57 & 0.315  
& 22.93 & 28.13 & 0.338
& 34.50 & 27.00 & 0.367 \\
Cyclic
& {21.95} & {29.23} & {0.306}
& {23.03} & {27.82} & {0.338}
& {35.52} & {26.95} & {0.372} \\
\bottomrule
\end{tabular}%
}
\end{table*}
\begin{table*}[htbp]
\centering
\setlength{\tabcolsep}{2.5pt}
\renewcommand{\arraystretch}{1.08}
\caption{Ablation on the timestep sampling on FFHQ-512 using adversarial prompt: \texttt{A photo of a cat}.}\label{tab:ablate-t-adv}
\resizebox{0.8\textwidth}{!}{%
\begin{tabular}{l ccc ccc ccc}
\toprule
\cmidrule(lr){2-10}
 & \multicolumn{3}{c}{\textbf{Gaussian Deblur}} & \multicolumn{3}{c}{\textbf{Motion Deblur}} & \multicolumn{3}{c}{\textbf{SR $\times8$}} \\
\cmidrule(lr){2-4}\cmidrule(lr){5-7}\cmidrule(lr){8-10}
\textbf{Method}  & \textbf{FID$\downarrow$} & \textbf{PSNR$\uparrow$} & \textbf{LPIPS$\downarrow$}
& \textbf{FID$\downarrow$} & \textbf{PSNR$\uparrow$} & \textbf{LPIPS$\downarrow$}
& \textbf{FID$\downarrow$} & \textbf{PSNR$\uparrow$} & \textbf{LPIPS$\downarrow$} \\
\midrule
Uniform
& 26.09 & 29.49 & 0.338
& 29.91 & 27.93 & 0.361
& 55.82 & 27.31 & 0.432 \\
Decreasing
& 25.46  & 29.79 & 0.320
& 25.60 & 28.27 & 0.349
& 50.39 & 27.53 & 0.410 \\
Cyclic
& {27.58} & {29.59} & {0.322} 
& {25.99} & {28.16} & {0.339} 
& {50.07} & {27.39} & {0.415} \\
\bottomrule
\end{tabular}%
}
\end{table*}

\subsection{Ablation on Particle Count \texorpdfstring{$N$}{N}}\label{sec:ablate-particle-count}
We now provide an ablation study on the number of particles $N$ used in CWGF in the high-dimensional setting on the FFHQ-512 dataset. We vary $N$ in $\{1,4,8\}$ and present the results for the positive and adversarial prompts in~\Cref{tab:ablate-particle-count-pos} and~\Cref{tab:ablate-particle-count-adv}, respectively. We evaluate the performance of multi-particle versions of CWGF by randomly selecting one of the $N$ particles at the end of the algorithm to compute the metrics. While we note that the improvements are only marginal, this may be due to the high-dimensional latent space. We refer to Appendix~\ref{app:mnist} for a more detailed analysis on the effect of $N$ in a lower-dimensional setting.
\begin{table*}[htbp]
\centering
\setlength{\tabcolsep}{2.5pt}
\renewcommand{\arraystretch}{1.08}
\caption{Ablation on the number of particles $N$ on FFHQ-512 using positive prompt: \texttt{A photo of a face}.}
\label{tab:ablate-particle-count-pos}
\resizebox{0.8\textwidth}{!}{%
\begin{tabular}{l ccc ccc ccc}
\toprule
\cmidrule(lr){2-10}
 & \multicolumn{3}{c}{\textbf{Gaussian Deblur}} & \multicolumn{3}{c}{\textbf{Motion Deblur}} & \multicolumn{3}{c}{\textbf{SR $\times8$}} \\
\cmidrule(lr){2-4}\cmidrule(lr){5-7}\cmidrule(lr){8-10}
\textbf{Method}  & \textbf{FID$\downarrow$} & \textbf{PSNR$\uparrow$} & \textbf{LPIPS$\downarrow$}
& \textbf{FID$\downarrow$} & \textbf{PSNR$\uparrow$} & \textbf{LPIPS$\downarrow$}
& \textbf{FID$\downarrow$} & \textbf{PSNR$\uparrow$} & \textbf{LPIPS$\downarrow$} \\
\midrule
N=1
& {21.95} & {29.23} & {0.306}
& {23.03} & {27.82} & {0.338}
& {35.52} & {26.95} & {0.372} \\
N=4
& {21.53} & {29.31} & {0.307}
& {23.22} & {27.93} & {0.337}
& {35.28} & {26.94} & {0.371} \\
N=8
& 21.45 & 29.27 & 0.306
& 23.10 & 27.93 & 0.336
& 35.16 & 26.94 &  0.370 \\
\bottomrule
\end{tabular}%
}
\end{table*}

\begin{table*}[htbp]
\centering
\setlength{\tabcolsep}{2.5pt}
\renewcommand{\arraystretch}{1.08}
\caption{Ablation on the number of particles $N$ on FFHQ-512 using adversarial prompt: \texttt{A photo of a cat}.}\label{tab:ablate-particle-count-adv}
\resizebox{0.8\textwidth}{!}{%
\begin{tabular}{l ccc ccc ccc}
\toprule
\cmidrule(lr){2-10}
 & \multicolumn{3}{c}{\textbf{Gaussian Deblur}} & \multicolumn{3}{c}{\textbf{Motion Deblur}} & \multicolumn{3}{c}{\textbf{SR $\times8$}} \\
\cmidrule(lr){2-4}\cmidrule(lr){5-7}\cmidrule(lr){8-10}
\textbf{Method}  & \textbf{FID$\downarrow$} & \textbf{PSNR$\uparrow$} & \textbf{LPIPS$\downarrow$}
& \textbf{FID$\downarrow$} & \textbf{PSNR$\uparrow$} & \textbf{LPIPS$\downarrow$}
& \textbf{FID$\downarrow$} & \textbf{PSNR$\uparrow$} & \textbf{LPIPS$\downarrow$} \\
\midrule
N=1
& {27.58} & {29.59} & {0.322} 
& {25.99} & {28.16} & {0.339} 
& {50.07} & {27.39} & {0.415} \\
N=4
& 25.68 & 29.66 & 0.319
& 26.45 & 28.18 & 0.351
& 52.02 & 27.44 & 0.411 \\
N=8
& 26.35 & 29.66 & 0.317
& 26.52 & 28.19 & 0.351
& 52.55 & 27.42 & 0.411 \\
\bottomrule
\end{tabular}%
}
\end{table*}

\subsection{Ablation on Prompt Optimisation Score}\label{sec:ablate-prompt-score}
We now discuss whether the LCM score is a good proxy for the teacher score in prompt optimisation via~\eqref{eq:prompt_gradient}. We compare the performance of CWGF when using the LCM score and the teacher score in prompt optimisation; when using the teacher score, we scale the prompt gradient to ensure it has the same norm as the LCM's prompt gradient. We present the results for the positive and adversarial prompts in~\Cref{tab:ablate-prompt-score-pos} and~\Cref{tab:ablate-prompt-score-adv}, respectively. Additionally, we plot in~\Cref{fig:cosine-sim-score} the cosine similarity values between the LCM score and the teacher score across diffusion times during CWGF computed over 50 held-out images. We note that the LCM score aligns well with the teacher score, suggesting the scale mismatch between the LCM and teacher parametrisations, likely due to differences in classifier-free guidance parametrisation during LCM training~\citep{Luo2023LatentCM}; we leave a more detailed analysis of this for future work.
\begin{table*}[htbp]
\centering
\setlength{\tabcolsep}{2.5pt}
\renewcommand{\arraystretch}{1.08}
\caption{Ablation on the use of teacher score in prompt optimisation on FFHQ-512 using positive prompt: \texttt{A photo of a face}.}\label{tab:ablate-prompt-score-pos}
\resizebox{0.8\textwidth}{!}{%
\begin{tabular}{l ccc ccc ccc}
\toprule
\cmidrule(lr){2-10}
 & \multicolumn{3}{c}{\textbf{Gaussian Deblur}} & \multicolumn{3}{c}{\textbf{Motion Deblur}} & \multicolumn{3}{c}{\textbf{SR $\times8$}} \\
\cmidrule(lr){2-4}\cmidrule(lr){5-7}\cmidrule(lr){8-10}
\textbf{Method}  & \textbf{FID$\downarrow$} & \textbf{PSNR$\uparrow$} & \textbf{LPIPS$\downarrow$}
& \textbf{FID$\downarrow$} & \textbf{PSNR$\uparrow$} & \textbf{LPIPS$\downarrow$}
& \textbf{FID$\downarrow$} & \textbf{PSNR$\uparrow$} & \textbf{LPIPS$\downarrow$} \\
\midrule
Teacher score (SD1.5)
& 21.97 & 29.43 & 0.319
& 23.36 & 28.00 & 0.340
& 36.19 & 27.04 & 0.371 \\
LCM score
& {21.95} & {29.23} & {0.306}
& {23.03} & {27.82} & {0.338}
& {35.52} & {26.95} & {0.372} \\
\bottomrule
\end{tabular}%
}
\end{table*}

\begin{table*}[htbp]
\centering
\setlength{\tabcolsep}{2.5pt}
\renewcommand{\arraystretch}{1.08}
\caption{Ablation on the use of teacher score in prompt optimisation on FFHQ-512 using adversarial prompt: \texttt{A photo of a cat}.}\label{tab:ablate-prompt-score-adv}
\resizebox{0.8\textwidth}{!}{%
\begin{tabular}{l ccc ccc ccc}
\toprule
\cmidrule(lr){2-10}
 & \multicolumn{3}{c}{\textbf{Gaussian Deblur}} & \multicolumn{3}{c}{\textbf{Motion Deblur}} & \multicolumn{3}{c}{\textbf{SR $\times8$}} \\
\cmidrule(lr){2-4}\cmidrule(lr){5-7}\cmidrule(lr){8-10}
\textbf{Method}  & \textbf{FID$\downarrow$} & \textbf{PSNR$\uparrow$} & \textbf{LPIPS$\downarrow$}
& \textbf{FID$\downarrow$} & \textbf{PSNR$\uparrow$} & \textbf{LPIPS$\downarrow$}
& \textbf{FID$\downarrow$} & \textbf{PSNR$\uparrow$} & \textbf{LPIPS$\downarrow$} \\
\midrule
Teacher score (SD1.5)
& 23.00  & 29.36 & 0.328
& 24.87 & 27.85 & 0.350
& 48.28 & 26.94 & 0.412 \\
LCM score 
& {27.58} & {29.59} & {0.322} 
& {25.99} & {28.16} & {0.339} 
& {50.07} & {27.39} & {0.415} \\
\bottomrule
\end{tabular}%
}
\end{table*}

\begin{figure}[htbp]
    \centering
    \includegraphics[width=0.7\linewidth]{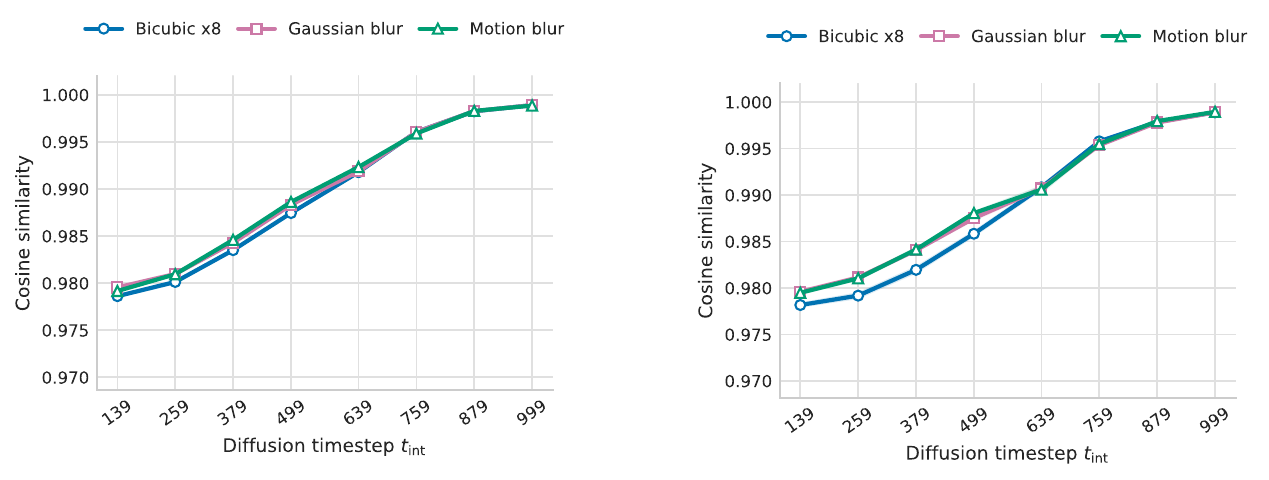}
    \caption{Cosine similarity between the teacher and LCM score estimate. Left: positive prompt. Right: adversarial prompt.}
    \label{fig:cosine-sim-score}
\end{figure}

\subsection{Ablation on Encoder Approximation}\label{sec:ablate-encoder-approx-image}
In this section, we provide an ablation study on the encoder approximation used in~\Cref{sec:method-likelihood}. For the pretrained image VAE, we first report the average cosine similarity between encoder likelihood and decoder likelihood gradients across 50 images in~\Cref{tab:ablate-encoder-cosine}, which shows positive but modest alignment. However, we remark that the encoder update is designed to avoid decoder backpropagation and remain stable with few updates. Using the same 16-step budget under adversarial prompts and held-out hyperparameter tuning, we obtain the results in~\Cref{tab:ablate-encoder-vs-decoder}, confirming the efficacy of the encoder approximation. Finally, we report in~\Cref{tab:ablate-decoder-jacobian-sv} the estimated largest singular value of the decoder Jacobian; these large values may suggest the instability under explicit Euler integration, consistent with the artefacts reported for decoder likelihood updates in~\citet{raphaeliSILOSolvingInverse2025}, as we discussed in~\Cref{sec:method-likelihood}.

\begin{table}[htbp]
\centering
\small
\caption{
Cosine similarity between encoder and decoder likelihood gradients using the pretrained VAE~\citep{Luo2023LatentCM,rombachhighresolutionimagesynthesis2022}; mean and standard deviation are computed over 50 images across all steps.
}\label{tab:ablate-encoder-cosine}
\begin{tabular}{lccc}
\toprule
 & Gaussian Deblur & Motion Deblur & SR$\times8$ \\
\midrule
Cosine similarity & $0.28_{\pm 0.09}$ & $0.23_{\pm 0.09}$ & $0.36_{\pm 0.08}$ \\
\bottomrule
\end{tabular}
\end{table}

\begin{table*}[htbp]
\centering
\setlength{\tabcolsep}{2.5pt}
\renewcommand{\arraystretch}{1.08}
\caption{Encoder vs.\ decoder likelihood update on FFHQ-512 using adversarial prompt: \texttt{A photo of a cat}.}\label{tab:ablate-encoder-vs-decoder}
\resizebox{0.8\textwidth}{!}{%
\begin{tabular}{l ccc ccc ccc}
\toprule
\cmidrule(lr){2-10}
 & \multicolumn{3}{c}{\textbf{Gaussian Deblur}} & \multicolumn{3}{c}{\textbf{Motion Deblur}} & \multicolumn{3}{c}{\textbf{SR $\times8$}} \\
\cmidrule(lr){2-4}\cmidrule(lr){5-7}\cmidrule(lr){8-10}
\textbf{Method}  & \textbf{FID$\downarrow$} & \textbf{PSNR$\uparrow$} & \textbf{LPIPS$\downarrow$}
& \textbf{FID$\downarrow$} & \textbf{PSNR$\uparrow$} & \textbf{LPIPS$\downarrow$}
& \textbf{FID$\downarrow$} & \textbf{PSNR$\uparrow$} & \textbf{LPIPS$\downarrow$} \\
\midrule
Encoder (CWGF)
& 27.58 & 29.59 & 0.322
& 25.99 & 28.16 & 0.339
& 50.07 & 27.39 & 0.415 \\
Decoder
& 55.59 & 20.65 & 0.519
& 47.57 & 23.43 & 0.524
& 57.79 & 20.02 & 0.534 \\
\bottomrule
\end{tabular}%
}
\end{table*}

\begin{table}[htbp]
\centering
\small
\caption{Estimated largest singular value of the decoder Jacobian averaged over 50 images across all steps.
}\label{tab:ablate-decoder-jacobian-sv}
\begin{tabular}{lccc}
\toprule
 & Gaussian Deblur & Motion Deblur & SR$\times8$ \\
\midrule
$\sigma_1$ & 38.90 & 31.26 & 74.79 \\
\bottomrule
\end{tabular}
\end{table}

\input{sections/appendix-toy}

\newpage
\section{Computational Resources}\label{sec:compute-resources}

\begin{wraptable}{r}{0.52\textwidth}
\vspace{-1.0em}
\centering
\small
\caption{Runtime and memory comparisons.}
\label{tab:runtime}
\resizebox{\linewidth}{!}{%
\begin{tabular}{lccc}
\toprule
\textbf{Method} & \textbf{VRAM (GB)} & \textbf{Time (s)} & \textbf{Resolution} \\
\midrule
\textbf{CWGF} & 8.66 & 8.72 & $512^2$ \\
\midrule
\textbf{LATINO} & 4.16 & 2.89 & $512^2$ \\
\textbf{LATINO-PRO} & 9.27 & 32.3 & $512^2$ \\
\textbf{TReg} & $\sim$6.40 & 40.5 & $512^2$ \\
\textbf{P2L} & $\sim$10.6 & 600 & $512^2$ \\
\textbf{LDPS} & 9.51 & 176 & $512^2$ \\
\textbf{PSLD} & 10.3 & 185 & $512^2$ \\
\bottomrule
\end{tabular}
}
\vspace{-1.0em}
\end{wraptable}

For all experiments, we use NVIDIA GPUs on internal HPCs, which include a mixture of A40, L40S, RTX A6000, and GH200s. However, we also point out that our method can run on consumer GPUs. Table~\ref{tab:runtime} shows the memory footprint and runtime of our method, compared to those of the baselines, computed on the same machine.

\newpage
\section{Additional samples}\label{sec:additional_samples}

\begin{figure*}[!h]
    \centering
    \begin{center}
        \begin{tabular}{c c c}
            \centering
             \sidecap{Measurement}&\includegraphics[width=0.48\textwidth]{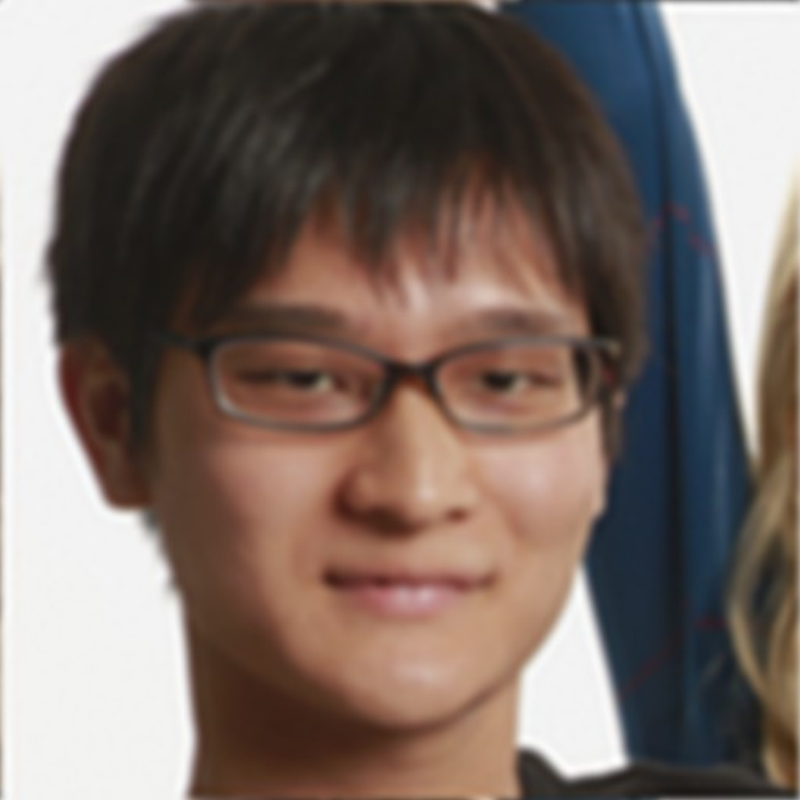} & 
            \includegraphics[width=0.48\textwidth]{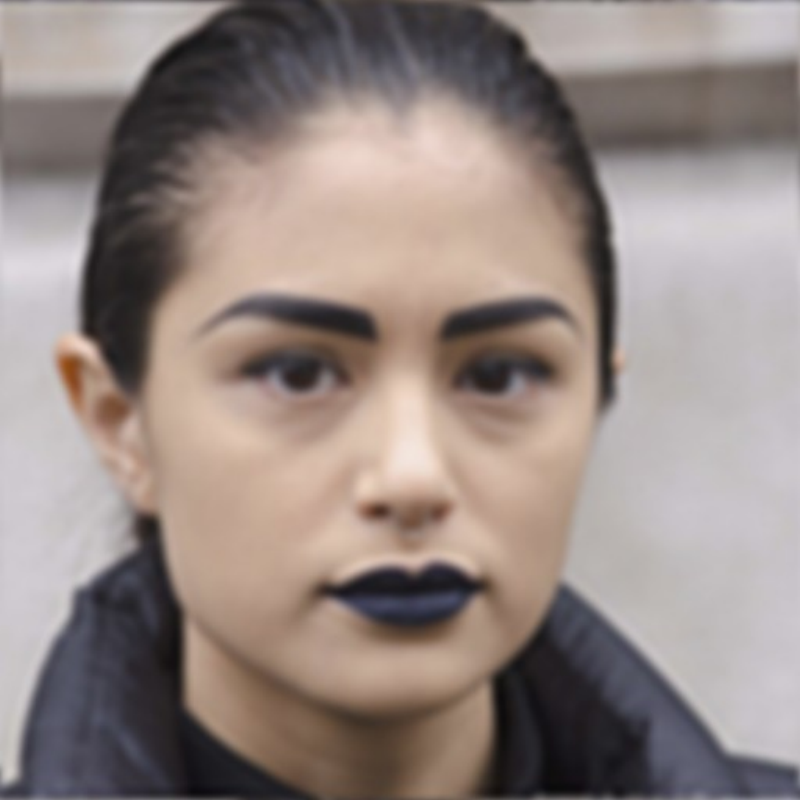} \\

            \sidecap{Ground truth}&\includegraphics[width=0.48\textwidth]{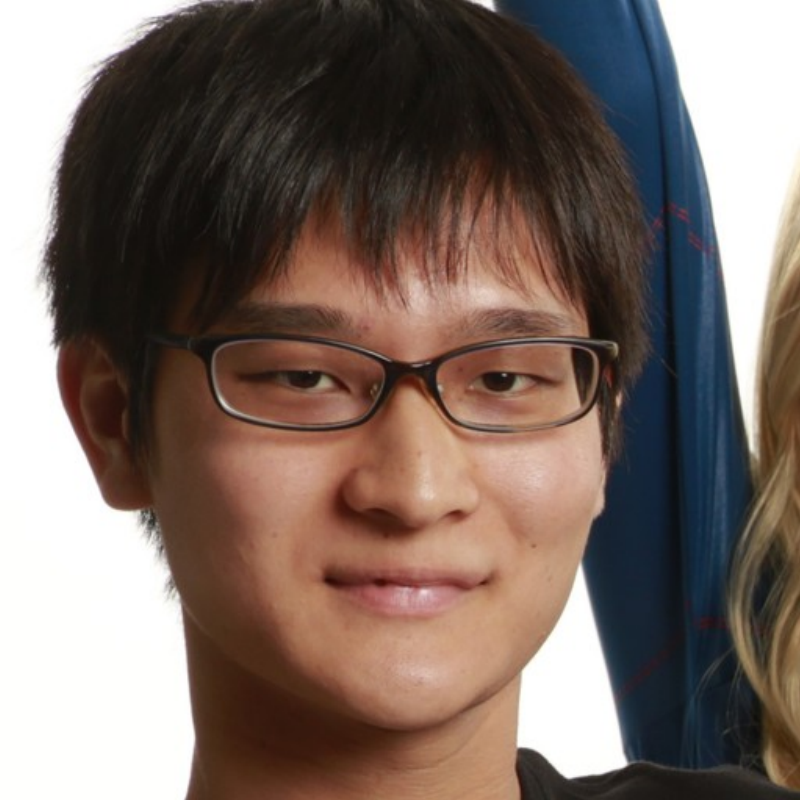} & 
            \includegraphics[width=0.48\textwidth]{Experiments/FFHQ/Gaussian-Blur/60007/GT.pdf} \\

            \sidecap{Restored}&\includegraphics[width=0.48\textwidth]{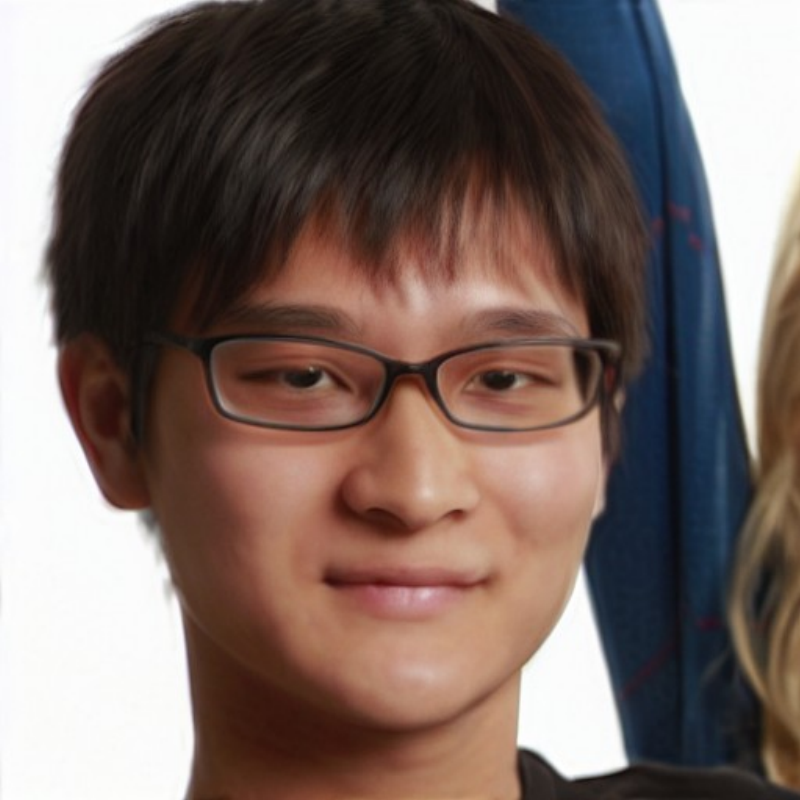} & 
            \includegraphics[width=0.48\textwidth]{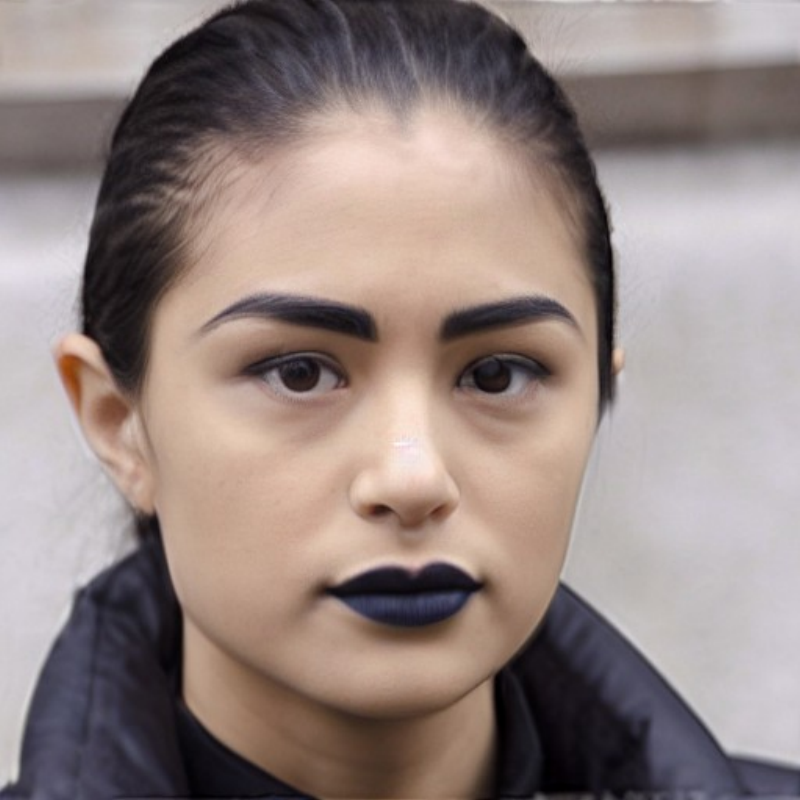} \\
        \end{tabular}
    \end{center}

    \caption{Gaussian deblur FFHQ-512 CWGF.}
    \label{fig:FFHQ_SR_sx16}
\end{figure*}

\begin{figure*}[!h]
    \centering
    \begin{center}
        \begin{tabular}{c c c}
            \centering
             \sidecap{Measurement}&\includegraphics[width=0.48\textwidth]{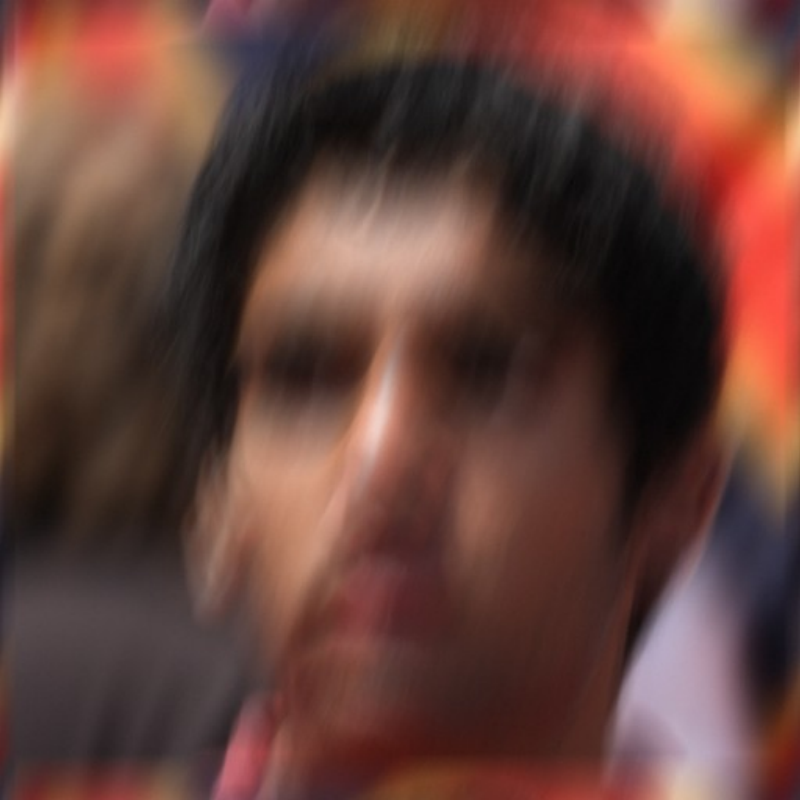} & 
            \includegraphics[width=0.48\textwidth]{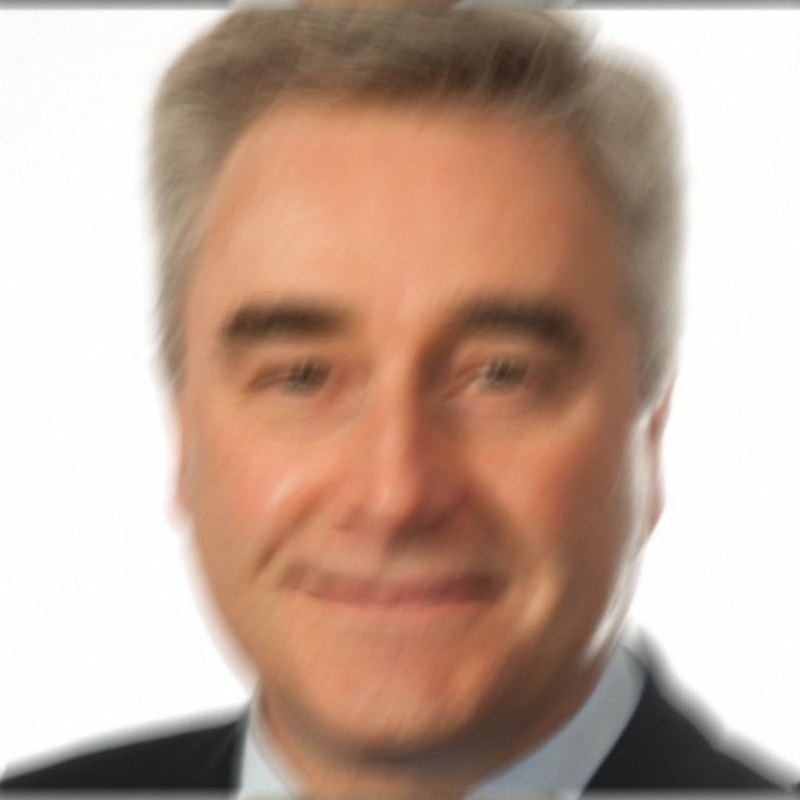} \\

            \sidecap{Ground truth}&\includegraphics[width=0.48\textwidth]{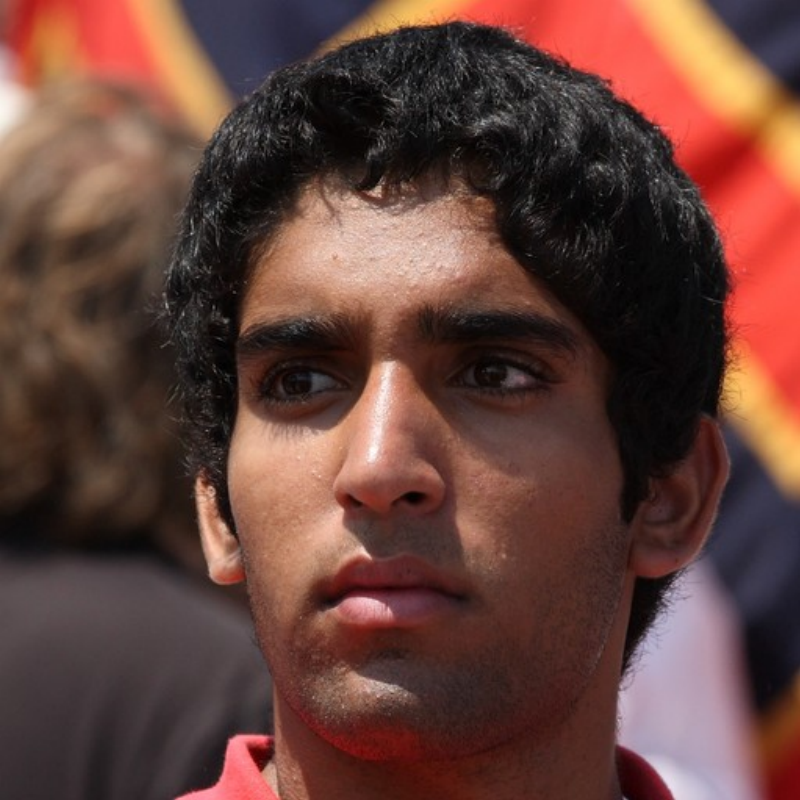} & 
            \includegraphics[width=0.48\textwidth]{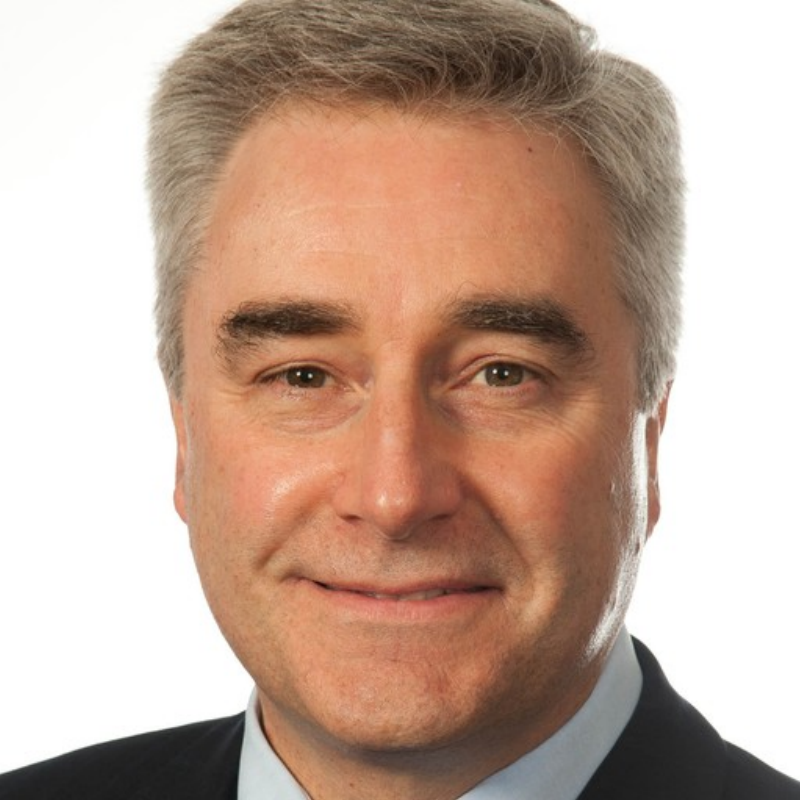} \\

            \sidecap{Restored}&\includegraphics[width=0.48\textwidth]{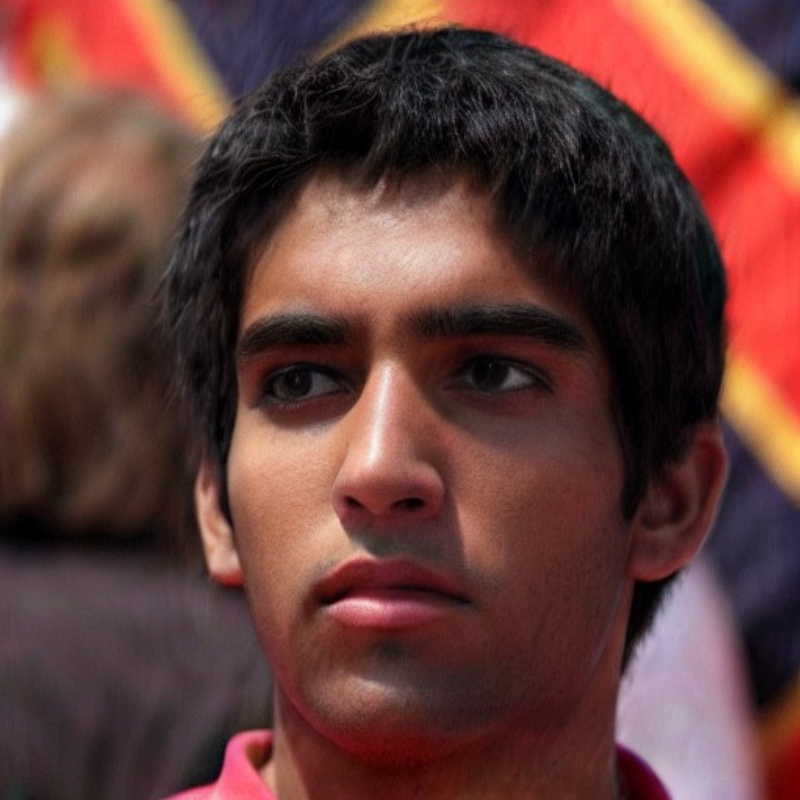} & 
            \includegraphics[width=0.48\textwidth]{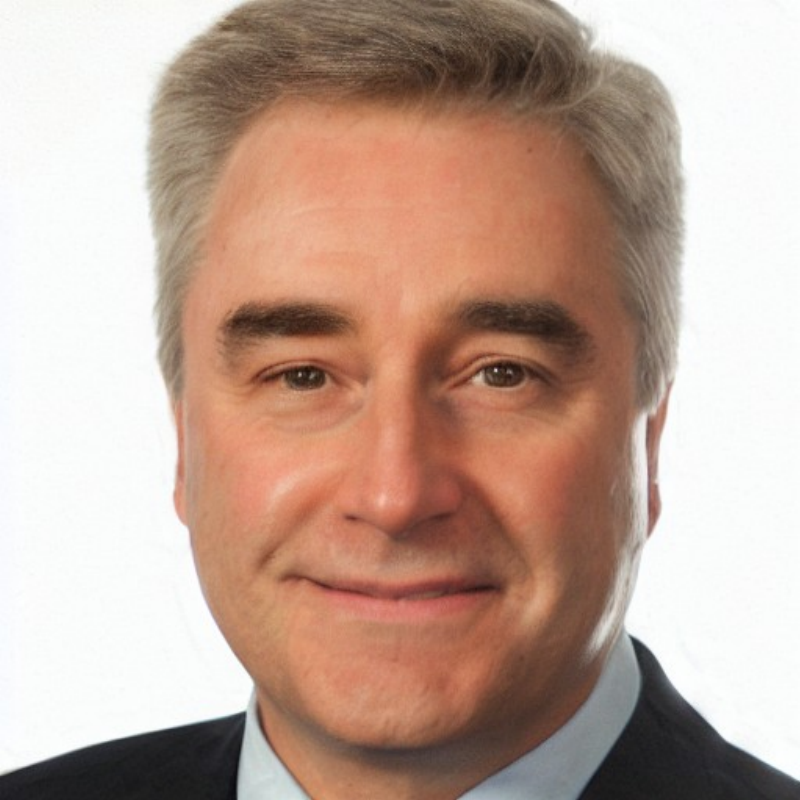} \\
        \end{tabular}
    \end{center}

    \caption{Motion deblur FFHQ-512 CWGF.}
    \label{fig:FFHQ_Deblur}
\end{figure*}

\begin{figure*}[!h]
    \centering
    \begin{center}
        \begin{tabular}{c c c}
            \centering
             \sidecap{Measurement}&\includegraphics[width=0.48\textwidth]{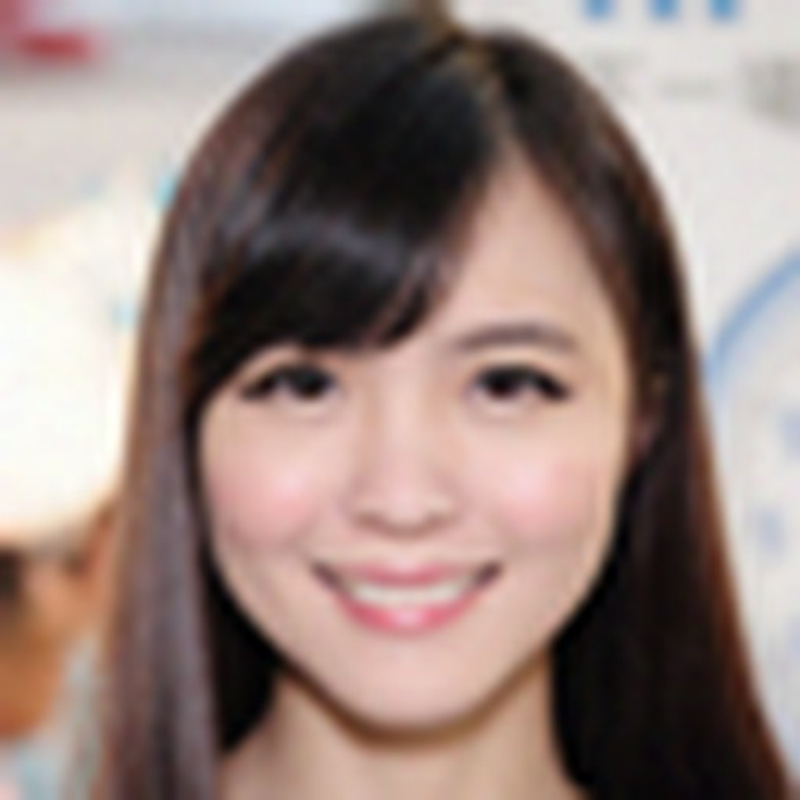} & 
            \includegraphics[width=0.48\textwidth]{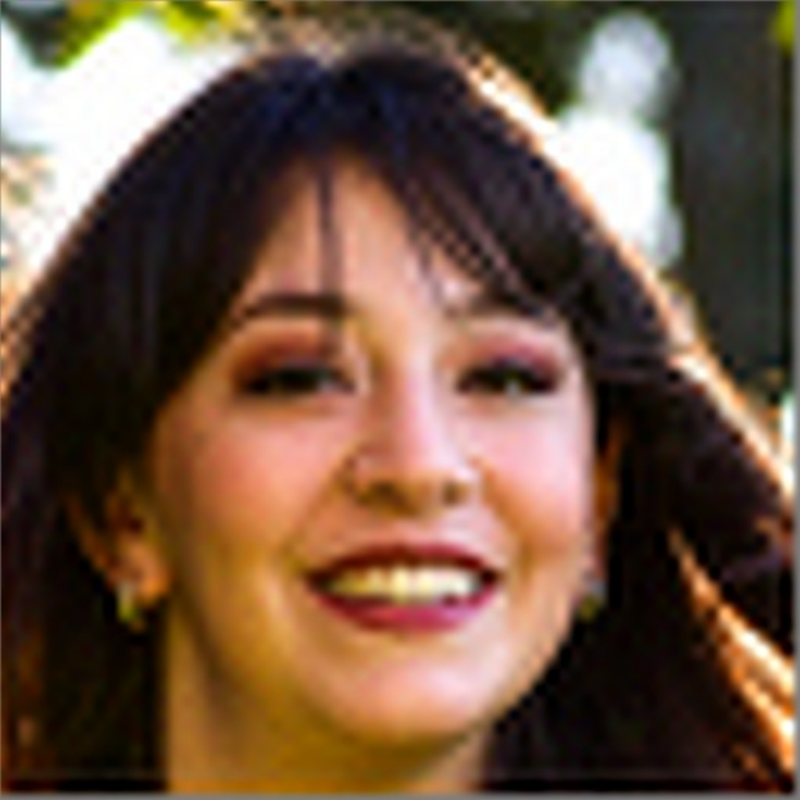} \\

            \sidecap{Ground truth}&\includegraphics[width=0.48\textwidth]{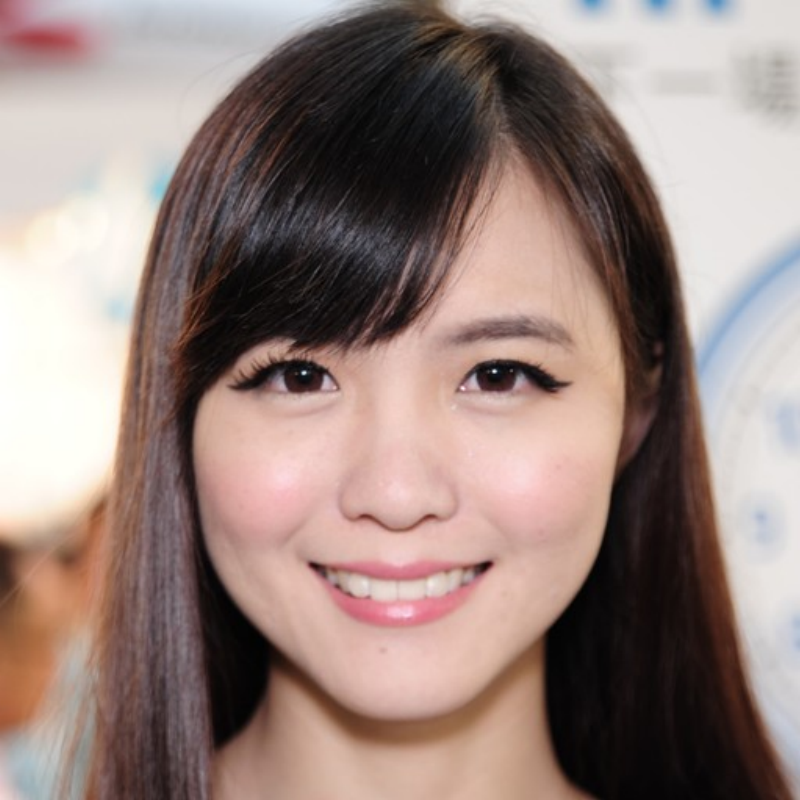} & 
            \includegraphics[width=0.48\textwidth]{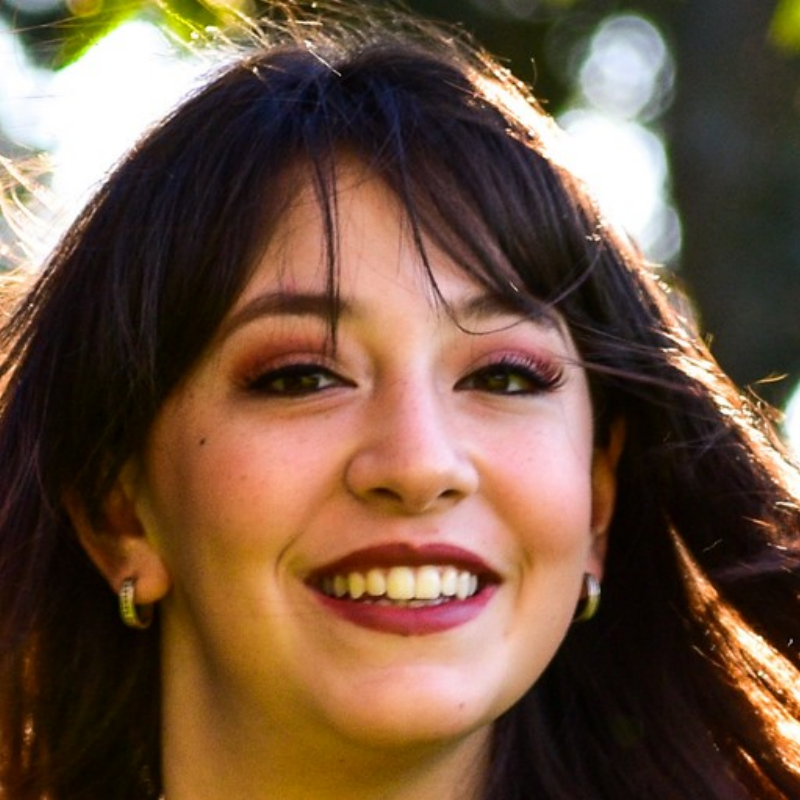} \\

            \sidecap{Restored}&\includegraphics[width=0.48\textwidth]{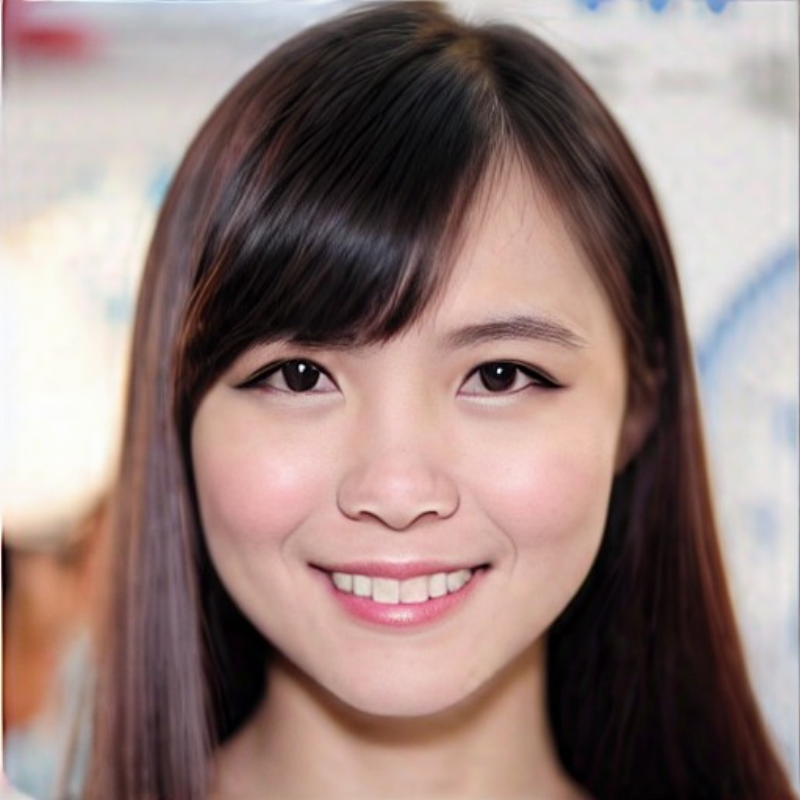} & 
            \includegraphics[width=0.48\textwidth]{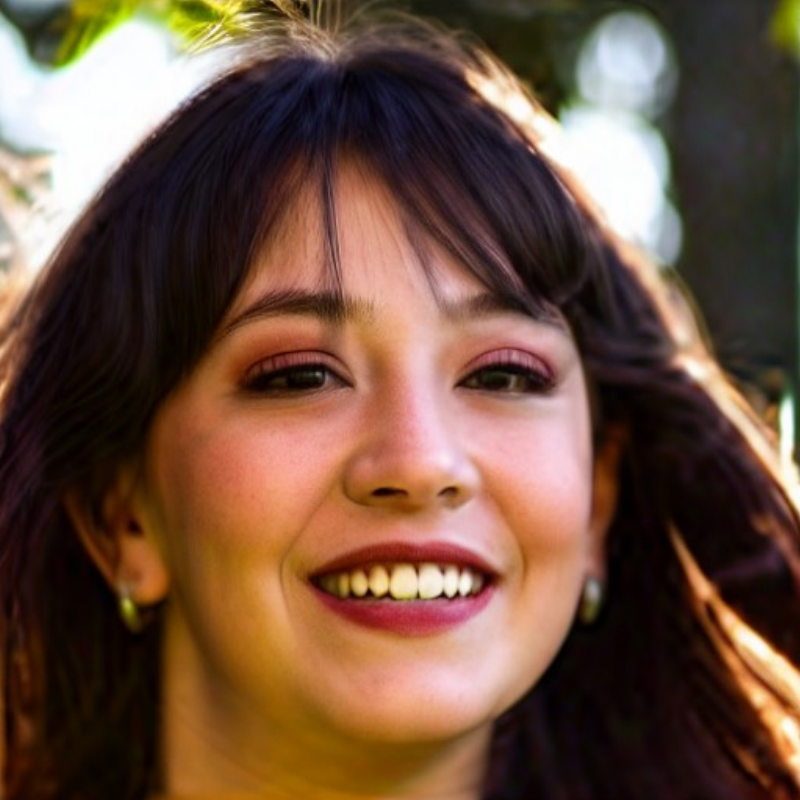} \\
        \end{tabular}
    \end{center}

    \caption{SR$\times8$ FFHQ-512 CWGF.}
    \label{fig:FFHQ_motion}
\end{figure*}

%% file: sections/appendix-relatedworks.tex
\subsection{Discussion on Related Works}
\label{sec:related-works-appendix}
\paragraph{Variational inference for inverse problems.} 
Several works cast posterior sampling with diffusion priors as variational inference, optimising a tractable family to approximate $p(\vx\mid\vy)$. A common choice is a Gaussian approximate posterior in pixel space~\citep{mardanivariationalperspectivesolving2023}, later adapted to latent diffusion priors~\citep{zilbersteinrepulsivelatentscore2024}. Other approaches parametrise the variational family through denoising distributions or intermediate-time marginals along the diffusion trajectory~\citep{guthaModeseekingInverseProblems2025,janati2024divide,moufad2025variational,zhengFastRobustLikelihoodGuided2026}. Our method is most closely related to~\citet{wang2025gradientflowapproachsolving}, who employ a diffusion model for the prior flow; however, their approach relies on a high number of NFEs and requires backpropagation through the decoder.

\paragraph{LDMs for inverse problems.}
Applying LDMs to inverse problems is complicated by the mismatch between latent generation and pixel-space measurements. Many methods enforce data consistency by differentiating through the decoder~\citep{routsolvinglinearinverse2023,songsolvinginverseproblems2024,achituveInverseProblemSampling2025,wang2025gradientflowapproachsolving}, while gradient-free alternatives perform pixel-space correction steps (\textit{e.g.}\ minimising $\|\vy-A(\vx)\|$) followed by re-encoding~\citep{kim2025regularizationtextslatentdiffusion,Spagnoletti_2025_ICCV,shtanchaev2026gradientfree}. A complementary direction is to learn a surrogate forward operator directly in latent space~\citep{wangEMControlAddingConditional2025,raphaeliSILOSolvingInverse2025}. Our method provides a principled way to incorporate observation information into reconstruction using the encoder, circumventing the need for backpropagation through LDM components.

\paragraph{Prompt optimisation in (L)DMs.}
Prompt or embedding optimisation has been used both to personalise concepts~\citep{gal2022imageworthwordpersonalizing,mokady2022nulltextinversioneditingreal,kim2025directional} and to improve test-time generation via text/null-embedding calibration~\citep{wenHardPromptsMade2023,mahajanPromptingHardHardly2024,wangDiscretePromptOptimization2024,kim2025testtimealignmenttexttoimagediffusion}. For inverse problems, P2L~\citep{Chung2023PrompttuningLD} and TReg~\citep{kim2025regularizationtextslatentdiffusion} optimise (positive/negative) text embeddings to better match the observation, while LATINO-PRO~\citep{Spagnoletti_2025_ICCV} adopts an MMLE perspective and performs online prompt optimisation via SOUL~\citep{debortoli2020efficientstochasticoptimisationunadjusted}. Our Euclidean-Wasserstein gradient flow shares a similar MMLE interpretation~\citep{kuntz23a}, but achieves this without requiring nested Langevin steps for sampling.

%% file: sections/appendix-toy.tex
\section{Additional Synthetic Experiments}\label{app_sec:synth}
We empirically validate our method in a tractable synthetic experiment using a 25-component Gaussian Mixture Model (GMM) with a setup similar to~\citet{cardoso2024monte}.

\paragraph{Prior model.}~Let $z\in\mathbb R^2$, and index the GMM components by $(r,s)\in\{-2,-1,0,1,2\}^2$. We fix all component probabilities to $1/25$ and we set $\mu_j=(2r,2s)$. To obtain an analogue of the prompt, we introduce an additional parameter $c$ as the global translation of the means:
$$
\begin{aligned}
m_j(c)=\mu_j+\frac{1}{2}cv,\qquad
p_c(z)=\frac1{25}\sum_{j=1}^{25}
  \mathcal N\left(z;m_j(c),\sigma_z^2I_2\right),
\end{aligned}
$$ 
where $v=(1/{\sqrt2})(1,1)^\top$ and we set the standard deviation to $\sigma_z=0.2$. For modelling the VAE used in the image experiments, we consider an orthogonal linear encoder $E:\R^4\to \R^2$ and a decoder $D:\R^2\to \R^4$:
\begin{align*}
    D=\frac{1}{\sqrt3}\begin{bmatrix}1&0&1&1\\0&1&1&-1\end{bmatrix}^\top,
\end{align*}
with $D^\top D=I_2$ and set $p(x\mid z)=\mathcal N\left(x;Dz,\sigma_{\rm dec}^2I_4\right)$ with $\sigma_{\rm dec}=1.0$. We construct the approximate posterior, such that $q(z|x)=p(z|x)\propto p(x|z)\mathcal{N}(z;0,I)$:
\begin{align*}
    q_E(z\mid x)=\mathcal N\left(z;\rho D^\top x,\rho\sigma_{\rm dec}^2I_2
\right),\qquad \rho=(1+\sigma_{\rm dec}^2)^{-1}.
\end{align*}
\paragraph{Forward Model.}~We consider a linear forward model $y=Ax+\sigma_y \epsilon$ with $A\in \R^{4\times 2}$ being a random row orthonormal matrix, $\epsilon\sim \mathcal{N}(0,I)$, and $\sigma_y=0.1$. Under those choices, we can compute $p(y|z)$ as a Gaussian:
\begin{equation*}
    p(y|z) = \mathcal{N}(y;Bz, R), \quad B=AD, \quad R=\sigma_y^2I_2+\sigma_{\rm dec}^2AA^\top.
\end{equation*}
Therefore, it can be shown that the exact posterior is given by
\begin{equation*}
    p_c(z\mid y)=\sum_{j=1}^{25}\omega_j(c,y)\mathcal N\left(z;\bar m_j(c,y),\Sigma_{\rm post}\right),
\end{equation*}
where we have for $K=\sigma_z^2B^\top(\sigma_z^2BB^\top +R)^{-1}$: means $\bar m_j(c,y)=m_j(c)+K\left(y-Bm_j(c)\right)$ and covariances $\Sigma_{\rm post}=\sigma_z^2(I-KB)$. The weights are given by $\omega_j(c,y)=L_j(c;y)/p_c(y)$ for $L_j(c,y):=\mathcal{N}(y;Bm_j(c), \sigma_z^2BB^\top +R)$ and the evidence
\begin{equation*}
    p_c(y)=\frac1{25}\sum_{j=1}^{25}L_j(c;y).
\end{equation*}
\paragraph{Experiments.}~We sample $100$ observations using the same prompt $c=0$ (no translation) and aggregate the results. Both the log evidence $\log p_c(y)$ and posterior $p_c(z|y)$ can be computed exactly. For each observation, the optimal prompt $c_\star:=\argmax_c \log p_c(y)$ can be found using numerical methods by optimising over a grid in 1D. Below we validate CWGF's performance on this tractable example with different numbers of particles $N\in\{1,5,25,100\}$ in~\Cref{fig:app-combined-scaling}. Additionally, we provide an ablation study for each component of CWGF at fixed particle count $N=100$ in~\Cref{tab:app-cwgf_ablation}: in the row \textit{PFODE-integration}, we replace the CM approximation of $G_t^{p_c}$ in~\eqref{eq:flow-map-discrepancy} with numerical integration of the PF-ODE, in the row \textit{W2 flow}, we replace the CM approximation using the blob method~\citep{carrillo2019blobmethoddiffusion} for simulating the exact W2 flow in~\eqref{eq:R-particle-drift}.
\begin{figure}[htbp]
    \centering
    \includegraphics[width=0.7\linewidth]{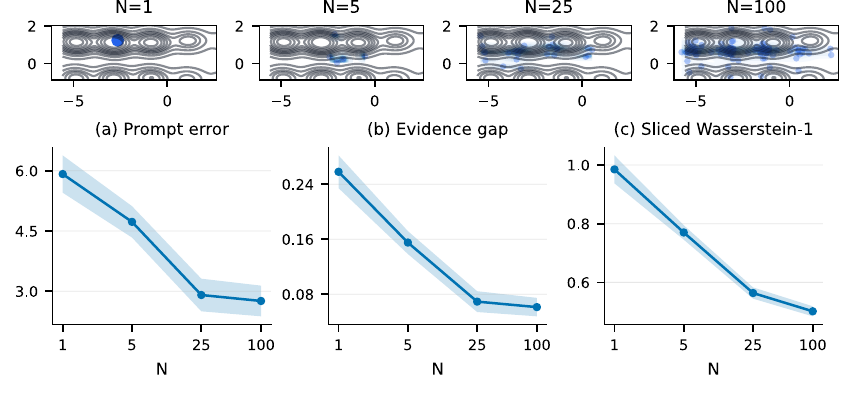}
    \caption{KDE plots across particle counts (top row), prompt error $|c-c_\star|$, evidence gap $|\log p_c(y)-\log p_{c_\star}(y)|$, and sliced Wasserstein-1 distance (bottom row). Shaded regions show $95\%$ confidence intervals across $100$ observations.}
    \label{fig:app-combined-scaling}
\end{figure}
\begin{table}[htbp]
    \centering
    \small
    \setlength{\tabcolsep}{5pt}
    \renewcommand{\arraystretch}{1.15}
    \caption{
        Ablation study of CWGF.
        Results are reported as mean with subscripted 95\% confidence intervals
        over 100 runs. SW1 denotes the sliced Wasserstein-1 distance computed with 1000 projections.
        Lower is better for all metrics.
    }
    \label{tab:app-cwgf_ablation}
    \begin{tabularx}{\linewidth}{
        @{}
        >{\raggedright\arraybackslash}X
        r
        r
        r
        r
        @{}
    }
        \toprule
        Variant
        & Steps
        & SW1 $\downarrow$
        & Prompt error $\downarrow$
        & Evidence gap $\downarrow$ \\
        \midrule
        Full CWGF
        & 100
        & $0.516_{\pm 0.019}$
        & $2.758_{\pm 0.390}$
        & $0.063_{\pm 0.014}$ \\
        PFODE-integration
        & 100
        & $0.576_{\pm 0.016}$
        & $2.622_{\pm 0.372}$
        & $0.056_{\pm 0.014}$ \\
        W2 flow
        & 100
        & $1.611_{\pm 0.155}$
        & $5.659_{\pm 0.871}$
        & $0.308_{\pm 0.080}$ \\
        \quad W2 flow with more steps
        & 500
        & $\mathbf{0.515_{\pm 0.024}}$
        & $\mathbf{2.022_{\pm 0.163}}$
        & $\mathbf{0.033_{\pm 0.005}}$ \\
        \bottomrule
    \end{tabularx}
\end{table}
We note that the direct W2 prior flow requires a much larger computational budget of 500 steps, while also reducing prompt error and evidence gap, possibly due to the multimodal nature of the posterior, whereas CWGF uses the CM flow-map update and diffusion annealing.

%% file: includes/checklist.tex
\section*{NeurIPS Paper Checklist}

\begin{enumerate}

\item {\bf Claims}
    \item[] Question: Do the main claims made in the abstract and introduction accurately reflect the paper's contributions and scope?
    \item[] Answer: \answerYes{}
    \item[] Justification: The abstract and introduction describe where our method aims at improving SOTA, and its success in doing so is motivated by the tables in the Experiments Section~\ref{sec:experiments}.

\item {\bf Limitations}
    \item[] Question: Does the paper discuss the limitations of the work performed by the authors?
    \item[] Answer: \answerYes{}
    \item[] Justification: The Conclusion Section~\ref{sec:conclusion} states some limitations of our work, and points to future perspectives to improve them.

\item {\bf Theory assumptions and proofs}
    \item[] Question: For each theoretical result, does the paper provide the full set of assumptions and a complete (and correct) proof?
    \item[] Answer: \answerYes{}
    \item[] Justification: The main theoretical assumptions are stated in Section~\ref{sec:Method}, while the proofs can be found in Appendix~\ref{app:proofs}

    \item {\bf Experimental result reproducibility}
    \item[] Question: Does the paper fully disclose all the information needed to reproduce the main experimental results of the paper to the extent that it affects the main claims and/or conclusions of the paper (regardless of whether the code and data are provided or not)?
    \item[] Answer: \answerYes{}
    \item[] Justification: The detailed pseudocode is provided in Appendix~\ref{sec:algorithm} of our algorithm, together with the hyperparameters chosen to obtain the results shown in Section~\ref{sec:experiments}. The implementation code is publicly available at \url{https://github.com/akyildiz-group/CWGF-release}.

\item {\bf Open access to data and code}
    \item[] Question: Does the paper provide open access to the data and code, with sufficient instructions to faithfully reproduce the main experimental results, as described in supplemental material?
    \item[] Answer: \answerYes{}
    \item[] Justification: The code is released at \url{https://github.com/akyildiz-group/CWGF-release}.

\item {\bf Experimental setting/details}
    \item[] Question: Does the paper specify all the training and test details (e.g., data splits, hyperparameters, how they were chosen, type of optimizer) necessary to understand the results?
    \item[] Answer: \answerYes{}
    \item[] Justification: The details are described in Section~\ref{sec:experiments}, including the datasets split used and the inverse problems definitions. Section~\ref{sec:algorithm} provides the hyperparameters used.

\item {\bf Experiment statistical significance}
    \item[] Question: Does the paper report error bars suitably and correctly defined or other appropriate information about the statistical significance of the experiments?
    \item[] Answer: \answerNA{}
    \item[] Justification: While the tables in Section~\ref{sec:experiments} may report the empirical std value, it is uncommon in the literature to do so, as it is not of great interest.

\item {\bf Experiments compute resources}
    \item[] Question: For each experiment, does the paper provide sufficient information on the computer resources (type of compute workers, memory, time of execution) needed to reproduce the experiments?
    \item[] Answer: \answerYes{}
    \item[] Justification: Appendix~\ref{sec:compute-resources} reports the GPU used to run the experiments and the NFEs required by each method (all the baselines share the same underlying architecture).
    
\item {\bf Code of ethics}
    \item[] Question: Does the research conducted in the paper conform, in every respect, with the NeurIPS Code of Ethics \url{https://neurips.cc/public/EthicsGuidelines}?
    \item[] Answer: \answerYes{}
    \item[] Justification: The authors state that the research conducted conforms to NeurIPS Code of Ethics, as no human subject or participant has been involved, and no data-related concern emerges from the dataset used.

\item {\bf Broader impacts}
    \item[] Question: Does the paper discuss both potential positive societal impacts and negative societal impacts of the work performed?
    \item[] Answer: \answerYes{}
    \item[] Justification: This work has a neutral societal impact.
    
\item {\bf Safeguards}
    \item[] Question: Does the paper describe safeguards that have been put in place for responsible release of data or models that have a high risk for misuse (e.g., pre-trained language models, image generators, or scraped datasets)?
    \item[] Answer: \answerNA{}
    \item[] Justification: This work does not pose a high risk for misuse.

\item {\bf Licenses for existing assets}
    \item[] Question: Are the creators or original owners of assets (e.g., code, data, models), used in the paper, properly credited and are the license and terms of use explicitly mentioned and properly respected?
    \item[] Answer: \answerYes{}
    \item[] Justification: The code and data are released under CC-BY 4.0. The original publications are cited for all existing datasets, models, and software used in this work.

\item {\bf New assets}
    \item[] Question: Are new assets introduced in the paper well documented and is the documentation provided alongside the assets?
    \item[] Answer: \answerYes{}
    \item[] Justification: Documentation for the code and data is provided in the public repository at \url{https://github.com/akyildiz-group/CWGF-release}.

\item {\bf Crowdsourcing and research with human subjects}
    \item[] Question: For crowdsourcing experiments and research with human subjects, does the paper include the full text of instructions given to participants and screenshots, if applicable, as well as details about compensation (if any)? 
    \item[] Answer: \answerNA{}
    \item[] Justification: This work does not involve crowdsourcing or research with human subjects.

\item {\bf Institutional review board (IRB) approvals or equivalent for research with human subjects}
    \item[] Question: Does the paper describe potential risks incurred by study participants, whether such risks were disclosed to the subjects, and whether Institutional Review Board (IRB) approvals (or an equivalent approval/review based on the requirements of your country or institution) were obtained?
    \item[] Answer: \answerNA{}
    \item[] Justification: Not applicable.

\item {\bf Declaration of LLM usage}
    \item[] Question: Does the paper describe the usage of LLMs if it is an important, original, or non-standard component of the core methods in this research? Note that if the LLM is used only for writing, editing, or formatting purposes and does \emph{not} impact the core methodology, scientific rigor, or originality of the research, declaration is not required.
    \item[] Answer: \answerNA{}
    \item[] Justification: LLMs are used only for writing, editing, and formatting purposes.

\end{enumerate}